\RequirePackage{fix-cm}
\documentclass{article}

\usepackage{iclr2027_conference,times}
\iclrfinalcopy %

\usepackage{booktabs}
\usepackage{longtable}
\usepackage{amssymb}
\usepackage{pifont}
\usepackage{amsmath,amssymb}
\usepackage[dvipsnames,table]{xcolor}

\usepackage{tikz}
\usetikzlibrary{arrows.meta,calc,positioning,decorations.pathreplacing}
\usepackage[utf8]{inputenc}
\usepackage[T1]{fontenc}
\usepackage[colorlinks=true,
            linkcolor=purple,
            citecolor=blue,
            urlcolor=BrickRed]{hyperref}
\usepackage{url}
\usepackage{booktabs}
\usepackage{pifont}
\usepackage{colortbl}
\usepackage{makecell}
\usepackage{graphicx}
\usepackage{xspace}

\definecolor{headerblue}{RGB}{218,235,247}
\definecolor{headerpurple}{RGB}{229,224,244}
\definecolor{headeryellow}{RGB}{252,239,205}
\definecolor{headerorange}{RGB}{248,224,213}
\definecolor{headergreen}{RGB}{220,238,225}
\definecolor{lightgray}{RGB}{247,247,247}
\definecolor{edisgreen}{RGB}{229,242,233}
\usepackage{booktabs}
\usepackage{tabularx}
\usepackage{array}

\definecolor{headerblue}{HTML}{DCEAF7}
\definecolor{headerpeach}{HTML}{F8E1D6}
\definecolor{sectionblue}{HTML}{EDF5FB}
\definecolor{sectiongreen}{HTML}{EDF7F0}
\definecolor{sectionpeach}{HTML}{FCF1EB}
\definecolor{sectiongray}{HTML}{F3F4F5}

\usepackage{graphicx}
\usepackage{booktabs}
\usepackage{multirow}
\usepackage{colortbl}
\usepackage{colortbl}
\usepackage{rotating}
\usepackage{array}
\usepackage{amsmath}
\usepackage{amssymb}
\usepackage{amsfonts}
\usepackage{amsthm}
\usepackage[ruled,vlined,linesnumbered]{algorithm2e}
\usepackage{float}
\usepackage{wrapfig}
\usepackage{needspace}
\usepackage[section]{placeins}
\usepackage{booktabs}
\usepackage{multirow}
\usepackage{graphicx}
\usepackage{adjustbox}
\usepackage{soul}
\usepackage{enumitem}
\usepackage[most]{tcolorbox}
\newtcolorbox{keyquestion}{
  colback=sectiongray,
  colframe=black,
  boxrule=0.8pt,
  left=6pt, right=6pt, top=2pt, bottom=2pt,
  arc=3pt,
  before skip=2pt, after skip=2pt
}
\theoremstyle{plain}
\newtheorem{theorem}{Theorem}

\newtheorem{proposition}[theorem]{Proposition}
\newtheorem{corollary}[theorem]{Corollary}

\newtheorem{fact}{Fact}
\theoremstyle{definition}

\newtheorem{definition}{Definition}

\newcommand{\edis}{\texttt{EDiS}\xspace}

\title{EDiS: Edge Disjoint Subgraph Sparsification Framework for Graph Neural Networks}

\author{\parbox{\dimexpr\textwidth-2\tabcolsep\relax}{\raggedright
Sai Karthik Navuluru$^{1}$, Siddhartha Shankar Das$^{2}$, Franck Dernoncourt$^{3}$,
S M Ferdous$^{4}$, Ryan A.\ Rossi$^{3}$, Nesreen K.\ Ahmed$^{5}$,
Baris Coskunuzer$^{1}$, Alex Pothen$^{6}$, Lakshman Tamil$^{1}$,
Mahantesh M Halappanavar$^{2}$\\[6pt]
\normalfont\small
$^{1}$University of Texas at Dallas \quad
$^{2}$Pacific Northwest National Laboratory \quad
$^{3}$Adobe Research\\
$^{4}$University of North Carolina at Charlotte \quad
$^{5}$Cisco \quad
$^{6}$Purdue University}}

\usepackage[font=small]{caption}
\usepackage{etoolbox}
\makeatletter
\patchcmd{\l@section}{\addvspace{1.0em \@plus\p@}}{\addvspace{0.02em \@plus\p@}}{}{}
\makeatother

\begin{document}

\maketitle
\lhead{EDiS} %

\begin{abstract}
Sparse GNN training reduces computation, but deciding which edges to
keep can be costly. Reusing one sparse graph is cheap but locks
training to a fixed topology, while varying it across epochs can
require repeated sampling or recomputation. We introduce
\edis{} (\textbf{E}dge-\textbf{Di}sjoint \textbf{S}ubgraph
sparsification framework), which separates \emph{one-time structural
extraction} from \emph{per-epoch graph composition}. \edis{} decomposes the
graph once into cacheable edge-disjoint subgraphs, then
recombines them into graphs with edge-budget constraints across epochs
and retention ratios without re-extracting structure.

Our default construction uses feature-based scores and successive
maximum-score covering forests, while the same composition mechanism also
supports alternative edge-selection rules. We provide a combinatorial
analysis of the per-epoch sampler, the composition step that draws a
training graph from the cached decomposition. We show that, under the
default covering-forest selector, the stored decomposition
deterministically preserves high-score cut edges, and we derive a
selector-agnostic conditional bound on high-score cut survival in
composed training graphs.
Across 19 homophilic, heterophilic, and large-scale node-classification
benchmarks against 17 baselines under the same edge-budget, EDiS
achieves the highest mean benchmark score (accuracy/ROC-AUC) and the
lowest average rank and gap-to-best among ranked methods. Ablations show the clearest benefits of structural decomposition and epoch variation at tight edge-budgets. 
\end{abstract}

\let\savedaddcontentsline\addcontentsline
\renewcommand{\addcontentsline}[3]{}

\section{Introduction}
\label{sec:introduction}

\begingroup
\setlength{\parskip}{4.5pt}
Graph neural networks (GNNs) propagate information along the edges of the graph \citep{Kipf:2017tc}, which determines the amount of information computed, exchanged, and stored at each epoch during training of the GNN.
Sparsifying the graph, by selectively removing edges, can therefore reduce the computational costs \citep{Chen:2021ugs}, but adds a different cost of deciding \emph{which} edges to keep (or remove). 
In some sampling-based methods, this cost dominates training itself. For example, on Reddit, a graph with 233K nodes and 11.6M edges, the sampling-based method AS-GCN \citep{Huang:2018asgcn} spends 20.7 seconds per epoch for sampling neighbors, versus 3.5 seconds for the forward and backward pass \citep{Zeng:2020graphsaint}. 
Thus, a sparsified graph does not guarantee economical training if the graph must be rebuilt frequently.

One option is to build a sparse graph once and reuse it every epoch to amortize the cost of sparsification. However, this also locks the GNN into the same set of edges for the entire training run \citep{Chen:2021ugs,TEDDY}.
Another option is to change the sparse graph over time using different approaches such as repeated subgraph sampling in GraphSAINT \citep{Zeng:2020graphsaint}, incrementally grown spanning subgraph in SpanGNN \citep{Gu:2024spangnn}, or computing a fresh random spanning tree at each epoch in GERN \citep{Bonchi:2025gern}. 
While these methods show the value of varying sparse views during training, each couples graph variation to a particular sampling or structural-update mechanism at a fixed sparsity level.
To the best of our knowledge, existing methods do not formulate the problem as a persistent, edge-disjoint structural library that is built once and efficiently recombined into different sparse graphs during training. This leads to an important question: 
\emph{Can we extract graph structure once and reuse it to cheaply generate many different training graphs with specific edge-budget constraints?}

\edis{} (\textbf{E}dge-\textbf{Di}sjoint \textbf{S}ubgraph sparsification framework) addresses this question by turning edge-disjointness from a one-off computational step into a \emph{cacheable training paradigm}. It decomposes the original graph into persistent and non-overlapping structural subgraphs only once.
It then repeatedly recombines them into training graphs with edge-budget constraints across epochs and retention ratios, without recomputing the decomposition.
In our default setup, subgraphs are successive connectivity-preserving sparse backbones, extracted from residual edges and cached before the GNN training begins.

While GNN architecture, parameters, and training objectives remain the same, the sparse message-passing graph changes across epochs. For example, a $30\%$ edge budget does not mean the model is stuck with the same $30\%$ of edges for the entire run, but that it processes $30\%$ of edges per epoch, while different epochs are exposed to different parts of the original graph (Figure~\ref{fig:intro-edis-motivation}).

\begin{wrapfigure}{r}{2.1in}
\vspace*{-.15in}
\centering
\includegraphics[width=\linewidth]{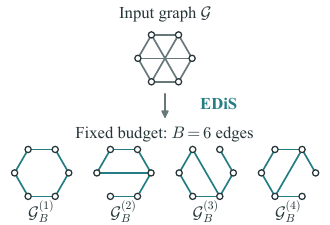}
\caption{\footnotesize \textbf{\edis{} generates sparse training graphs under a fixed edge budget.}}
\label{fig:intro-edis-motivation}
\vspace{-.15in}
\end{wrapfigure}

\edis{} has three main components: \emph{edge scoring}, \emph{structural extraction}, and \emph{budgeted composition}.
While the first two are a one-time computation, budgeted composition is repeated at every epoch with specific edge-budgets.
Edge scores rank connections, a residual selector groups them into cacheable subgraphs, and composition selects a training graph from those subgraphs. The subgraphs in extraction are disjoint, but edges can be repeated over time in the composed graphs ($\mathcal G_B^{(t)}$).

Our default algorithm pairs feature-based cosine scores with successive
maximum-score covering forests. The first forest provides a high-score, connectivity-preserving backbone of the original graph, and once these edges are removed, later forests span the connected components of the residual graph using previously unused edges.
During training, the stored subgraphs are sampled with score-dependent priority and combined until the target budget is reached, trimming excess edges and filling any shortfall from the residual pool (detailed in Section~\ref{sec:methodology}).

\edis{} is not tied to one scoring or selection function: edge scoring and residual selection functions are modular, so any combination shares the same edge-budget-constrained procedure (Section~\ref{sec:methodology}). Separating extraction from composition also clarifies what we can guarantee: for maximum-score residual forests, stored prefixes preserve score-threshold cut certificates, and since a single epoch's graph may omit or trim these forests, we separately characterize conditional edge retention and fixed-cut survival under composition for any residual selector.
Matched-budget ablations provide empirical evidence that structural grouping and epoch resampling contribute to predictive performance (Section~\ref{sec:ablations}).

We evaluate \edis{} against 17 baselines on 19 datasets spanning homophilic, heterophilic, and large-scale graphs, separating sparse-graph inference (\edis{}-Lite) from full-graph inference (\edis{}-Full), while isolating structural decomposition, epoch resampling, selector choice, and edge scoring through matched-edge-budget ablations.
At $30\%$ retention, \edis{}-Lite achieves the highest mean benchmark score among ranked methods (average rank $2.53$, gap $1.25$ points), while \edis{}-Full serves as a reference (Section~\ref{sec:experiments}).
Cacheable subgraphs and graph variation give the largest gains at tight edge-budgets, while fixed graphs or simpler sampling remain competitive elsewhere.

Our contributions are as follows.
\begin{itemize}[leftmargin=*, itemsep=2pt, topsep=2pt, parsep=0pt, partopsep=0pt]
\item \textbf{A cacheable, edge-budget-agnostic graph decomposition:}
\edis{} extracts structural subgraphs once, then recomposes them into
training graphs with edge-budget constraints across epochs and sparsity
levels, without rerunning extraction or changing the GNN architecture
or loss.

\item \textbf{A novel, selector-agnostic composition procedure:}
To the best of our knowledge, \edis{} is the first method to separate
extraction from composition: composition is independent of the
extraction methods that build the subgraphs (e.g., maximum spanning
forests or random selection).

\item \textbf{A combinatorial analysis of the sampler:}
For maximum-score residual forests, we prove score-aware cut
preservation for the stored decomposition; for any residual selector,
we characterize conditional retention and high-score cut
survival in the composed graph.

\item \textbf{A controlled empirical study of structural reuse:}
Matched-edge-budget experiments across 17 baselines and 19 datasets
test whether a cacheable structural decomposition and epoch variation
justify their construction cost.

\end{itemize}
\endgroup

\vspace{-.15in}

\section{Related Work}
\label{sec:RelatedWork}

Existing methods either decide which part of the graph to use every
step, or decide once and fix the graph beforehand via
sparsification. \edis{} sits between these two settings, deciding once
but still varying the graph every epoch. It builds a cacheable
edge-disjoint structural library once, then composes an
exact-budget graph at each epoch. This section relates
\edis{} to four lines of work.

\textbf{Sampling-based GNN training.} \quad
Neighbor sampling bounds neighbors per node
\citep{Hamilton:2017tp}. Layer-wise sampling instead scores whole node layers
\citep{Chen:2018fastgcn}, and LADIES conditions each layer
on the layer above it to stay connected
\citep{Zou:2019ladies}. Subgraph methods sample
subgraphs directly \citep{Chiang:2019clustergcn} or correct
the resulting bias \citep{Zeng:2020graphsaint}, and attribute-guided
sampling prioritizes edges by similarity and diversity \citep{Das:2024agsgnn}. All
of these resample the full topology at every step,
rather than keeping an altered one.
Other work avoids resampling by caching past embeddings
\citep{Fey:2021gnnautoscale}, quantizing or compressing
activations \citep{Ding:2021vqgnn,Liu:2022exact}, repartitioning
across workers \citep{Wan:2022bnsgcn}, or compensating for
missing cross-partition messages \citep{Shi:2023lmc}. \edis{} instead organizes the topology
once and reuses precomputed subgraphs under an explicit edge-budget,
rather than resampling every step.

\textbf{Graph sparsification.} \quad
A large family of methods target edge removal before training begins, which includes degree- and similarity-based heuristics
\citep{Hamann:2016sparsification,Satuluri:2011lsim,Liu:2023dspar},
spectral and distance-preserving sparsifiers
\citep{Spielman:2008effres,Peleg:1989spanners}, and cut
sparsifiers built by repeatedly extracting a maximum spanning forest
\citep{Benczur:1996cutsparsification,ForsterDeVos2023}. In contrast, some methods learn which edges to retain jointly with graph learning
\citep{Chen:2021ugs,TEDDY,Zhang:2024adaglt,SGS-GNN,Zhang:2025mog}. However, 
no single method effectively preserves multiple graph properties of interest well
\citep{Chen:2023sparsifybench}. \edis{} employs similar sparsification principles, such as degree-,
neighbor/similarity-, spanner-, and forest-based residual selection,
and its composition step treats them identically regardless of the
method used to build the library.

\textbf{Stochastic topology variation.} \quad
DropEdge \citep{Rong:2020dropedge} and dropout-induced subgraphs
\citep{Luo:2025dropoutgcn} vary the topology by removing
edges independently at random at each epoch. Other methods instead learn
edge predictors \citep{Chen:2025adedgedrop} or stabilize message
passing under a changing structure \citep{Lee:2025aggbuffer}. These
structures overlap substantially across epochs and carry no edge-budget or
structural guarantee. By contrast, \edis{}'s epoch-to-epoch variation
comes from a controlled, cacheable partition rather than independent
per-edge noise.

\textbf{Cacheable edge-disjoint decomposition.} \quad
A few methods train over a sequence of sparse structures. For example, SpanGNN
\citep{Gu:2024spangnn} trains on grown subgraphs and GERN
\citep{Bonchi:2025gern} trains on random spanning trees, but neither is
persistent. Prior edge-disjoint
methods use disjoint structures as sampled or sequential
objects, not as a persistent, budget-agnostic decomposition
recomposed into exact-budget graphs across epochs and retention
ratios. 
To the best of our knowledge, \edis{} is the first approach to decouple
one-time extraction from per-epoch composition, so it stays
cached and reused throughout training.
Appendix~\ref{app:related_work}
(Table~\ref{tab:relwork-comparison}) compares these methods by
budget, reusability, exposure, and disjointness. Appendix~\ref{app:gern_spangnn}
adds quantitative comparisons against SpanGNN and GERN, and discusses our
attempted AGS-GNN evaluation.

\section{EDiS: A Structured Sparsification Framework}
\label{sec:methodology}

Sparse GNN training poses a tradeoff between computational cost and graph coverage. 
While reusing a sparse graph built once for the entire training is cheap, it significantly limits learning by fixing the graph structure for every epoch. 
Thus, changing the sparse graph over time exposes the GNN to more of the original topology. However, rebuilding from scratch at every epoch can make sparsification as expensive as the training it was meant to speed up.

\edis{} addresses this tradeoff using two main components: \emph{Decomposition}, a one-time step that builds the cached structural library; and \emph{Composition}, a per-epoch step that composes a graph from the library. 
We extract structural information once, organize it into cacheable edge-disjoint subgraphs, and recombine them into an exact-budget graph each epoch. 
The GNN then processes only a fixed number of edges at a time, while seeing many structured subsets of the graph. The same GNN parameters are trained throughout, and only the sparse message-passing graph varies across epochs, illustrated in Figure~\ref{fig:edis_pipeline}.

\begin{figure}[t]
    \centering
    \includegraphics[width=\linewidth]{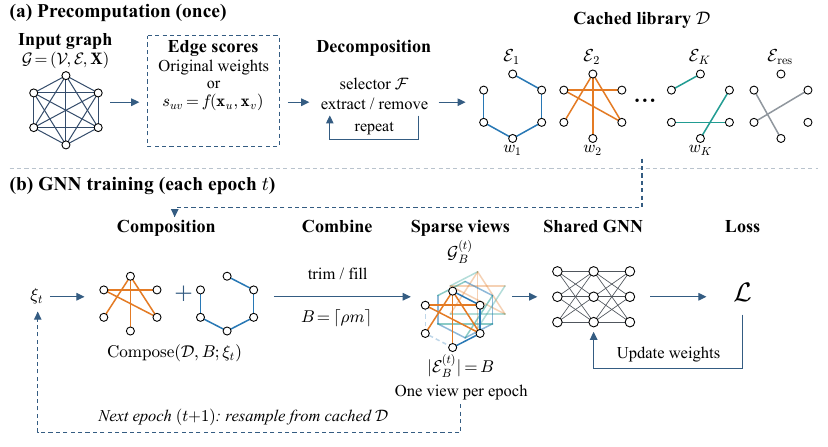}
    \caption{\textbf{Overview of the \edis{} pipeline.}}
    \label{fig:edis_pipeline}
\end{figure}

\subsection{One GNN, Many Sparse Views}
\label{subsec:edis_overview}

Let $\mathcal G=(\mathcal V,\mathcal E,\mathbf X)$
be an undirected graph with $n$ nodes, $m$ edges, and node features
$\mathbf X$. An edge-retention ratio $\rho\in(0,1]$ sets the training
budget $B=\lceil \rho m\rceil$, the maximum number of edges used for a given epoch.
\edis{} \emph{does not pick one $B$-edge graph
and train on it for the whole run}. Instead, the same GNN trains on a
sequence of different sparse graphs, a different one for each epoch $t$,
\begin{equation}
\mathcal G_B^{(t)}
=
\bigl(\mathcal V,\mathcal E_B^{(t)},\mathbf X\bigr),
\qquad
|\mathcal E_B^{(t)}|=B,
\label{eq:edis_epoch_graph}
\end{equation}
where $\mathcal E_B^{(t)}$ is freshly composed each epoch.

The GNN's parameters are shared across all epochs, and a GCN layer
at epoch $t$ still uses
\[    \mathbf H_{\ell+1}
    =
    \sigma\!\left(
    \widehat{\mathbf A}_B^{(t)}
    \mathbf H_{\ell}
    \mathbf W_{\ell}
    \right),
    \qquad
    \mathbf H_{0}=\mathbf X,
\]
where only the sparse adjacency $\widehat{\mathbf A}_B^{(t)}$ changes
from epoch to epoch, while the learned weights $\mathbf W_{\ell}$
are updated through backpropagation. Over $T$ epochs, the total number
of distinct edges seen, $\bigl|\bigcup_{t=1}^T \mathcal
E_B^{(t)}\bigr|$, can therefore be much larger than $B$, even though
each epoch is limited by $B$.

The challenge is to generate new sparse graphs without
rerunning extraction for each epoch. \edis{} splits the method into two
stages, a one-time decomposition (parametrized by edge score $s$ and
residual selector $\mathcal F$, defined next), and an efficient per-epoch
composition.
\begin{equation}
\boxed{
\mathcal D=\operatorname{Decompose}(\mathcal G;s,\mathcal F),
\qquad
\mathcal E_B^{(t)}
=
\operatorname{Compose}(\mathcal D,B;\xi_t).
}
\label{eq:edis_factorization}
\end{equation}
The decomposition $\mathcal D$ is computed once, and each epoch then
needs only randomness $\xi_t$ to compose a new $B$-edge graph
from it. This composition is \emph{not} adapted by the GNN loss,
since the edge scores, subgraphs, and weights it draws on are all fixed
before training begins.

\subsection{One-Time Structural Decomposition}
\label{subsec:budget_construction}

Each original edge $e=(u,v)$ gets a fixed score $s_e$,
with higher scores preferred by the selector. Our default
score is a shifted cosine similarity between endpoint features,
\begin{equation}
s_{uv}
=
\frac12
\left(
1+
\frac{\mathbf x_u^\top\mathbf x_v}
{\|\mathbf x_u\|_2\,\|\mathbf x_v\|_2}
\right).
\end{equation}
Thus, the default score also depends on both the node features and graph topology. In
Section~\ref{sec:experiments} we evaluate other scoring methods. Scores alone,
however, do not decide the subgraphs. A \emph{residual selector} $\mathcal
F$ picks one subgraph at a time from remaining edges, mapping it to a
subset at round $k$,
\begin{equation}
\mathcal E_k=\mathcal F\bigl(\mathcal V,\mathcal E_{\mathrm{rem}}^{(k-1)},s\bigr)
\subseteq\mathcal E_{\mathrm{rem}}^{(k-1)}.
\end{equation}

The selector may use $s$ or ignore it entirely, and need not return a
forest or a fixed edge count. \edis{} treats all selectors the same way
during training, supporting degree-based, neighbor-based,
spanner, and random selectors (detailed in Appendix~\ref{app:selector_family}).

Our default selector, $\mathcal F_{\max}$, groups edges into
\emph{structural subgraphs} rather than picking high-score edges one
at a time. At each round, it extracts a maximum-score covering forest
(\textbf{MaxCF}) from the remaining edges. At round $k$,
\begin{equation}
\mathcal E_k
=
\mathcal F_{\max}\bigl(\mathcal V,\mathcal E_{\mathrm{rem}}^{(k-1)},s\bigr)
\in
\operatorname*{arg\,max}_{F\text{ covering forest}}
\sum_{e\in F}s_e,
\qquad
\mathcal E_{\mathrm{rem}}^{(k)}
=
\mathcal E_{\mathrm{rem}}^{(k-1)}
\setminus\mathcal E_k.
\end{equation}
Extraction stops when the per-round limit is reached, the coverage target is
met, or no edges remain, giving $K$ subgraphs. These
choices are fixed once, independent of $B$. This independence, rather than
the absence of choices, is what \emph{budget-agnostic} means in this context. The
first subgraph $\mathcal E_1$ is thus a high-score, connectivity-preserving
backbone. After removing those edges, $\mathcal E_2$ gives an
alternative backbone from remaining edges, and so on. The result is an edge-disjoint
collection,
\begin{equation}
\mathcal E
=
\mathcal E_1
\sqcup\cdots\sqcup
\mathcal E_K
\sqcup
\mathcal E_{\mathrm{res}},
\label{eq:prelim_partition}
\end{equation}
where $\mathcal E_{\mathrm{res}}$ holds any edges not assigned to a
stored subgraph. Each subgraph also receives a weight,
\begin{equation}
w_k=\max\Bigl\{\textstyle\sum_{e\in\mathcal E_k}s_e,\ \varepsilon\Bigr\},
\end{equation}
so $w_k$ is the subgraph's total score, floored at a small $\varepsilon>0$
to keep sampling weights strictly positive. This floor is not a
maximum over individual edges, so $w_k$ always reflects the subgraph's
aggregate score, not that of any single edge. The subgraphs, weights, and
residual pool are computed once and stay fixed.

\subsection{An Exact-Budget Graph at Every Epoch}
\label{subsec:composition}

At epoch $t$, \edis{} decides which subgraphs to include, and in what
order, by drawing a weighted random ordering without replacement. If
$R_r$ is the set of subgraphs still available before draw $r$,
\begin{equation}
\pi^{(t)}=(k_1,\ldots,k_K),
\qquad
\Pr(k_r=k\mid R_r)
=
\frac{w_k}{\sum_{j\in R_r}w_j}.
\end{equation}
Higher-weighted subgraphs are thus more likely to appear early, but no
fixed order is imposed. We weight by total score, not average score.
We add subgraphs one at a time in this order, stopping once the edge
count reaches or surpasses $B$. If it overshoots $B$, we remove edges
starting with the lowest-weighted subgraph. Each low-weight
subgraph is removed in full until the remaining excess is smaller
than the next subgraph, which is then trimmed by that amount and
kept. At most one subgraph ends up partially trimmed. Every
higher-weighted subgraph remains in its entirety.
If instead there is a shortfall, it is filled uniformly from
$\mathcal E_{\mathrm{res}}$. Either way, the result satisfies
\begin{equation}
\boxed{|\mathcal E_B^{(t)}|=B.}
\end{equation}
A deterministic alternative that always takes the top subgraphs will repeat the same subgraphs for every epoch, thus collapsing the epoch-to-epoch diversity that drives most of \edis{}'s gain at tight budgets (Appendix~\ref{app:trimming}, Appendix~\ref{app:coverage-reuse}).
Exact construction and boundary cases appear in Appendix~\ref{app:edis_composition}.
A worked example in Appendix~\ref{app:selector-worked-example} shows different selectors producing different subgraphs that still compose into the same edge budget. 
At small budgets, one high-weight subgraph may already fill $B$, so variation comes mainly from trimming within that subgraph, not the choice of subgraph itself. 
As $B$ grows and several subgraphs are needed, epoch-to-epoch variation increasingly comes from the differences at the subgraph-level.

The distinction from a fixed sparsifier matters. A one-shot method
trains the GNN on the same set of edges every epoch,
\begin{equation}
\mathcal E_B,\mathcal E_B,\ldots,\mathcal E_B,
\end{equation}
whereas \edis{} trains the same GNN on a different graph at
each epoch,
\begin{equation}
\mathcal E_B^{(1)},
\mathcal E_B^{(2)},
\ldots,
\mathcal E_B^{(T)}.
\end{equation}
\edis{} thus pairs sparse per-epoch computation with exposure to
many structured views of the graph. Under the implementation conditions
of Appendix~\ref{app:complexity}, per-epoch composition scales with $B$ and $K$
rather than $m$, so decomposition costs amortize across epochs and budgets.

\subsection{Structural Preservation for Decomposition and Composition}
\label{subsec:theory}

\begingroup
\setlength{\abovedisplayskip}{6pt}
\setlength{\belowdisplayskip}{6pt}
\setlength{\abovedisplayshortskip}{4pt}
\setlength{\belowdisplayshortskip}{4pt}

Since composition can omit or trim subgraphs for a smaller budget, the structural guarantees for the cached decomposition $\mathcal D$ may not hold for the composed graph $\mathcal E_B^{(t)}$, which is processed by the GNN at each epoch. We therefore analyze these two objects separately.

Both guarantees below concern cuts carrying many high-score edges,
so we index by a score threshold $\tau$.
Let $U_k=\mathcal E_1\cup\cdots\cup\mathcal E_k$, and let
$\mathcal E_{\ge\tau}=\{e\in\mathcal E:s_e\ge\tau\}$,
$\mathcal G_{\ge\tau}=(\mathcal V,\mathcal E_{\ge\tau})$, and
$U_{k,\ge\tau}=U_k\cap\mathcal E_{\ge\tau}$. For a nontrivial node set
$S\subsetneq\mathcal V$ and graph $H$, $\delta_H(S)$ denotes the
edges of $H$ crossing the cut induced by $S$.

\textbf{The stored decomposition.}\quad
MaxCFs satisfy a bottleneck property (Fact~\ref{fact:bottleneck}).
Together with the successive-forest cut certificate
(Fact~\ref{fact:cut_certificate}), which yields the guarantee below.
\vspace{4pt}
\begin{theorem}[\underline{Score-threshold cut preservation}]
\label{thm:score_threshold_certificate}
If each subgraph is a MaxCF of the current
residual graph, then for every completed round $1\le k\le K$,
threshold $\tau$, and a nontrivial cut $S$,
\begin{equation}
    \bigl|\delta_{U_{k,\ge\tau}}(S)\bigr|
    \ge
    \min\bigl\{k,\,|\delta_{\mathcal G_{\ge\tau}}(S)|\bigr\}.
    \label{eq:score_threshold_certificate}
\end{equation}
\end{theorem}
\vspace{-6pt}
\noindent(Proof and top-$k$ corollary are in
Appendices~\ref{app:score_certificate}--\ref{app:topk_corollaries}).
Informally, cuts that mattered in $\mathcal G$ retain crossing edges
once enough subgraphs are kept. This holds for the default
covering-forest selector.

\textbf{The composed training graph.}\quad
Once $B<|U_K|$, the sparse-budget regime this paper targets, $\mathcal
E_B^{(t)}$ need not contain any $U_k$. We therefore bound retention
for the composed graph directly, conditioning on the sampled subgraph
order $\pi^{(t)}$. Index the residual edge set as $k=0$ alongside the
$K$ stored subgraphs. Let $\theta_k$ be set $k$'s retained fraction
after trimming and $x_k(S,\tau)$ its edges crossing $S$ with score at
least $\tau$.
\vspace{4pt}
\begin{theorem}[\underline{Conditional retention under composition}]
\label{thm:app_budget_retention}
Condition on the subgraph order $\pi^{(t)}$ and trim each subgraph
uniformly at random to its retained size. For any nontrivial cut $S$
and threshold $\tau$, the probability that the composed edge set
$\mathcal E_B^{(t)}$ keeps a high-score edge crossing $S$ satisfies
\begin{equation}
    \Pr\!\left(\mathcal E_B^{(t)}\cap\delta_{\mathcal G}(S)
    \cap\mathcal E_{\ge\tau}\ne\varnothing\;\middle|\;\pi^{(t)}\right)
    \ge
    1-\exp\!\left(-\sum_{k=0}^{K}\theta_k\,x_k(S,\tau)\right).
    \label{eq:highscore_survival}
\end{equation}
\end{theorem}
\vspace{-6pt}
\noindent(Proof and the expected-retention identity are detailed in
Appendix~\ref{app:budget_theory}.) Informally, a cut with enough
high-score edges is unlikely to be fully severed, so the message-passing
channel across it usually survives composition too, though no single
subgraph's survival is guaranteed. Unlike
Theorem~\ref{thm:score_threshold_certificate}, this bound is
selector-agnostic, needing only a uniform trim within each subgraph, not
forest structure, but it gives no guarantee that $\mathcal G_B^{(t)}$
stays connected at any fixed $B$ (Appendix~\ref{app:budget_theory}).

Together, Theorems~\ref{thm:score_threshold_certificate}
and~\ref{thm:app_budget_retention} deliberately decouple two
guarantees typically bundled together, each scoped to where it
holds. The stored decomposition carries a strong, selector-specific
certificate established once at extraction time. The composed graph
instead carries a selector-agnostic bound that holds for any
selector, trading certificate strength for that generality, which is
exactly what lets \edis{} swap selectors freely.
\endgroup

\section{Experimental Results}
\label{sec:experiments}

\subsection{Experimental Setup}
\label{sec:setup}

We perform empirical evaluation on 19 benchmark datasets (six heterophilic,
eight homophilic, and five large-scale) and node-classification as the task, which is similar in spirit to \citet{Luo:2024classicgnns}.
We use accuracy as the metric, except for Minesweeper, Questions, and ogbn-proteins where we use ROC-AUC as the metric. Complete details are listed in
Table~\ref{tab:datasets} (Appendix~\ref{app:experimental}). We compare \edis{} against 17 other methods, including 3 reference anchors (graph-free MLP, spanning tree, and tuned full-graph GCN), 11 topology-guided sparsifiers, and 3 semantic-guided sparsifiers (Appendix~\ref{app:experimental}).

Our default pairs a GCN backbone with the MaxCF selector
and cosine-similarity as edge scores
(Theorem~\ref{thm:score_threshold_certificate}).
\textbf{EDiS-Lite} and \textbf{EDiS-Full} share this pipeline and
differ only at inference, using the sparse composed graph or the
complete graph. \textbf{Full GCN} is a separately tuned
complete-graph reference, distinct from \edis{} at $100\%$
retention. Every method shares the same architecture, optimizer, and epoch
budget per dataset (Appendix~\ref{app:hyperparameters}). Full-batch
training recomputes normalized adjacency from the current epoch
graph, with scores guiding selection only (not propagation)
(Appendix~\ref{app:propagation_details}). Predictive runs, matched in
count across methods, are averaged per dataset.

\subsection{Predictive Performance}
\label{sec:results}

\textbf{(RQ1, EDiS vs.\ 17 baselines) Does decoupling decomposition
from composition cost accuracy at a matched budget?} We benchmark \edis{}
against all 17 baselines at $\rho=0.30$ on 19 datasets, reporting the
geometric mean (GM) and arithmetic mean (AM) across all 19 datasets while excluding reference
anchors. Table~\ref{tab:edis-main-results-all-datasets} shows EDiS-Lite
attaining the best GM and AM, ahead of every
topology- and semantic-guided sparsifier at the same budget.

\begin{table*}[t]
\centering
\fontsize{6.3pt}{6.8pt}\selectfont
\setlength{\tabcolsep}{0.6pt}
\renewcommand{\arraystretch}{0.8}
\begin{tabular*}{\textwidth}{@{}c@{\hspace{3pt}}l|@{\extracolsep{\fill}}*{10}{c}@{}}
\toprule
\rowcolor[rgb]{0.90,0.93,0.97}
& \textbf{Method} & \textbf{Cora} & \textbf{CiteSeer} & \textbf{\shortstack{Amazon\\Photo}} & \textbf{WikiCS} & \textbf{\shortstack{Amazon\\Comp.}} & \textbf{\shortstack{Coauthor\\CS}} & \textbf{PubMed} & \textbf{Chameleon} & \textbf{Squirrel} & \textbf{Minesweeper} \\
\midrule
\multirow{4}{*}{\rotatebox[origin=c]{90}{\fontsize{5.3pt}{5.5pt}\selectfont Ref.}} & MLP (No graph) & 60.14{\fontsize{4.8pt}{5.0pt}\selectfont\,\textpm2.62} & 57.00{\fontsize{4.8pt}{5.0pt}\selectfont\,\textpm1.13} & 92.16{\fontsize{4.8pt}{5.0pt}\selectfont\,\textpm0.24} & 73.10{\fontsize{4.8pt}{5.0pt}\selectfont\,\textpm0.19} & 86.99{\fontsize{4.8pt}{5.0pt}\selectfont\,\textpm0.39} & 95.11{\fontsize{4.8pt}{5.0pt}\selectfont\,\textpm0.05} & 69.82{\fontsize{4.8pt}{5.0pt}\selectfont\,\textpm1.86} & 36.59{\fontsize{4.8pt}{5.0pt}\selectfont\,\textpm3.32} & 39.77{\fontsize{4.8pt}{5.0pt}\selectfont\,\textpm1.59} & 51.34{\fontsize{4.8pt}{5.0pt}\selectfont\,\textpm0.77} \\
 & Spanning Tree & 80.86{\fontsize{4.8pt}{5.0pt}\selectfont\,\textpm1.17} & 70.50{\fontsize{4.8pt}{5.0pt}\selectfont\,\textpm0.64} & 92.46{\fontsize{4.8pt}{5.0pt}\selectfont\,\textpm0.12} & 75.05{\fontsize{4.8pt}{5.0pt}\selectfont\,\textpm0.23} & 89.76{\fontsize{4.8pt}{5.0pt}\selectfont\,\textpm0.30} & 94.68{\fontsize{4.8pt}{5.0pt}\selectfont\,\textpm0.14} & 80.40{\fontsize{4.8pt}{5.0pt}\selectfont\,\textpm0.99} & 34.72{\fontsize{4.8pt}{5.0pt}\selectfont\,\textpm2.81} & 36.93{\fontsize{4.8pt}{5.0pt}\selectfont\,\textpm1.58} & 70.73{\fontsize{4.8pt}{5.0pt}\selectfont\,\textpm0.21} \\
 & Full GCN & 84.26{\fontsize{4.8pt}{5.0pt}\selectfont\,\textpm0.66} & 72.50{\fontsize{4.8pt}{5.0pt}\selectfont\,\textpm0.36} & 95.95{\fontsize{4.8pt}{5.0pt}\selectfont\,\textpm0.14} & 80.57{\fontsize{4.8pt}{5.0pt}\selectfont\,\textpm0.06} & 93.70{\fontsize{4.8pt}{5.0pt}\selectfont\,\textpm0.31} & 95.61{\fontsize{4.8pt}{5.0pt}\selectfont\,\textpm0.36} & 80.76{\fontsize{4.8pt}{5.0pt}\selectfont\,\textpm0.55} & 44.67{\fontsize{4.8pt}{5.0pt}\selectfont\,\textpm3.84} & 44.79{\fontsize{4.8pt}{5.0pt}\selectfont\,\textpm1.82} & 97.49{\fontsize{4.8pt}{5.0pt}\selectfont\,\textpm0.18} \\
 & \textbf{EDiS-Full (Ours)} & 84.08{\fontsize{4.8pt}{5.0pt}\selectfont\,\textpm0.83} & 70.88{\fontsize{4.8pt}{5.0pt}\selectfont\,\textpm0.78} & 96.14{\fontsize{4.8pt}{5.0pt}\selectfont\,\textpm0.19} & 80.39{\fontsize{4.8pt}{5.0pt}\selectfont\,\textpm0.29} & 93.59{\fontsize{4.8pt}{5.0pt}\selectfont\,\textpm0.05} & 95.98{\fontsize{4.8pt}{5.0pt}\selectfont\,\textpm0.15} & 80.20{\fontsize{4.8pt}{5.0pt}\selectfont\,\textpm0.46} & 43.13{\fontsize{4.8pt}{5.0pt}\selectfont\,\textpm2.53} & 43.26{\fontsize{4.8pt}{5.0pt}\selectfont\,\textpm2.43} & 82.62{\fontsize{4.8pt}{5.0pt}\selectfont\,\textpm1.09} \\
\midrule
\multirow{11}{*}{\rotatebox[origin=c]{90}{\fontsize{5.3pt}{5.5pt}\selectfont Topology-guided}} & Random & 73.40{\fontsize{4.8pt}{5.0pt}\selectfont\,\textpm0.51} & 65.86{\fontsize{4.8pt}{5.0pt}\selectfont\,\textpm0.95} & 94.38{\fontsize{4.8pt}{5.0pt}\selectfont\,\textpm0.16} & 77.04{\fontsize{4.8pt}{5.0pt}\selectfont\,\textpm0.63} & 91.05{\fontsize{4.8pt}{5.0pt}\selectfont\,\textpm0.15} & 94.71{\fontsize{4.8pt}{5.0pt}\selectfont\,\textpm0.04} & 74.26{\fontsize{4.8pt}{5.0pt}\selectfont\,\textpm0.62} & 37.85{\fontsize{4.8pt}{5.0pt}\selectfont\,\textpm3.35} & 40.90{\fontsize{4.8pt}{5.0pt}\selectfont\,\textpm1.70} & 72.39{\fontsize{4.8pt}{5.0pt}\selectfont\,\textpm1.04} \\
 & Rank Deg. & \cellcolor[rgb]{0.87,0.88,0.91}\underline{76.34}{\fontsize{4.8pt}{5.0pt}\selectfont\,\textpm0.78} & 65.10{\fontsize{4.8pt}{5.0pt}\selectfont\,\textpm0.24} & 93.79{\fontsize{4.8pt}{5.0pt}\selectfont\,\textpm0.05} & 76.13{\fontsize{4.8pt}{5.0pt}\selectfont\,\textpm0.53} & 88.04{\fontsize{4.8pt}{5.0pt}\selectfont\,\textpm0.28} & 94.68{\fontsize{4.8pt}{5.0pt}\selectfont\,\textpm0.06} & 76.38{\fontsize{4.8pt}{5.0pt}\selectfont\,\textpm0.46} & 40.35{\fontsize{4.8pt}{5.0pt}\selectfont\,\textpm3.82} & 40.17{\fontsize{4.8pt}{5.0pt}\selectfont\,\textpm1.89} & 71.09{\fontsize{4.8pt}{5.0pt}\selectfont\,\textpm0.21} \\
 & Local Deg. & 73.52{\fontsize{4.8pt}{5.0pt}\selectfont\,\textpm0.47} & 64.76{\fontsize{4.8pt}{5.0pt}\selectfont\,\textpm0.84} & 94.49{\fontsize{4.8pt}{5.0pt}\selectfont\,\textpm0.03} & 75.69{\fontsize{4.8pt}{5.0pt}\selectfont\,\textpm0.48} & 89.28{\fontsize{4.8pt}{5.0pt}\selectfont\,\textpm0.33} & 94.93{\fontsize{4.8pt}{5.0pt}\selectfont\,\textpm0.22} & \cellcolor[rgb]{0.87,0.88,0.91}\underline{77.32}{\fontsize{4.8pt}{5.0pt}\selectfont\,\textpm1.08} & 41.37{\fontsize{4.8pt}{5.0pt}\selectfont\,\textpm2.59} & 41.80{\fontsize{4.8pt}{5.0pt}\selectfont\,\textpm1.66} & 72.53{\fontsize{4.8pt}{5.0pt}\selectfont\,\textpm1.15} \\
 & Forest Fire & 71.98{\fontsize{4.8pt}{5.0pt}\selectfont\,\textpm0.39} & 63.82{\fontsize{4.8pt}{5.0pt}\selectfont\,\textpm1.05} & 95.40{\fontsize{4.8pt}{5.0pt}\selectfont\,\textpm0.22} & \cellcolor[rgb]{0.96,0.87,0.70}79.20{\fontsize{4.8pt}{5.0pt}\selectfont\,\textpm0.58} & 92.57{\fontsize{4.8pt}{5.0pt}\selectfont\,\textpm0.12} & 94.67{\fontsize{4.8pt}{5.0pt}\selectfont\,\textpm0.33} & 73.46{\fontsize{4.8pt}{5.0pt}\selectfont\,\textpm0.85} & 41.28{\fontsize{4.8pt}{5.0pt}\selectfont\,\textpm1.96} & 43.15{\fontsize{4.8pt}{5.0pt}\selectfont\,\textpm1.90} & 70.76{\fontsize{4.8pt}{5.0pt}\selectfont\,\textpm0.29} \\
 & G-Spar & 72.02{\fontsize{4.8pt}{5.0pt}\selectfont\,\textpm1.05} & 64.28{\fontsize{4.8pt}{5.0pt}\selectfont\,\textpm0.57} & 93.42{\fontsize{4.8pt}{5.0pt}\selectfont\,\textpm0.37} & 74.91{\fontsize{4.8pt}{5.0pt}\selectfont\,\textpm0.39} & 91.02{\fontsize{4.8pt}{5.0pt}\selectfont\,\textpm0.21} & 94.78{\fontsize{4.8pt}{5.0pt}\selectfont\,\textpm0.08} & 73.48{\fontsize{4.8pt}{5.0pt}\selectfont\,\textpm0.90} & 38.62{\fontsize{4.8pt}{5.0pt}\selectfont\,\textpm4.13} & 39.77{\fontsize{4.8pt}{5.0pt}\selectfont\,\textpm1.57} & \cellcolor[rgb]{0.78,0.87,0.97}\textbf{78.97}{\fontsize{4.8pt}{5.0pt}\selectfont\,\textpm0.51} \\
 & L-Sim & \cellcolor[rgb]{0.96,0.87,0.70}75.18{\fontsize{4.8pt}{5.0pt}\selectfont\,\textpm0.72} & \cellcolor[rgb]{0.96,0.87,0.70}66.06{\fontsize{4.8pt}{5.0pt}\selectfont\,\textpm0.42} & 94.88{\fontsize{4.8pt}{5.0pt}\selectfont\,\textpm0.13} & 78.71{\fontsize{4.8pt}{5.0pt}\selectfont\,\textpm0.96} & \cellcolor[rgb]{0.96,0.87,0.70}93.04{\fontsize{4.8pt}{5.0pt}\selectfont\,\textpm0.07} & 95.02{\fontsize{4.8pt}{5.0pt}\selectfont\,\textpm0.29} & 76.98{\fontsize{4.8pt}{5.0pt}\selectfont\,\textpm0.74} & 39.71{\fontsize{4.8pt}{5.0pt}\selectfont\,\textpm3.97} & 39.23{\fontsize{4.8pt}{5.0pt}\selectfont\,\textpm2.38} & 77.17{\fontsize{4.8pt}{5.0pt}\selectfont\,\textpm0.41} \\
 & SCAN & 71.38{\fontsize{4.8pt}{5.0pt}\selectfont\,\textpm0.67} & 64.54{\fontsize{4.8pt}{5.0pt}\selectfont\,\textpm0.37} & 94.01{\fontsize{4.8pt}{5.0pt}\selectfont\,\textpm0.22} & 75.08{\fontsize{4.8pt}{5.0pt}\selectfont\,\textpm0.44} & 91.60{\fontsize{4.8pt}{5.0pt}\selectfont\,\textpm0.18} & \cellcolor[rgb]{0.96,0.87,0.70}95.09{\fontsize{4.8pt}{5.0pt}\selectfont\,\textpm0.02} & 74.32{\fontsize{4.8pt}{5.0pt}\selectfont\,\textpm0.50} & 36.05{\fontsize{4.8pt}{5.0pt}\selectfont\,\textpm3.41} & 39.50{\fontsize{4.8pt}{5.0pt}\selectfont\,\textpm1.10} & \cellcolor[rgb]{0.87,0.88,0.91}\underline{78.53}{\fontsize{4.8pt}{5.0pt}\selectfont\,\textpm0.35} \\
 & L-Spar & 71.38{\fontsize{4.8pt}{5.0pt}\selectfont\,\textpm1.48} & 64.32{\fontsize{4.8pt}{5.0pt}\selectfont\,\textpm0.57} & 94.07{\fontsize{4.8pt}{5.0pt}\selectfont\,\textpm0.20} & 75.10{\fontsize{4.8pt}{5.0pt}\selectfont\,\textpm0.25} & 90.95{\fontsize{4.8pt}{5.0pt}\selectfont\,\textpm0.22} & 95.03{\fontsize{4.8pt}{5.0pt}\selectfont\,\textpm0.22} & 73.48{\fontsize{4.8pt}{5.0pt}\selectfont\,\textpm0.90} & 36.11{\fontsize{4.8pt}{5.0pt}\selectfont\,\textpm3.06} & 39.17{\fontsize{4.8pt}{5.0pt}\selectfont\,\textpm1.31} & \cellcolor[rgb]{0.96,0.87,0.70}78.45{\fontsize{4.8pt}{5.0pt}\selectfont\,\textpm0.86} \\
 & Spectral & 72.86{\fontsize{4.8pt}{5.0pt}\selectfont\,\textpm0.85} & \cellcolor[rgb]{0.78,0.87,0.97}\textbf{66.82}{\fontsize{4.8pt}{5.0pt}\selectfont\,\textpm0.44} & 95.24{\fontsize{4.8pt}{5.0pt}\selectfont\,\textpm0.29} & 78.80{\fontsize{4.8pt}{5.0pt}\selectfont\,\textpm0.47} & 92.13{\fontsize{4.8pt}{5.0pt}\selectfont\,\textpm0.11} & 94.68{\fontsize{4.8pt}{5.0pt}\selectfont\,\textpm0.13} & 75.04{\fontsize{4.8pt}{5.0pt}\selectfont\,\textpm1.43} & \cellcolor[rgb]{0.96,0.87,0.70}42.98{\fontsize{4.8pt}{5.0pt}\selectfont\,\textpm3.60} & \cellcolor[rgb]{0.78,0.87,0.97}\textbf{43.94}{\fontsize{4.8pt}{5.0pt}\selectfont\,\textpm1.18} & 70.48{\fontsize{4.8pt}{5.0pt}\selectfont\,\textpm1.94} \\
 & $t$-Spanner & 72.60{\fontsize{4.8pt}{5.0pt}\selectfont\,\textpm0.90} & 65.70{\fontsize{4.8pt}{5.0pt}\selectfont\,\textpm0.63} & 94.95{\fontsize{4.8pt}{5.0pt}\selectfont\,\textpm0.24} & 78.07{\fontsize{4.8pt}{5.0pt}\selectfont\,\textpm0.23} & 91.51{\fontsize{4.8pt}{5.0pt}\selectfont\,\textpm0.11} & 94.64{\fontsize{4.8pt}{5.0pt}\selectfont\,\textpm0.27} & 74.74{\fontsize{4.8pt}{5.0pt}\selectfont\,\textpm0.59} & 38.71{\fontsize{4.8pt}{5.0pt}\selectfont\,\textpm3.72} & 42.65{\fontsize{4.8pt}{5.0pt}\selectfont\,\textpm2.01} & 71.13{\fontsize{4.8pt}{5.0pt}\selectfont\,\textpm1.03} \\
 & DSpar & 73.90{\fontsize{4.8pt}{5.0pt}\selectfont\,\textpm0.36} & 64.70{\fontsize{4.8pt}{5.0pt}\selectfont\,\textpm1.22} & \cellcolor[rgb]{0.78,0.87,0.97}\textbf{95.62}{\fontsize{4.8pt}{5.0pt}\selectfont\,\textpm0.33} & \cellcolor[rgb]{0.87,0.88,0.91}\underline{79.47}{\fontsize{4.8pt}{5.0pt}\selectfont\,\textpm0.68} & \cellcolor[rgb]{0.87,0.88,0.91}\underline{93.17}{\fontsize{4.8pt}{5.0pt}\selectfont\,\textpm0.08} & 94.52{\fontsize{4.8pt}{5.0pt}\selectfont\,\textpm0.12} & \cellcolor[rgb]{0.78,0.87,0.97}\textbf{77.86}{\fontsize{4.8pt}{5.0pt}\selectfont\,\textpm0.46} & 39.73{\fontsize{4.8pt}{5.0pt}\selectfont\,\textpm3.87} & \cellcolor[rgb]{0.87,0.88,0.91}\underline{43.48}{\fontsize{4.8pt}{5.0pt}\selectfont\,\textpm1.56} & 72.56{\fontsize{4.8pt}{5.0pt}\selectfont\,\textpm2.00} \\
\midrule
\multirow{3}{*}{\raisebox{-6pt}{\rotatebox[origin=c]{90}{\fontsize{5.3pt}{5.5pt}\selectfont\shortstack{Semantic\\[-2pt]guided}}}} & Unified-LTH & 72.12{\fontsize{4.8pt}{5.0pt}\selectfont\,\textpm3.47} & 63.76{\fontsize{4.8pt}{5.0pt}\selectfont\,\textpm1.19} & 91.94{\fontsize{4.8pt}{5.0pt}\selectfont\,\textpm0.36} & 75.49{\fontsize{4.8pt}{5.0pt}\selectfont\,\textpm0.17} & 86.88{\fontsize{4.8pt}{5.0pt}\selectfont\,\textpm0.74} & 94.99{\fontsize{4.8pt}{5.0pt}\selectfont\,\textpm0.37} & 75.26{\fontsize{4.8pt}{5.0pt}\selectfont\,\textpm0.67} & 37.61{\fontsize{4.8pt}{5.0pt}\selectfont\,\textpm3.67} & 38.17{\fontsize{4.8pt}{5.0pt}\selectfont\,\textpm2.09} & 69.70{\fontsize{4.8pt}{5.0pt}\selectfont\,\textpm1.12} \\
 & AdaGLT & 72.74{\fontsize{4.8pt}{5.0pt}\selectfont\,\textpm0.63} & 65.26{\fontsize{4.8pt}{5.0pt}\selectfont\,\textpm1.00} & 94.16{\fontsize{4.8pt}{5.0pt}\selectfont\,\textpm0.19} & 76.48{\fontsize{4.8pt}{5.0pt}\selectfont\,\textpm0.13} & 91.53{\fontsize{4.8pt}{5.0pt}\selectfont\,\textpm0.42} & \cellcolor[rgb]{0.87,0.88,0.91}\underline{95.13}{\fontsize{4.8pt}{5.0pt}\selectfont\,\textpm0.08} & 75.84{\fontsize{4.8pt}{5.0pt}\selectfont\,\textpm0.56} & 38.03{\fontsize{4.8pt}{5.0pt}\selectfont\,\textpm4.12} & 42.30{\fontsize{4.8pt}{5.0pt}\selectfont\,\textpm3.91} & 62.60{\fontsize{4.8pt}{5.0pt}\selectfont\,\textpm1.27} \\
 & MoG & 73.66{\fontsize{4.8pt}{5.0pt}\selectfont\,\textpm0.77} & 65.28{\fontsize{4.8pt}{5.0pt}\selectfont\,\textpm1.50} & \cellcolor[rgb]{0.96,0.87,0.70}95.45{\fontsize{4.8pt}{5.0pt}\selectfont\,\textpm0.29} & 77.48{\fontsize{4.8pt}{5.0pt}\selectfont\,\textpm0.66} & 92.55{\fontsize{4.8pt}{5.0pt}\selectfont\,\textpm0.05} & 94.94{\fontsize{4.8pt}{5.0pt}\selectfont\,\textpm0.20} & 71.56{\fontsize{4.8pt}{5.0pt}\selectfont\,\textpm0.54} & \cellcolor[rgb]{0.87,0.88,0.91}\underline{43.40}{\fontsize{4.8pt}{5.0pt}\selectfont\,\textpm3.13} & \cellcolor[rgb]{0.96,0.87,0.70}43.25{\fontsize{4.8pt}{5.0pt}\selectfont\,\textpm1.12} & 74.02{\fontsize{4.8pt}{5.0pt}\selectfont\,\textpm0.32} \\
\midrule
 & \textbf{EDiS-Lite (Ours)} & \cellcolor[rgb]{0.78,0.87,0.97}\textbf{77.16}{\fontsize{4.8pt}{5.0pt}\selectfont\,\textpm0.98} & \cellcolor[rgb]{0.87,0.88,0.91}\underline{66.08}{\fontsize{4.8pt}{5.0pt}\selectfont\,\textpm1.23} & \cellcolor[rgb]{0.87,0.88,0.91}\underline{95.60}{\fontsize{4.8pt}{5.0pt}\selectfont\,\textpm0.15} & \cellcolor[rgb]{0.78,0.87,0.97}\textbf{80.25}{\fontsize{4.8pt}{5.0pt}\selectfont\,\textpm0.79} & \cellcolor[rgb]{0.78,0.87,0.97}\textbf{93.35}{\fontsize{4.8pt}{5.0pt}\selectfont\,\textpm0.21} & \cellcolor[rgb]{0.78,0.87,0.97}\textbf{95.72}{\fontsize{4.8pt}{5.0pt}\selectfont\,\textpm0.06} & \cellcolor[rgb]{0.96,0.87,0.70}77.28{\fontsize{4.8pt}{5.0pt}\selectfont\,\textpm0.57} & \cellcolor[rgb]{0.78,0.87,0.97}\textbf{43.75}{\fontsize{4.8pt}{5.0pt}\selectfont\,\textpm4.14} & 42.63{\fontsize{4.8pt}{5.0pt}\selectfont\,\textpm2.04} & 73.82{\fontsize{4.8pt}{5.0pt}\selectfont\,\textpm0.97} \\
\bottomrule
\end{tabular*}
\vspace{1pt}
\begin{tabular*}{\textwidth}{@{}c@{\hspace{3pt}}l|@{\extracolsep{\fill}}*{9}{c}@{\hspace{4pt}}|@{\hspace{4pt}}cc@{}}
\toprule
\rowcolor[rgb]{0.90,0.93,0.97}
& \textbf{Method} & \textbf{\shortstack{Coauthor\\Phys.}} & \textbf{Roman} & \textbf{\shortstack{Amazon\\Ratings}} & \textbf{Questions} & \textbf{Proteins} & \textbf{Arxiv} & \textbf{Reddit} & \textbf{Pokec} & \textbf{Products} & \textbf{AM} & \textbf{GM} \\
\midrule
\multirow{4}{*}{\rotatebox[origin=c]{90}{\fontsize{5.3pt}{5.5pt}\selectfont Ref.}} & MLP (No graph) & 96.36{\fontsize{4.8pt}{5.0pt}\selectfont\,\textpm0.01} & 66.23{\fontsize{4.8pt}{5.0pt}\selectfont\,\textpm0.07} & 50.32{\fontsize{4.8pt}{5.0pt}\selectfont\,\textpm0.24} & 69.69{\fontsize{4.8pt}{5.0pt}\selectfont\,\textpm0.35} & 72.91{\fontsize{4.8pt}{5.0pt}\selectfont\,\textpm0.17} & 55.60{\fontsize{4.8pt}{5.0pt}\selectfont\,\textpm0.16} & 69.66{\fontsize{4.8pt}{5.0pt}\selectfont\,\textpm0.09} & 62.31{\fontsize{4.8pt}{5.0pt}\selectfont\,\textpm0.02} & 45.46{\fontsize{4.8pt}{5.0pt}\selectfont\,\textpm0.19}& 65.82 & 63.49 \\
 & Spanning Tree & 96.94{\fontsize{4.8pt}{5.0pt}\selectfont\,\textpm0.02} & 87.36{\fontsize{4.8pt}{5.0pt}\selectfont\,\textpm0.23} & 51.73{\fontsize{4.8pt}{5.0pt}\selectfont\,\textpm0.37} & 74.85{\fontsize{4.8pt}{5.0pt}\selectfont\,\textpm0.24} & 74.13{\fontsize{4.8pt}{5.0pt}\selectfont\,\textpm0.16} & 67.27{\fontsize{4.8pt}{5.0pt}\selectfont\,\textpm0.12} & 86.07{\fontsize{4.8pt}{5.0pt}\selectfont\,\textpm0.14} & 69.36{\fontsize{4.8pt}{5.0pt}\selectfont\,\textpm0.01} & 67.79{\fontsize{4.8pt}{5.0pt}\selectfont\,\textpm0.09}& 73.77 & 71.31 \\
 & Full GCN & 97.51{\fontsize{4.8pt}{5.0pt}\selectfont\,\textpm0.07} & 90.74{\fontsize{4.8pt}{5.0pt}\selectfont\,\textpm0.03} & 53.52{\fontsize{4.8pt}{5.0pt}\selectfont\,\textpm0.42} & 78.49{\fontsize{4.8pt}{5.0pt}\selectfont\,\textpm0.18} & 77.48{\fontsize{4.8pt}{5.0pt}\selectfont\,\textpm0.06} & 72.83{\fontsize{4.8pt}{5.0pt}\selectfont\,\textpm0.20} & 95.44{\fontsize{4.8pt}{5.0pt}\selectfont\,\textpm0.05} & 86.24{\fontsize{4.8pt}{5.0pt}\selectfont\,\textpm0.04} & 80.67{\fontsize{4.8pt}{5.0pt}\selectfont\,\textpm0.10}& 80.17 & 78.17 \\
 & \textbf{EDiS-Full (Ours)} & 97.45{\fontsize{4.8pt}{5.0pt}\selectfont\,\textpm0.03} & 75.00{\fontsize{4.8pt}{5.0pt}\selectfont\,\textpm0.51} & 53.16{\fontsize{4.8pt}{5.0pt}\selectfont\,\textpm0.68} & 77.70{\fontsize{4.8pt}{5.0pt}\selectfont\,\textpm0.53} & 76.56{\fontsize{4.8pt}{5.0pt}\selectfont\,\textpm0.22} & 73.29{\fontsize{4.8pt}{5.0pt}\selectfont\,\textpm0.30} & 95.16{\fontsize{4.8pt}{5.0pt}\selectfont\,\textpm0.01} & 84.45{\fontsize{4.8pt}{5.0pt}\selectfont\,\textpm0.19} & 81.38{\fontsize{4.8pt}{5.0pt}\selectfont\,\textpm0.13} & 78.13 & 76.16 \\
\midrule
\multirow{11}{*}{\rotatebox[origin=c]{90}{\fontsize{5.3pt}{5.5pt}\selectfont Topology-guided}} & Random & 96.66{\fontsize{4.8pt}{5.0pt}\selectfont\,\textpm0.05} & 67.07{\fontsize{4.8pt}{5.0pt}\selectfont\,\textpm0.24} & 48.40{\fontsize{4.8pt}{5.0pt}\selectfont\,\textpm0.43} & 71.89{\fontsize{4.8pt}{5.0pt}\selectfont\,\textpm0.90} & 76.83{\fontsize{4.8pt}{5.0pt}\selectfont\,\textpm0.08} & 67.77{\fontsize{4.8pt}{5.0pt}\selectfont\,\textpm0.18} & 94.69{\fontsize{4.8pt}{5.0pt}\selectfont\,\textpm0.02} & 76.33{\fontsize{4.8pt}{5.0pt}\selectfont\,\textpm0.05} & 79.08{\fontsize{4.8pt}{5.0pt}\selectfont\,\textpm0.06} & 73.71 & 71.47 \\
 & Rank Deg. & 96.40{\fontsize{4.8pt}{5.0pt}\selectfont\,\textpm0.17} & 70.53{\fontsize{4.8pt}{5.0pt}\selectfont\,\textpm0.33} & 51.23{\fontsize{4.8pt}{5.0pt}\selectfont\,\textpm0.52} & 71.49{\fontsize{4.8pt}{5.0pt}\selectfont\,\textpm0.48} & 66.01{\fontsize{4.8pt}{5.0pt}\selectfont\,\textpm0.78} & 62.59{\fontsize{4.8pt}{5.0pt}\selectfont\,\textpm0.07} & 77.28{\fontsize{4.8pt}{5.0pt}\selectfont\,\textpm0.04} & 67.29{\fontsize{4.8pt}{5.0pt}\selectfont\,\textpm0.10} & 63.13{\fontsize{4.8pt}{5.0pt}\selectfont\,\textpm0.02} & 70.95 & 69.05 \\
 & Local Deg. & 96.30{\fontsize{4.8pt}{5.0pt}\selectfont\,\textpm0.18} & 70.79{\fontsize{4.8pt}{5.0pt}\selectfont\,\textpm0.31} & \cellcolor[rgb]{0.96,0.87,0.70}51.47{\fontsize{4.8pt}{5.0pt}\selectfont\,\textpm0.71} & \cellcolor[rgb]{0.87,0.88,0.91}\underline{74.39}{\fontsize{4.8pt}{5.0pt}\selectfont\,\textpm0.83} & 71.10{\fontsize{4.8pt}{5.0pt}\selectfont\,\textpm0.00} & 68.34{\fontsize{4.8pt}{5.0pt}\selectfont\,\textpm0.17} & 90.82{\fontsize{4.8pt}{5.0pt}\selectfont\,\textpm0.15} & \cellcolor[rgb]{0.87,0.88,0.91}\underline{80.07}{\fontsize{4.8pt}{5.0pt}\selectfont\,\textpm0.09} & 77.62{\fontsize{4.8pt}{5.0pt}\selectfont\,\textpm0.16} & 74.03 & 72.11 \\
 & Forest Fire & \cellcolor[rgb]{0.96,0.87,0.70}96.91{\fontsize{4.8pt}{5.0pt}\selectfont\,\textpm0.03} & 67.43{\fontsize{4.8pt}{5.0pt}\selectfont\,\textpm0.62} & 48.37{\fontsize{4.8pt}{5.0pt}\selectfont\,\textpm0.22} & 71.34{\fontsize{4.8pt}{5.0pt}\selectfont\,\textpm0.46} & \cellcolor[rgb]{0.78,0.87,0.97}\textbf{78.25}{\fontsize{4.8pt}{5.0pt}\selectfont\,\textpm0.32} & 68.75{\fontsize{4.8pt}{5.0pt}\selectfont\,\textpm0.21} & \cellcolor[rgb]{0.87,0.88,0.91}\underline{95.41}{\fontsize{4.8pt}{5.0pt}\selectfont\,\textpm0.00} & 79.43{\fontsize{4.8pt}{5.0pt}\selectfont\,\textpm0.10} & \cellcolor[rgb]{0.87,0.88,0.91}\underline{80.65}{\fontsize{4.8pt}{5.0pt}\selectfont\,\textpm0.05} & 74.36 & 72.27 \\
 & G-Spar & 96.73{\fontsize{4.8pt}{5.0pt}\selectfont\,\textpm0.06} & 72.20{\fontsize{4.8pt}{5.0pt}\selectfont\,\textpm0.42} & 49.64{\fontsize{4.8pt}{5.0pt}\selectfont\,\textpm0.61} & 70.35{\fontsize{4.8pt}{5.0pt}\selectfont\,\textpm1.77} & 70.39{\fontsize{4.8pt}{5.0pt}\selectfont\,\textpm0.56} & 65.86{\fontsize{4.8pt}{5.0pt}\selectfont\,\textpm0.06} & 90.05{\fontsize{4.8pt}{5.0pt}\selectfont\,\textpm0.01} & 67.35{\fontsize{4.8pt}{5.0pt}\selectfont\,\textpm0.01} & 72.55{\fontsize{4.8pt}{5.0pt}\selectfont\,\textpm0.05} & 72.44 & 70.31 \\
 & L-Sim & 96.83{\fontsize{4.8pt}{5.0pt}\selectfont\,\textpm0.06} & 71.33{\fontsize{4.8pt}{5.0pt}\selectfont\,\textpm0.21} & \cellcolor[rgb]{0.87,0.88,0.91}\underline{51.66}{\fontsize{4.8pt}{5.0pt}\selectfont\,\textpm0.09} & \cellcolor[rgb]{0.78,0.87,0.97}\textbf{75.39}{\fontsize{4.8pt}{5.0pt}\selectfont\,\textpm0.60} & \cellcolor[rgb]{0.87,0.88,0.91}\underline{77.95}{\fontsize{4.8pt}{5.0pt}\selectfont\,\textpm0.11} & \cellcolor[rgb]{0.87,0.88,0.91}\underline{70.69}{\fontsize{4.8pt}{5.0pt}\selectfont\,\textpm0.08} & 94.35{\fontsize{4.8pt}{5.0pt}\selectfont\,\textpm0.05} & 76.56{\fontsize{4.8pt}{5.0pt}\selectfont\,\textpm0.06} & 79.20{\fontsize{4.8pt}{5.0pt}\selectfont\,\textpm0.09} & \cellcolor[rgb]{0.87,0.88,0.91}\underline{75.26} & \cellcolor[rgb]{0.87,0.88,0.91}\underline{73.06} \\
 & SCAN & 96.84{\fontsize{4.8pt}{5.0pt}\selectfont\,\textpm0.03} & \cellcolor[rgb]{0.87,0.88,0.91}\underline{73.37}{\fontsize{4.8pt}{5.0pt}\selectfont\,\textpm0.23} & 50.25{\fontsize{4.8pt}{5.0pt}\selectfont\,\textpm0.61} & 72.72{\fontsize{4.8pt}{5.0pt}\selectfont\,\textpm0.45} & 70.62{\fontsize{4.8pt}{5.0pt}\selectfont\,\textpm0.42} & 67.36{\fontsize{4.8pt}{5.0pt}\selectfont\,\textpm0.37} & 90.28{\fontsize{4.8pt}{5.0pt}\selectfont\,\textpm0.01} & 68.56{\fontsize{4.8pt}{5.0pt}\selectfont\,\textpm0.02} & 72.56{\fontsize{4.8pt}{5.0pt}\selectfont\,\textpm0.15} & 72.77 & 70.50 \\
 & L-Spar & 96.74{\fontsize{4.8pt}{5.0pt}\selectfont\,\textpm0.05} & \cellcolor[rgb]{0.96,0.87,0.70}72.24{\fontsize{4.8pt}{5.0pt}\selectfont\,\textpm0.66} & 49.84{\fontsize{4.8pt}{5.0pt}\selectfont\,\textpm0.19} & 70.29{\fontsize{4.8pt}{5.0pt}\selectfont\,\textpm1.82} & 70.80{\fontsize{4.8pt}{5.0pt}\selectfont\,\textpm0.39} & 65.82{\fontsize{4.8pt}{5.0pt}\selectfont\,\textpm0.05} & 90.06{\fontsize{4.8pt}{5.0pt}\selectfont\,\textpm0.01} & 67.42{\fontsize{4.8pt}{5.0pt}\selectfont\,\textpm0.02} & 72.64{\fontsize{4.8pt}{5.0pt}\selectfont\,\textpm0.02} & 72.31 & 70.03 \\
 & Spectral & 96.74{\fontsize{4.8pt}{5.0pt}\selectfont\,\textpm0.08} & 66.72{\fontsize{4.8pt}{5.0pt}\selectfont\,\textpm0.42} & 46.86{\fontsize{4.8pt}{5.0pt}\selectfont\,\textpm0.41} & 74.00{\fontsize{4.8pt}{5.0pt}\selectfont\,\textpm0.88} & 77.09{\fontsize{4.8pt}{5.0pt}\selectfont\,\textpm0.66} & \cellcolor[rgb]{0.96,0.87,0.70}69.88{\fontsize{4.8pt}{5.0pt}\selectfont\,\textpm0.15} & \cellcolor[rgb]{0.87,0.88,0.91}\underline{95.41}{\fontsize{4.8pt}{5.0pt}\selectfont\,\textpm0.03} & \cellcolor[rgb]{0.96,0.87,0.70}79.65{\fontsize{4.8pt}{5.0pt}\selectfont\,\textpm0.07} & \cellcolor[rgb]{0.96,0.87,0.70}80.00{\fontsize{4.8pt}{5.0pt}\selectfont\,\textpm0.15} & \cellcolor[rgb]{0.96,0.87,0.70}74.70 & \cellcolor[rgb]{0.96,0.87,0.70}72.69 \\
 & $t$-Spanner & 96.75{\fontsize{4.8pt}{5.0pt}\selectfont\,\textpm0.10} & 67.09{\fontsize{4.8pt}{5.0pt}\selectfont\,\textpm0.45} & 47.96{\fontsize{4.8pt}{5.0pt}\selectfont\,\textpm1.01} & 72.42{\fontsize{4.8pt}{5.0pt}\selectfont\,\textpm0.61} & \cellcolor[rgb]{0.96,0.87,0.70}77.43{\fontsize{4.8pt}{5.0pt}\selectfont\,\textpm0.12} & 67.51{\fontsize{4.8pt}{5.0pt}\selectfont\,\textpm0.09} & 94.72{\fontsize{4.8pt}{5.0pt}\selectfont\,\textpm0.01} & 76.38{\fontsize{4.8pt}{5.0pt}\selectfont\,\textpm0.07} & 79.30{\fontsize{4.8pt}{5.0pt}\selectfont\,\textpm0.01} & 73.91 & 71.74 \\
 & DSpar & 96.84{\fontsize{4.8pt}{5.0pt}\selectfont\,\textpm0.05} & 67.07{\fontsize{4.8pt}{5.0pt}\selectfont\,\textpm0.09} & 47.71{\fontsize{4.8pt}{5.0pt}\selectfont\,\textpm0.61} & 72.77{\fontsize{4.8pt}{5.0pt}\selectfont\,\textpm0.47} & 51.94{\fontsize{4.8pt}{5.0pt}\selectfont\,\textpm1.16} & 55.19{\fontsize{4.8pt}{5.0pt}\selectfont\,\textpm0.13} & \cellcolor[rgb]{0.96,0.87,0.70}95.22{\fontsize{4.8pt}{5.0pt}\selectfont\,\textpm0.01} & 69.39{\fontsize{4.8pt}{5.0pt}\selectfont\,\textpm0.03} & 77.43{\fontsize{4.8pt}{5.0pt}\selectfont\,\textpm0.08} & 72.03 & 69.64 \\
\midrule
\multirow{3}{*}{\raisebox{-6pt}{\rotatebox[origin=c]{90}{\fontsize{5.3pt}{5.5pt}\selectfont\shortstack{Semantic\\[-2pt]guided}}}} & Unified-LTH & 96.49{\fontsize{4.8pt}{5.0pt}\selectfont\,\textpm0.00} & 65.43{\fontsize{4.8pt}{5.0pt}\selectfont\,\textpm0.58} & 40.49{\fontsize{4.8pt}{5.0pt}\selectfont\,\textpm0.52} & 72.42{\fontsize{4.8pt}{5.0pt}\selectfont\,\textpm0.64} & 67.15{\fontsize{4.8pt}{5.0pt}\selectfont\,\textpm0.00} & 56.37{\fontsize{4.8pt}{5.0pt}\selectfont\,\textpm0.00} & 92.68{\fontsize{4.8pt}{5.0pt}\selectfont\,\textpm0.00} & 68.23{\fontsize{4.8pt}{5.0pt}\selectfont\,\textpm0.00} & 75.26{\fontsize{4.8pt}{5.0pt}\selectfont\,\textpm0.00} & 70.55 & 68.02 \\
 & AdaGLT & 96.65{\fontsize{4.8pt}{5.0pt}\selectfont\,\textpm0.00} & \cellcolor[rgb]{0.78,0.87,0.97}\textbf{81.43}{\fontsize{4.8pt}{5.0pt}\selectfont\,\textpm0.00} & 48.05{\fontsize{4.8pt}{5.0pt}\selectfont\,\textpm0.35} & 72.34{\fontsize{4.8pt}{5.0pt}\selectfont\,\textpm0.33} & 52.31{\fontsize{4.8pt}{5.0pt}\selectfont\,\textpm0.00} & 55.93{\fontsize{4.8pt}{5.0pt}\selectfont\,\textpm0.00} & 85.09{\fontsize{4.8pt}{5.0pt}\selectfont\,\textpm0.00} & 64.42{\fontsize{4.8pt}{5.0pt}\selectfont\,\textpm0.00} & 66.74{\fontsize{4.8pt}{5.0pt}\selectfont\,\textpm0.00} & 70.37 & 68.05 \\
 & MoG & \cellcolor[rgb]{0.87,0.88,0.91}\underline{96.93}{\fontsize{4.8pt}{5.0pt}\selectfont\,\textpm0.10} & 70.44{\fontsize{4.8pt}{5.0pt}\selectfont\,\textpm0.60} & 51.24{\fontsize{4.8pt}{5.0pt}\selectfont\,\textpm0.33} & 68.90{\fontsize{4.8pt}{5.0pt}\selectfont\,\textpm1.68} & 72.72{\fontsize{4.8pt}{5.0pt}\selectfont\,\textpm1.75} & 56.42{\fontsize{4.8pt}{5.0pt}\selectfont\,\textpm0.16} & 90.13{\fontsize{4.8pt}{5.0pt}\selectfont\,\textpm0.08} & 66.94{\fontsize{4.8pt}{5.0pt}\selectfont\,\textpm0.07} & 68.27{\fontsize{4.8pt}{5.0pt}\selectfont\,\textpm0.05} & 72.29 & 70.39 \\
\midrule
 & \textbf{EDiS-Lite (Ours)} & \cellcolor[rgb]{0.78,0.87,0.97}\textbf{97.06}{\fontsize{4.8pt}{5.0pt}\selectfont\,\textpm0.09} & 68.58{\fontsize{4.8pt}{5.0pt}\selectfont\,\textpm0.12} & \cellcolor[rgb]{0.78,0.87,0.97}\textbf{52.31}{\fontsize{4.8pt}{5.0pt}\selectfont\,\textpm0.85} & \cellcolor[rgb]{0.96,0.87,0.70}74.12{\fontsize{4.8pt}{5.0pt}\selectfont\,\textpm0.88} & 76.43{\fontsize{4.8pt}{5.0pt}\selectfont\,\textpm0.22} & \cellcolor[rgb]{0.78,0.87,0.97}\textbf{71.41}{\fontsize{4.8pt}{5.0pt}\selectfont\,\textpm0.26} & \cellcolor[rgb]{0.78,0.87,0.97}\textbf{95.61}{\fontsize{4.8pt}{5.0pt}\selectfont\,\textpm0.03} & \cellcolor[rgb]{0.78,0.87,0.97}\textbf{81.19}{\fontsize{4.8pt}{5.0pt}\selectfont\,\textpm0.08} & \cellcolor[rgb]{0.78,0.87,0.97}\textbf{80.78}{\fontsize{4.8pt}{5.0pt}\selectfont\,\textpm0.42} & \cellcolor[rgb]{0.78,0.87,0.97}\textbf{75.95} & \cellcolor[rgb]{0.78,0.87,0.97}\textbf{74.01} \\
\bottomrule
\end{tabular*}

\caption{Mean $\pm$ standard deviation (\%) at retention ratio 0.30. The metric is accuracy, except ROC-AUC on Minesweeper, Questions, and Proteins. Best, second-best, and third-best results are shaded blue/bold, gray/underlined, and tan, respectively. AM and GM denote the arithmetic and geometric means across all 19 datasets. The first four reference rows are excluded from ranking.}
\label{tab:edis-main-results-all-datasets}
\end{table*}

\begin{table}[t]
\caption{\textbf{Summary of Table~\ref{tab:edis-main-results-all-datasets} over all 19 datasets} at $\rho=0.3$/$0.5$. Score is the mean of per-dataset accuracy or ROC-AUC. Rank and Gap are to the best ranked method per dataset, averaged; shading as in Table~\ref{tab:edis-main-results-all-datasets}.}
\label{tab:main-results-summary}
\centering
\scriptsize
\setlength{\tabcolsep}{4pt}
\renewcommand{\arraystretch}{0.8}

\begin{tabular}{@{}ll ccc | ccc@{}}
\toprule
\rowcolor[rgb]{0.90,0.93,0.97}
& & \multicolumn{3}{c}{$\boldsymbol{\rho=0.3}$}
& \multicolumn{3}{c}{$\boldsymbol{\rho=0.5}$} \\

\rowcolor[rgb]{0.90,0.93,0.97}
& \textbf{Method}
& \textbf{Score}\,$\uparrow$
& \textbf{Rank}\,$\downarrow$
& \textbf{Gap}\,$\downarrow$
& \textbf{Score}\,$\uparrow$
& \textbf{Rank}\,$\downarrow$
& \textbf{Gap}\,$\downarrow$ \\

\midrule

\multirow{11}{*}{\rotatebox[origin=c]{90}{\scriptsize Topology-guided}}

& Random
& 73.71 & 8.50 & 3.49
& 75.95 & 7.55 & 3.04 \\

& Local Degree
& 74.03 & 7.05 & 3.17
& 76.60 & 7.18 & 2.39 \\

& Rank Degree
& 70.95 & 10.24 & 6.26
& 73.36 & 9.87 & 5.63 \\

& Forest Fire
& 74.36 & 7.08 & 2.84
& \cellcolor[rgb]{0.96,0.87,0.70}77.10 & \cellcolor[rgb]{0.87,0.88,0.91}\underline{4.42} & \cellcolor[rgb]{0.96,0.87,0.70}1.89 \\

& SCAN
& 72.77 & 8.79 & 4.43
& 75.71 & 8.08 & 3.28 \\

& L-Spar
& 72.31 & 10.24 & 4.89
& 74.60 & 10.50 & 4.39 \\

& G-Spar
& 72.44 & 10.55 & 4.76
& 74.49 & 10.92 & 4.50 \\

& L-Sim
& \cellcolor[rgb]{0.87,0.88,0.91}\underline{75.26} & \cellcolor[rgb]{0.87,0.88,0.91}\underline{4.74} & \cellcolor[rgb]{0.87,0.88,0.91}\underline{1.94}
& \cellcolor[rgb]{0.87,0.88,0.91}\underline{77.62} & \cellcolor[rgb]{0.96,0.87,0.70}4.63 & \cellcolor[rgb]{0.87,0.88,0.91}\underline{1.37} \\

& Spectral
& \cellcolor[rgb]{0.96,0.87,0.70}74.70 & \cellcolor[rgb]{0.96,0.87,0.70}6.13 & \cellcolor[rgb]{0.96,0.87,0.70}2.50
& 76.39 & 6.16 & 2.60 \\

& $t$-Spanner
& 73.91 & 7.66 & 3.30
& 75.89 & 7.50 & 3.10 \\

& DSpar
& 72.03 & 7.16 & 5.17
& 74.40 & 7.16 & 4.60 \\

\midrule

\multirow{3}{*}{\rotatebox[origin=c]{90}{\scriptsize\shortstack{Sem\\guided}}}

& Unified-LTH
& 70.55 & 11.87 & 6.65
& 71.19 & 12.76 & 7.81 \\

& AdaGLT
& 70.37 & 9.89 & 6.83
& 71.76 & 11.37 & 7.23 \\

& MoG
& 72.29 & 7.58 & 4.91
& 74.02 & 8.58 & 4.97 \\

\midrule

& \textbf{EDiS-Lite}
& \cellcolor[rgb]{0.78,0.87,0.97}\textbf{75.95} & \cellcolor[rgb]{0.78,0.87,0.97}\textbf{2.53} & \cellcolor[rgb]{0.78,0.87,0.97}\textbf{1.25}
& \cellcolor[rgb]{0.78,0.87,0.97}\textbf{77.76} & \cellcolor[rgb]{0.78,0.87,0.97}\textbf{3.32} & \cellcolor[rgb]{0.78,0.87,0.97}\textbf{1.23} \\

\bottomrule
\end{tabular}
\end{table}

\textbf{(RQ2) How large is the average margin?}
Table~\ref{tab:main-results-summary} aggregates results over all 19 datasets. The mean benchmark score averages per-dataset scores, which are accuracy on 16 datasets and ROC-AUC on Minesweeper, Questions, and ogbn-proteins. Among the ranked methods, EDiS-Lite achieves the highest mean benchmark score at both budgets. Since rank and gap are computed per dataset relative to the best ranked method and then averaged, EDiS-Lite attains average ranks of $2.53$/$3.32$ and average gaps of only $1.25$/$1.23$ points at $\rho=0.3$/$0.5$, respectively. Per-dataset ranks are easier to interpret across the two metrics.

\textbf{(RQ3, EDiS-Lite vs.\ EDiS-Full) What does full-graph
inference recover over sparse-graph inference, and what does it
cost?}
\label{sec:full-graph-gap}

Both variants share the same trained model and differ only at
inference. At tight budgets, EDiS-Lite can trail significantly behind
the independently tuned Full GCN reference, most visibly on
Minesweeper and Pokec. Switching the evaluation graph to the full
graph (EDiS-Full) recovers much of it without retraining, and the
gap shrinks as $\rho$ grows (Figure~\ref{fig:full-graph-gap}).
Reddit is a mild exception, where EDiS-Lite already sits close to
Full GCN. That recovery is not free. Evaluation time rises and peak
GPU memory increases mildly, while training cost is unchanged. These
results make EDiS-Lite the default, with EDiS-Full reserved for
cases where the accuracy gain is worth the extra inference cost.
Full curves and the runtime and memory breakdown are in
Appendices~\ref{app:inference-profile} and~\ref{app:runtime-memory}.

\begin{figure}[h]
\centering
\includegraphics[width=\linewidth]{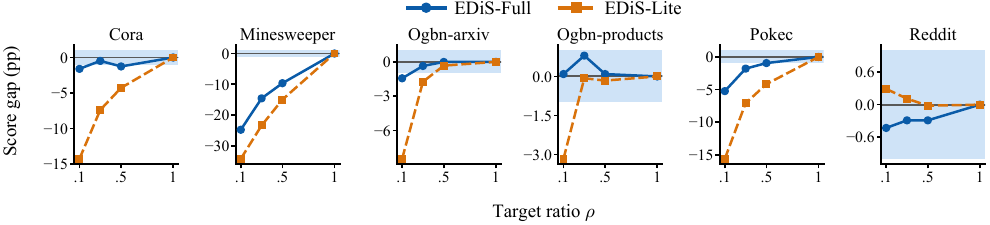}
\caption{\textbf{Full-graph evaluation across regimes.}  \textbf{MaxCF} with cosine scores.}
\label{fig:full-graph-gap}
\end{figure}

\vspace{-.15in}

\subsection{Ablations and Generality}
\label{sec:ablations}

\setlength{\intextsep}{3pt}

\textbf{(RQ4, EDiS-Lite vs.\ matched-budget controls) Do cacheable
subgraphs and per-epoch resampling drive the gains, or would any
same-budget sampling do as well?} We compare EDiS-Lite against three
controls that fix the backbone, scores, and budget. \textbf{\underline{Frozen}}
reuses one sampled graph for the entire run, \textbf{\underline{Direct}} samples
edges by score without forming subgraphs, and \textbf{\underline{Shuffled groups}}
keeps subgraph sizes constant but scrambles which edges belong to each
subgraph, thus separating structure from size.

\begin{figure}[H]
\centering
\includegraphics[width=\linewidth]{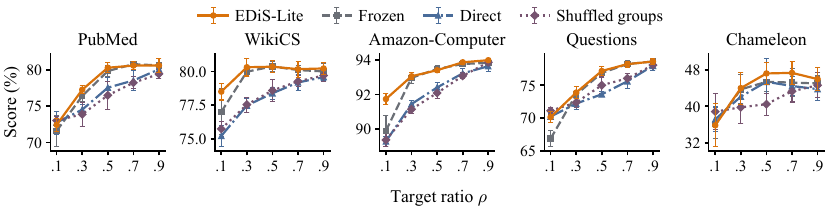}
\caption{\textbf{Decomposition and resampling across edge budgets.}}
\label{fig:mechanism-profile}
\end{figure}

Figure~\ref{fig:mechanism-profile} shows EDiS-Lite matching or
leading all three controls across the measured budget range, with
the largest margin at tight budgets where each epoch graph holds few
edges. Removing epoch variation (Frozen) hurts most under small
budgets. Dropping subgraph formation (Direct) leaves a gap that closes
only near the full graph. Shuffled groups, despite matching
EDiS's subgraph sizes exactly, is consistently the weakest control.
These results indicate that the gain comes from the structure of the
decomposition and from resampling it every epoch, not merely from
the number or size
of edges available. Full results are given in
Appendix~\ref{app:mechanism-profile}.

\textbf{(RQ5, selectors and edge scores) How sensitive is \edis{}'s predictive performance to the residual selector and edge-scoring function?}
Theorem~\ref{thm:app_budget_retention} applies across selectors, but does not establish comparable predictive performance, so we evaluate this empirically
by sweeping seven selectors (MaxCF, MinCF, Degree, $k$-NN,
Spanner, Random, Hybrid) and five edge scores (cosine, Euclidean,
Dot, Uniform, MLP-score) through the same composition procedure, changing nothing else.
\begin{figure}[!ht]
\centering
\includegraphics[width=0.8\linewidth]{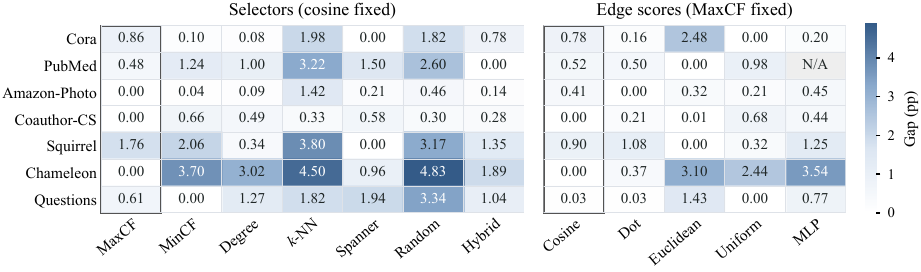}
\caption{\textbf{Selector and score sensitivity.} GCN at $\rho=0.3$.
Cells show the gap to the best mean per row and panel (pp, lower is
better). Outlines mark defaults. N/A marks PubMed--MLP-score, which is
not applicable.}
\label{fig:selector-gap}
\end{figure}
Figure~\ref{fig:selector-gap} reports the observed sensitivity. Across all seven
selectors and five scores, most cells sit within a couple of points
of the best choice on a given dataset, and no single selector or
score dominates. The default MaxCF and cosine pairing stays
competitive throughout. This extends to other
datasets and to the GraphSAGE backbone
(Appendix~\ref{app:selector-profiles}--\ref{app:score-profiles},
\ref{app:sage-robustness}).

\textbf{(RQ6) How does extraction depth affect a fixed-budget
graph?} The cap $K_{\max}$ sets how many subgraphs we extract, and the
budget $B$ sets how many edges each epoch uses. A larger cap moves
edges from the residual pool into subgraphs. We vary $K_{\max}$ on five
large graphs at $\rho=0.3$ (Figure~\ref{fig:kmax-sensitivity}). More
subgraphs help on some datasets and hurt on others, and Lite and Full
do not always agree. Thus, additional subgraphs do not always give a better
graph, and once extraction stops before the cap, a larger $K_{\max}$
changes nothing (Appendix~\ref{app:kmax-sensitivity}).

\begin{figure}[H]
\centering
\includegraphics[width=\linewidth]{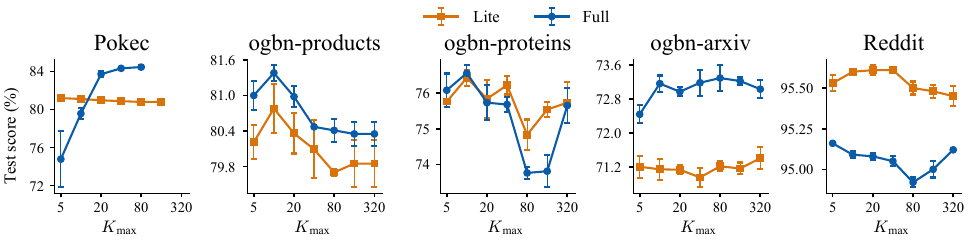}
\caption{\textbf{Extraction depth at $\rho=0.3$.} Reported means
with standard-deviation bars. Proteins uses ROC-AUC, others
accuracy. Each panel has its own scale.}
\label{fig:kmax-sensitivity}
\end{figure}

\textbf{(RQ7) Does the method work only with GCN?} We fix MaxCF and
cosine-similarity but switch the backbone from GCN to GraphSAGE
(Figure~\ref{fig:backbone-profile}). Both backbones improve with the
budget and stay close, and neither is better overall. \edis{}
therefore works with either backbone (Appendix~\ref{app:backbone-profile}).

\begin{figure}[H]
\centering
\includegraphics[width=0.9\linewidth]{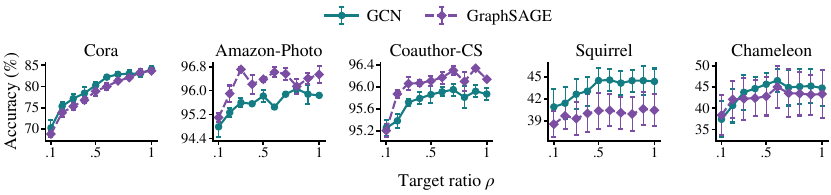}
\caption{\textbf{Backbone comparison within EDiS-Lite.}}
\label{fig:backbone-profile}
\end{figure}

\textbf{Further analysis.}\quad The appendix reports more
experiments. It covers runtime and memory by budget
(Appendix~\ref{app:runtime-memory}), a cost analysis of construction and
composition (Appendix~\ref{app:complexity}), the full ablation study
(Appendix~\ref{app:additional-profiles}), complete per-dataset results
(Appendix~\ref{app:experimental}), and its limitations
(Appendix~\ref{app:limitations}).

\section{Conclusion}

We introduced \edis{}, a framework that separates one-time structural
extraction from per-epoch graph composition. 
We extract structure once, then compose many sparse
graphs from it for training, avoiding both a fixed topology and repeated
extraction. The same cached library supplies exact-budget graphs at
every epoch and retention level without changing the GNN architecture
or objective. Our default MaxCF construction is just one instance of
this framework, which stays modular across scores and selectors.

Across 19 node-classification benchmarks under the same edge-budget,
EDiS-Lite achieves the highest mean benchmark score among the ranked
methods at both retention levels.
EDiS-Full, reported as a reference using full-graph for inference,
shows how much accuracy can be recovered when sparsification
is restricted to training, and at a modest extra cost for inference. 
Structural decomposition and epoch-to-epoch variation matter most at tighter
edge-budgets. More broadly, a cacheable structural representation that generates
many sparse views can be more effective than searching for a single
sparse graph to replace the original.
Our future work will consider task-aware scores for decomposition, dynamic (temporal) composition, and extensions to graph- and link-level
prediction.

\newpage

\section*{AI Use Disclosure}
Generative AI tools were used to assist with polishing the manuscript and
supporting research execution, including software implementation and experimental workflows. All AI-assisted content and code were
reviewed, revised, and validated by the authors. The authors independently
verified the reported results and claims and take full responsibility for the
final content of this work.

\section*{Ethics Statement}
This work does not involve human subjects, personal data, or new data collection. All experiments
use publicly available benchmark graphs under their original licenses. We are not aware of ethical
risks specific to this methodology beyond those inherent to downstream applications of graph
learning.

\bibliographystyle{iclr2027_conference}
\bibliography{reference}

\let\addcontentsline\savedaddcontentsline
\renewcommand{\contentsname}{Appendix Contents}

\newpage
{\footnotesize\linespread{0.85}\selectfont
\hypersetup{linkcolor=black}
\tableofcontents
}
\newpage
\appendix
\section{Notation}
\label{app:notation}

This appendix collects the notation used throughout \edis{}, including
the symbols introduced for the proofs in Appendix~\ref{app:proofs}.
Table~\ref{tab:notation} lists the main symbols, grouped by where each
is introduced. Symbols that appear only in a single appendix result
are instead defined locally, at the point where they are used. The
letter $H$ is intentionally reused as a generic placeholder for the
current graph.

\begin{table}[H]
\centering
\footnotesize
\setlength{\tabcolsep}{5pt}
\renewcommand{\arraystretch}{0.9}

\rowcolors{2}{white}{black!3}
\begin{tabularx}{\columnwidth}{
    >{\raggedright\arraybackslash}p{0.33\columnwidth}
    >{\raggedright\arraybackslash}X
}
\toprule

\textbf{Symbol} & \textbf{Meaning} \\

\midrule

\multicolumn{2}{@{}l}{\textit{\textbf{Graph and edge-budget}}} \\
\addlinespace[0.5pt]

$\mathcal{G}=(\mathcal{V},\mathcal{E},\mathbf{X})$
&
Attributed undirected graph ($\mathcal V,\mathbf X$ fixed; only the
edge set varies below)
\\

$n=|\mathcal{V}|$, $m=|\mathcal{E}|$
&
Number of nodes and edges
\\

$\rho$
&
Edge-retention ratio
\\

$B=\lceil \rho m\rceil$
&
Number of active edges per epoch
\\

\addlinespace[0.5pt]
\midrule
\multicolumn{2}{@{}l}{\textit{\textbf{One-time edge decomposition}}} \\
\addlinespace[0.5pt]

$s_e$
&
Score assigned to edge $e$
\\

$\mathcal F$, $\mathcal F_k$
&
Residual selector ($\mathcal F_k$: round-dependent schedule)
\\

$\mathcal{E}_{\mathrm{rem}}^{(k)}$
&
Edges remaining after decomposition round $k$
\\

$\mathcal{E}_k$
&
Subgraph extracted at round $k$
\\

$K$
&
Number of extracted subgraphs
\\

$K_{\max}$
&
Cap on extraction rounds ($K\le K_{\max}$)
\\

$\mathcal{E}_{\mathrm{res}}$
&
Residual edges after the final round
\\

$U_k=\mathcal E_1\cup\cdots\cup\mathcal E_k$
&
Union of the first $k$ subgraphs ($U_K$: full stored union)
\\

\addlinespace[0.5pt]
\midrule
\multicolumn{2}{@{}l}{\textit{\textbf{Per-epoch graph construction}}} \\
\addlinespace[0.5pt]

$T$ & Total number of training epochs \\
$\operatorname{Decompose}$ & One-time decomposition of $\mathcal G$ under
score $s$ and selector $\mathcal F$ into the stored $\mathcal D$ \\
$\operatorname{Compose}$ & Budgeted composition, mapping the stored decomposition and $B$
to $\mathcal{E}_B^{(t)}$ \\
$\mathcal D$ & Edge partition with cached subgraph weights, reused across edge-budgets \\
$w_k$ & Positive subgraph weight $\max\{\sum_{e\in\mathcal E_k}s_e,\varepsilon\}$ \\
$R_r$, $k_r$ & Remaining subgraph indices and the subgraph selected at draw $r$ \\
$\xi_t$ & Epoch randomness, including $\pi^{(t)}$ and within-pool draws \\
$A^{(t)}$ & Subgraph indices selected at epoch $t$ \\
$\widetilde{\mathcal{E}}_k^{(t)}$ & Selected subgraph edges
remaining after trimming \\
$Q_B^{(t)}$ & Residual edges used to complete the per-epoch edge-budget \\
$\pi^{(t)}$ & Weighted subgraph order for epoch $t$ \\

$\mathcal{E}_B^{(t)}$
&
Active edge set at epoch $t$
\\

$\mathcal{G}_B^{(t)}$
&
Graph for message passing at epoch $t$ ($\mathcal V,\mathbf X$ unchanged)
\\

$\widehat{\mathbf A}_B^{(t)}$
&
Normalized adjacency of $\mathcal G_B^{(t)}$ used for message passing
\\

\addlinespace[0.5pt]
\midrule
\multicolumn{2}{@{}l}{\textit{\textbf{Score-aware guarantee}}} \\
\addlinespace[0.5pt]

$\tau$
&
Edge-score threshold
\\

$\mathcal{E}_{\ge\tau}
 =\{e\in\mathcal{E}:s_e\ge\tau\}$
&
Edges with score at least $\tau$
\\

$\mathcal{G}_{\ge\tau}=(\mathcal V,\mathcal{E}_{\ge\tau})$,
$U_{k,\ge\tau}=U_k\cap\mathcal{E}_{\ge\tau}$
&
Threshold graph and the stored union restricted to it
\\

$a_j(S)$, $b_j(S)$
&
$j$-th largest crossing score at cut $S$ in $\mathcal G$ and in $U_k$
\\

$T_k(S;A)=\textstyle\sum\nolimits_{j=1}^{\min\{k,|\delta_A(S)|\}}s_{(j)}(\delta_A(S))$
&
Sum of the $\min\{k,\cdot\}$ largest scores crossing $S$ in $A$
\\

$\delta_H(S)$
&
Edges of $H$ crossing the cut $(S,\mathcal{V}\setminus S)$
\\

$d_H(v)$ & Degree of $v$ in $H$ \\

\addlinespace[0.5pt]
\midrule
\multicolumn{2}{@{}l}{\textit{\textbf{Proof-specific notation
(Appendix~\ref{app:proofs})}}} \\
\addlinespace[0.5pt]

$s$ & Additive score function, extending $s_e$ to edge sets \\
$\omega(F)=\sum_{e\in F}\omega_e$ & Generic additive edge-weight function \\
$D_k$ & Stored subgraph $k$, with $D_0=\mathcal E_{\mathrm{res}}$ and
$D_k=\mathcal E_k$ for $k\ge1$ \\
$n_k=|D_k|$ & Size of subgraph $k$ \\
$r_k$ & Number of edges retained from subgraph $k$ \\
$\theta_k=r_k/n_k$ & Retained fraction / inclusion probability given
the order (0 if the pool is empty) \\
$x_k(S,\tau)=|D_k\cap\delta_{\mathcal G}(S)\cap\mathcal E_{\ge\tau}|$ &
Edges of subgraph $k$ crossing $S$ with score at least $\tau$ \\

\bottomrule
\end{tabularx}
\caption{Main notation used in \edis{}.}
\label{tab:notation}
\end{table}

\section{Extended Related Work}
\label{app:related_work}

\paragraph{Sampling-based GNN training.}

Many GNN methods reduce training cost by processing only part of the
graph at a time rather than constructing a new topology. Neighbor
sampling selects a limited number of neighbors per node at every GNN
layer \citep{Hamilton:2017tp}. Layer-wise sampling instead selects a
fixed number of nodes by importance \citep{Chen:2018fastgcn}. LADIES
extends this idea by picking each layer's nodes based on which nodes
were kept in the layer above it, so the sampled nodes stay connected
from one layer to the next \citep{Zou:2019ladies}. Other
methods divide the graph into clusters and train on these smaller
parts \citep{Chiang:2019clustergcn}, or sample training subgraphs and
pass messages only within them \citep{Zeng:2020graphsaint}. AGS-GNN
\citep{Das:2024agsgnn} guides node and subgraph sampling by feature
similarity and diversity, and reappears below as an edge-disjoint
variant. Related approaches reuse historical node representations
\citep{Fey:2021gnnautoscale},
quantize activations \citep{Ding:2021vqgnn} or compress them further
\citep{Liu:2022exact}, or repartition the graph across workers
\citep{Wan:2022bnsgcn} or compensate for missing cross-partition
messages \citep{Shi:2023lmc}.

These methods reduce how much of the graph is processed at once, but
the resulting structures are repeatedly resampled from the original
topology. \edis{} instead organizes the topology itself and controls
which precomputed subgraphs are used under a specified
edge-retention budget.

\paragraph{Graph sparsification and topology modification.}

A second group of methods changes which edges are used for message
passing. Sparsification methods remove edges while preserving
structural or task-relevant information. Classical topology-based
approaches include uniform random removal, Local Degree
\citep{Hamann:2016sparsification}, Rank Degree
\citep{Voudigari:2016rankdegree}, Forest Fire
\citep{Leskovec:2006sampling}, and SCAN \citep{Xu:2007scan}.
Similarity-based approaches include L-Spar, G-Spar, and L-Sim, which
retain edges using neighborhood similarity \citep{Satuluri:2011lsim}.
Other structural approaches include spectral sparsification via
effective resistance \citep{Spielman:2008effres}, $t$-spanners
\citep{Peleg:1989spanners}, and degree-based DSpar
\citep{Liu:2023dspar}. A classical line of work builds cut sparsifiers
by repeatedly extracting a maximum spanning forest from what is left
of the graph, which gives provable guarantees on how well cut sizes
are preserved \citep{Benczur:1996cutsparsification,ForsterDeVos2023}. \edis{} reuses this same
repeated-forest-extraction idea, but keeps each extracted forest as a
separate, cacheable training subgraph instead of merging them into
one fixed sparsifier.

A separate line learns or adapts which elements to retain. UGS
\citep{Chen:2021ugs}, from the Unified Lottery Ticket framework,
prunes edges and model parameters together over repeated training
rounds. TEDDY \citep{TEDDY} instead removes edges in a single pass
using a degree-based rule, avoiding UGS's repeated pruning cycles.
AdaGLT \citep{Zhang:2024adaglt} adds adaptive, automated
sparsification. SGS-GNN \citep{SGS-GNN} learns which edges to keep
using supervision and lets the user directly set the target sparsity.
MoG \citep{Zhang:2025mog} combines sparsification experts with
different pruning criteria and sparsity levels into task-adaptive
graphs.

Different sparsification methods preserve different graph properties,
and no single method preserves all of them well
\citep{Chen:2023sparsifybench}. Removing edges therefore affects not
only training cost but also how information propagates.
Sparsification has also been used against oversmoothing and
oversquashing, via spectrum-preserving pruning
\citep{Liang:2025spectralsparse,Jamadandi:2024spectralpruning} or a
truss-based approach that removes edges outside dense subgraphs
\citep{Hossain:2024trusssparse}. These methods
generally focus on identifying or adapting a useful sparse topology
for message passing, rather than on how that topology is exposed to
the GNN across training steps.

A related line changes the topology stochastically during training.
DropEdge removes a random subset of edges at each iteration
\citep{Rong:2020dropedge}. Even feature dropout, which does not
explicitly remove edges, induces different
stochastic subgraphs across feature dimensions. A theoretical
analysis attributes part of dropout's ability to mitigate
oversmoothing in GCNs to this variation \citep{Luo:2025dropoutgcn}.
Other methods replace random removal with learned edge predictors
\citep{Chen:2025adedgedrop}, or add parameters to stabilize message
passing under changing structures \citep{Lee:2025aggbuffer}.

Structures generated by DropEdge or dropout-induced subgraphs at
different training steps can substantially overlap in the edges they
use. \edis{} instead decomposes the graph once and reuses the
resulting non-overlapping subgraphs, composing them under an
explicit edge-retention budget that lets the message-passing topology
vary while directly controlling the number of retained edges.

\paragraph{Cacheable edge-disjoint decomposition.}

A smaller body of work trains GNNs using multiple graph structures
rather than one sparse topology. The AGS-GNN sampling framework
\citep{Das:2024agsgnn} introduced above also includes AGS-GS-Disjoint,
which sequentially constructs sparse subgraphs with different
heuristics, removing selected edges before constructing the next.

Our objective still differs, and so does the intent behind the
construction itself. AGS-GS-Disjoint's edge-disjoint subgraphs feed a
dual-channel sampler that uses submodular function maximization to
select similar and diverse neighborhoods, adapting to homophily and
heterophily. Its per-epoch sample draws a fixed number of these
cached subgraphs padded with random residual edges, with no
exact-edge-count guarantee or attached structural theory. \edis{}
instead studies edge-disjoint decomposition as a general
sparsification primitive. Any residual selector composes into an
exact-$B$-edge training graph, with a proven cut-preservation
guarantee for the stored decomposition and a retention guarantee for
the composed graph that holds regardless of selector. This lets us
directly compare organizing a fixed number of retained edges into
non-overlapping subgraphs against alternative ways of using the same
budget, rather than optimizing for informative neighborhoods.

SpanGNN \citep{Gu:2024spangnn} trains using spanning subgraphs to
reduce memory, and MoG \citep{Zhang:2025mog} combines sparsification
experts to produce adaptive sparse graphs. GERN \citep{Bonchi:2025gern}
also exposes a GNN to multiple sparse structures, progressively
training over a sequence of random spanning trees transformed into
path graphs. Unlike \edis{}, it does not build a persistent
edge-disjoint repository, progressing instead through independently
generated sparse path structures toward efficient training rather
than budget-controlled composition of a fixed repository. To our
knowledge, none of them partitions the original edge set into a
persistent, mutually edge-disjoint collection that is later
recombined under a retained-edge constraint.

These approaches confirm that exposing a GNN to multiple sparse
structures can help, though each still addresses a different question
from \edis{}, namely how to allocate a fixed retained-edge budget across a
cacheable collection of non-overlapping subgraphs. The graph is
decomposed once, unused edges remain explicitly accounted for in a
residual subgraph, and the subgraphs are recomposed under the same
edge-budget constraint at every epoch, so different training-graph
constructions can be compared under the same retained-edge count.
Table~\ref{tab:relwork-comparison} compares these methods along these
axes.

\providecommand{\good}{\ding{51}}
\providecommand{\bad}{\ding{55}}
\begin{table}[H]
\centering
\tiny
\setlength{\tabcolsep}{3.5pt}
\renewcommand{\arraystretch}{0.87}
\rowcolors{2}{white}{black!3}
\begin{tabularx}{\linewidth}{@{}X l c c c c@{}}
\toprule
\textbf{Method} & \textbf{Type} &
\makecell{\textbf{Edge}\\\textbf{budget}} &
\textbf{Reusable} &
\makecell{\textbf{Multi-}\\\textbf{struct.}} &
\makecell{\textbf{Edge-}\\\textbf{disjoint}} \\
\midrule
$t$-Spanner \citep{Peleg:1989spanners} & Sparsification & \good & \good & \bad & N/A \\
Forest Fire \citep{Leskovec:2006sampling} & Sampling & \bad & \bad & \bad & N/A \\
SCAN \citep{Xu:2007scan} & Sparsification & \good & \good & \bad & N/A \\
Spectral \citep{Spielman:2008effres} & Sparsification & \good & \good & \bad & N/A \\
L-Spar \citep{Satuluri:2011lsim} & Sparsification & \good & \good & \bad & N/A \\
G-Spar \citep{Satuluri:2011lsim} & Sparsification & \good & \good & \bad & N/A \\
L-Sim \citep{Chen:2023sparsifybench} & Sparsification & \good & \good & \bad & N/A \\
Local Degree \citep{Hamann:2016sparsification} & Sparsification & \good & \good & \bad & N/A \\
Rank Degree \citep{Voudigari:2016rankdegree} & Sampling & \bad & \bad & \bad & N/A \\
Neighbor Sampling \citep{Hamilton:2017tp} & Sampling & \bad & \bad & \good & \bad \\
FastGCN \citep{Chen:2018fastgcn} & Sampling & \bad & \bad & \good & \bad \\
LADIES \citep{Zou:2019ladies} & Sampling & \bad & \bad & \good & \bad \\
Cluster-GCN \citep{Chiang:2019clustergcn} & Sampling & \bad & \good & \good & \bad \\
DropEdge \citep{Rong:2020dropedge} & Stochastic & \bad & \bad & \good & \bad \\
GraphSAINT \citep{Zeng:2020graphsaint} & Sampling & \bad & \bad & \good & \bad \\
Unified-LTH / UGS \citep{Chen:2021ugs} & Learned sparsif. & \good & \good & \bad & N/A \\
Cut Sparsification \citep{ForsterDeVos2023} & Sparsification & \good & \good & \bad & N/A \\
DSpar \citep{Liu:2023dspar} & Sparsification & \good & \good & \bad & N/A \\
AGS-GS-Disjoint \citep{Das:2024agsgnn} & Sampling & \bad & \good & \good & \good \\
TEDDY \citep{TEDDY} & Learned sparsif. & \good & \good & \bad & N/A \\
AdaGLT \citep{Zhang:2024adaglt} & Learned sparsif. & \good & \good & \bad & N/A \\
SpanGNN \citep{Gu:2024spangnn} & Multi-structure & \bad & \good & \good & \bad \\
Feature Dropout \citep{Luo:2025dropoutgcn} & Stochastic & \bad & \bad & \good & \bad \\
Learned Edge Predictors \citep{Chen:2025adedgedrop} & Learned stochastic & \bad & \bad & \good & \bad \\
Aggregation Buffering \citep{Lee:2025aggbuffer} & Stabilization & N/A & N/A & N/A & N/A \\
SGS-GNN \citep{SGS-GNN} & Learned sparsif. & \good & \good & \bad & N/A \\
MoG \citep{Zhang:2025mog} & Learned sparsif. & \good & \good & \good & \bad \\
GERN \citep{Bonchi:2025gern} & Multi-structure & \bad & \good & \good & \bad \\
\midrule
\textbf{EDiS (ours)} & Sparsification & \good & \good & \good & \good \\
\bottomrule
\end{tabularx}
\caption{Comparison of graph sampling, sparsification, and
multi-structure training methods surveyed in
Section~\ref{sec:RelatedWork} and Appendix~\ref{app:related_work}.
\emph{Budget}: exposes an explicit target edge count or ratio.
\emph{Reusable}: the structure is computed once and reused rather
than resampled from scratch every step. \emph{Multi-struct.}:
training exposes the GNN to more than one structural view.
\emph{Edge-disjoint}: multiple structures, when used, share no
edges. ``N/A'' marks a column that does not apply to a given method.}
\label{tab:relwork-comparison}
\end{table}

\section{Formal Definitions and Construction}
\label{app:edis_technical}

This appendix gives the formal definitions and exact composition rules
underlying the intuitive description in Section~\ref{sec:methodology}.

\subsection{Graph, GNN, and Edge Budget}
\label{app:edis_setup}

\begin{definition}[\underline{Graph}]
\label{def:graph}
Let $\mathcal G=(\mathcal V,\mathcal E,\mathbf X)$ be a finite,
loopless, undirected graph with $n=|\mathcal V|$ nodes,
$m=|\mathcal E|$ edges, and
$\mathbf X\in\mathbb R^{n\times d_0}$.
Each undirected edge is counted once.
\end{definition}

\begin{definition}[\underline{Graph neural network}]
\label{def:gnn}
A message-passing GNN updates node representations using their
neighbors. A graph convolutional layer \citep{Kipf:2017tc} uses
\begin{equation}
    \mathbf H_{\ell+1}
    =
    \sigma\bigl(
    \hat{\mathbf A}\mathbf H_{\ell}\mathbf W_{\ell}
    \bigr),
    \qquad
    \mathbf H_{0}=\mathbf X,
    \label{eq:prelim_gcn_layer}
\end{equation}
where $\hat{\mathbf A}$ is normalized adjacency,
$\mathbf W_{\ell}$ is a learned weight matrix, and $\sigma$ is a
nonlinearity.
\end{definition}

\begin{definition}[\underline{Edge budget}]
\label{def:edge_budget}
At epoch $t$, nodes and features remain fixed while training uses
\begin{equation}
    \mathcal G_B^{(t)}
    =
    (\mathcal V,\mathcal E_B^{(t)},\mathbf X),
    \qquad
    \mathcal E_B^{(t)}\subseteq\mathcal E,
    \qquad
    |\mathcal E_B^{(t)}|=B.
    \label{eq:prelim_epoch_budget}
\end{equation}
The integer budget satisfies $0\le B\le m$. A retention ratio
$\rho\in(0,1]$ sets $B=\lceil\rho m\rceil$. Indices $k$
and $t$ denote extraction rounds and training epochs.
\end{definition}

\subsection{Edge Scores and Residual Selectors}
\label{app:edis_scores}

\edis{} separates edge scoring from structural extraction.

\begin{definition}[\underline{Edge score}]
\label{def:edge_score}
A fixed score $s:\mathcal E\rightarrow\mathbb R$ assigns larger values
to preferred edges, computed once before extraction. Uniform scores
set $s_e=1$. Our default is shifted
cosine similarity between endpoint features
(Section~\ref{subsec:budget_construction}).
\end{definition}

\begin{definition}[\underline{Residual selector}]
\label{def:residual_selector}
At extraction round $k$, a residual selector $\mathcal F_k$ maps the
current residual graph and edge scores to
\begin{equation}
    \mathcal E_k
    =
    \mathcal F_k
    \bigl(
    \mathcal V,
    \mathcal E_{\mathrm{rem}}^{(k-1)},
    s
    \bigr)
    \subseteq
    \mathcal E_{\mathrm{rem}}^{(k-1)},
\end{equation}
after which
\begin{equation}
    \mathcal E_{\mathrm{rem}}^{(k)}
    =
    \mathcal E_{\mathrm{rem}}^{(k-1)}
    \setminus
    \mathcal E_k.
    \label{eq:prelim_residual_extraction}
\end{equation}
The selector may use $s$ or ignore it and need not return a forest or a
fixed number of edges.
\end{definition}

The framework initializes the residual edge set to
$\mathcal E_{\mathrm{rem}}^{(0)}=\mathcal E$ and pairs any fixed score
with any residual selector. Implemented scores include cosine,
Euclidean, dot-product, uniform, and learned preferences. Implemented
selectors include maximum- and minimum-weight covering forests,
uniform sampling, degree-based selection, top-$k$ neighbors,
spanners, and a hybrid chaining several of these
(Appendix~\ref{app:selector_family}).

\subsection{Maximum-Score Residual Forests}
\label{app:edis_forests}

\begin{definition}[\underline{Covering forest}]
Given $H=(\mathcal V,E_H)$, let
\begin{equation}
    \mathfrak F(H)
    =
    \left\{
    F\subseteq E_H:
    F\text{ is acyclic and connects every connected component of }H
    \right\}.
    \label{eq:spanning_forest_set}
\end{equation}
\end{definition}

\begin{definition}[\underline{Principal selector}]
\label{def:principal_selector}
The default selector $\mathcal F_{\max}$ returns a maximum-score
covering forest,
\begin{equation}
    F^\star
    \in
    \operatorname*{arg\,max}_{F\in\mathfrak F(H)}
    \sum_{e\in F}s_e.
    \label{eq:max_spanning_forest}
\end{equation}
\end{definition}

This forest is obtained by classical maximum-weight graphic-matroid
greedy \citep{Edmonds1971}, equivalently by Kruskal's algorithm in
decreasing score order \citep{Kruskal1956}. Repeating this extraction
produces a partition into successive residual forests, giving the
edge-disjoint decomposition and subgraph weights of
Section~\ref{subsec:budget_construction}
(Eq.~\eqref{eq:prelim_partition}). It does not assert
a minimum-cardinality forest partition. That stronger guarantee instead
comes from matroid partition \citep{Edmonds1968}.

The subgraphs, their weights, and the
residual pool are computed once and reused across all training epochs
and all feasible budgets. Because composition
(Appendix~\ref{app:edis_composition}) never re-invokes the selector,
this holds regardless of how many budget-epoch pairs $(B,t)$ are
subsequently queried, and distinct epochs need not compose the same
graph (Proposition~\ref{prop:reuse_variation},
Appendix~\ref{app:reuse_feasibility}).

\subsection{Classical Forest Properties}
\label{app:edis_forest_properties}

\begin{fact}[Bottleneck property \citep{Kruskal1956,Edmonds1971}]
\label{fact:bottleneck}
Let $e=(u,v)$ be an edge omitted by a maximum-score covering forest
$F^\star$, and let $P_{F^\star}(u,v)$ denote the unique path between
$u$ and $v$ inside the corresponding tree. Then
\begin{equation}
    \min_{a\in P_{F^\star}(u,v)} s_a
    \ge s_e.
    \label{eq:prelim_bottleneck}
\end{equation}
\end{fact}

\begin{fact}[Cut certificate
\citep{NagamochiIbaraki1992,FrankIbarakiNagamochi1993}]
\label{fact:cut_certificate}
Let $\mathcal E_1,\ldots,\mathcal E_k$ be covering forests extracted
successively from the residual graph. For every nonempty proper
$S\subsetneq\mathcal V$,
\begin{equation}
    |\delta_{U_k}(S)|
    \ge
    \min
    \left\{
    k,
    |\delta_{\mathcal G}(S)|
    \right\}.
    \label{eq:prelim_certificate}
\end{equation}
\end{fact}

The bottleneck property concerns a single maximum-score forest, whereas
the cut certificate concerns the union of successive residual forests.
Neither claims that $U_k$ is jointly optimal among all edge sets of the
same size.

\subsection{Exact Stochastic Budgeted Composition}
\label{app:edis_composition}

At epoch $t$, \edis{} draws an ordering $\pi^{(t)}=(k_1,\ldots,k_K)$
without replacement, where, with $R_r$ the subgraph indices remaining
before draw $r$,
\begin{equation}
    \Pr(k_r=k\mid R_r)
    =
    \frac{w_k}
    {\sum_{j\in R_r}w_j},
    \qquad k\in R_r.
    \label{eq:app_group_draw_probability}
\end{equation}
With $N=|U_K|$, subgraphs accumulate in sampled order until their total
size first reaches or exceeds $B$, at round
\begin{equation}
    r_B^{(t)}
    =
    \min
    \left\{
    r\in\{1,\ldots,K\}:
    \textstyle\sum_{\ell=1}^{r}
    |\mathcal E_{k_\ell}|
    \ge B
    \right\},
    \label{eq:app_stopping_index}
\end{equation}
giving the selected subgraphs and their candidate edges
\begin{equation}
A^{(t)}
=
\begin{cases}
\varnothing, & B=0,\\
\{k_1,\ldots,k_{r_B^{(t)}}\}, & 0<B\le N,\\
\{1,\ldots,K\}, & N<B\le m,
\end{cases}
\qquad
C_B^{(t)}
=
\bigsqcup_{k\in A^{(t)}}\mathcal E_k.
\label{eq:app_selected_groups}
\end{equation}

If $|C_B^{(t)}|>B$, the excess $\Delta_B^{(t)}=|C_B^{(t)}|-B$ is
removed from selected subgraphs in increasing $(w_k,k)$ order, dropping
whole subgraphs until the remaining excess fits inside one subgraph, then
trimming that subgraph uniformly without replacement, so at most one
selected subgraph is partially retained. If $B>N$ instead, all stored
edges are kept and the remaining $B-N$ edges are sampled uniformly
without replacement from $\mathcal E_{\mathrm{res}}$ into a residual
completion set $Q_B^{(t)}$, otherwise $Q_B^{(t)}=\varnothing$.

Writing $\widetilde{\mathcal E}_k^{(t)}$ for the retained part of
subgraph $k$ after trimming (empty for $k\notin A^{(t)}$), the final
epoch graph is
\begin{equation}
    \boxed{
    \mathcal E_B^{(t)}
    =
    \left(
    \bigsqcup_{k=1}^{K}
    \widetilde{\mathcal E}_k^{(t)}
    \right)
    \sqcup
    Q_B^{(t)}.
    }
    \label{eq:app_epoch_composition}
\end{equation}

\begin{proposition}[Exact-budget feasibility]
\label{prop:exact_budget}
For every integer $0\le B\le m$,
\begin{equation}
    |\mathcal E_B^{(t)}|=B.
\end{equation}
\end{proposition}

The edge-disjoint partition prevents duplicates, and when $B>N$ the
residual pool always contains at least $B-N$ edges because
$m-N\ge B-N$. A complete proof is given in
Appendix~\ref{app:reuse_feasibility}.

\subsection{Structural Preservation of the Stored Decomposition}
\label{app:edis_score_certificate}

To track only the edges a threshold cares about, fix $\tau\in\mathbb
R$ and focus only on the high-score subgraph and its stored
counterpart,
\[
    \mathcal E_{\ge\tau}
    =
    \{e\in\mathcal E:s_e\ge\tau\},
    \qquad
    \mathcal G_{\ge\tau}
    =
    (\mathcal V,\mathcal E_{\ge\tau}),
    \qquad
    U_{k,\ge\tau}
    =
    U_k\cap\mathcal E_{\ge\tau}.
\]
Theorem~\ref{thm:score_threshold_certificate} in
Section~\ref{sec:methodology} states that every stored prefix
certifies a lower bound on the number of threshold-qualified crossing
edges within this restricted view. For every
completed round, threshold, and nontrivial cut,
\[
    |\delta_{U_{k,\ge\tau}}(S)|
    \ge
    \min
    \left\{
    k,
    |\delta_{\mathcal G_{\ge\tau}}(S)|
    \right\}.
\]
The exchange proof and the treatment of ties are given in
Appendix~\ref{app:score_certificate}.

The threshold certificate above only bounds how many high-score
edges survive. To state a stronger guarantee on their exact values,
let $a_j(S)$ and $b_j(S)$ denote the $j$-th largest crossing score
at cut $S$ in $\mathcal E$ and in $U_k$, respectively.

\begin{corollary}[Top-$k$ cut-score preservation]
\label{cor:topk_cutscore}
Under the assumptions of
Theorem~\ref{thm:score_threshold_certificate},
\begin{equation}
    b_j(S)=a_j(S),
    \qquad
    1\le
    j\le
    \min\{k,|\delta_{\mathcal G}(S)|\}.
    \label{eq:topk_cutscore}
\end{equation}
\end{corollary}

Thus every cut's top-$k$ crossing-score values are preserved with
multiplicity. The proof is given in
Appendix~\ref{app:topk_corollaries}.

\subsection{Retention After Budgeted Composition}
\label{app:edis_retention}

The deterministic certificates above concern the stored prefix $U_k$.
A composed epoch graph may omit or trim subgraphs, so it need not contain
that prefix. Theorem~\ref{thm:app_budget_retention} in
Section~\ref{sec:methodology} instead bounds the conditional probability
that a fixed cut keeps at least one qualifying crossing edge in $\mathcal
E_B^{(t)}$. We now state the exact identity behind that guarantee.

Condition on the sampled order $\pi^{(t)}$ together with the
decomposition, subgraph weights, and budget, leaving uniform trimming
or residual completion as the remaining source of randomness. We index the residual
set alongside the $K$ stored subgraphs by writing
\[
    D_0=\mathcal E_{\mathrm{res}},
    \qquad
    D_k=\mathcal E_k\ (1\le k\le K),
    \qquad
    n_k=|D_k|,
    \qquad
    r_k=
    |D_k\cap\mathcal E_B^{(t)}|,
\]
so $D_0,\ldots,D_K$ are the $K+1$ sets composition draws from, $n_k$
is each one's size, and $r_k$ counts how many of its edges the
composed graph actually keeps. The retained fraction is
\begin{equation}
    \theta_k
    =
    \begin{cases}
        r_k/n_k,&n_k>0,\\
        0,&n_k=0,
    \end{cases}
\end{equation}
and, for cut $S$ and threshold $\tau$, the number of subgraph-$k$ edges
that would matter for that cut is
\begin{equation}
    x_k(S,\tau)
    =
    |
    D_k
    \cap
    \delta_{\mathcal G}(S)
    \cap
    \mathcal E_{\ge\tau}
    |.
    \label{eq:composition_defs_1}
\end{equation}
Given the sampled order, every $n_k$ and $r_k$ is fixed, and at most
one subgraph contributes a randomly trimmed subset.

\begin{theorem}[\underline{Conditional retention under composition, in full}]
\label{thm:budget_retention_full}
For any fixed additive edge weight $\omega$,
\begin{equation}
    \mathbb E
    \left[
    \omega(\mathcal E_B^{(t)})
    \mid
    \pi^{(t)}
    \right]
    =
    \sum_{k=0}^{K}
    \theta_k\omega(D_k).
\end{equation}
Moreover,
\begin{equation}
\Pr\!\left(
\mathcal E_B^{(t)}
\cap
\delta_{\mathcal G}(S)
\cap
\mathcal E_{\ge\tau}
=
\varnothing
\;\middle|\;
\pi^{(t)}
\right)
=
\prod_{k=0}^{K}
\frac{
\binom{n_k-x_k(S,\tau)}{r_k}
}{
\binom{n_k}{r_k}
}.
\label{eq:budget_retention_cut}
\end{equation}
\end{theorem}

\begin{corollary}[Per-cut high-score edge survival]
\label{cor:highscore_survival}
Under the same conditioning,
\begin{equation}
\Pr\!\left(
\mathcal E_B^{(t)}
\cap
\delta_{\mathcal G}(S)
\cap
\mathcal E_{\ge\tau}
\neq
\varnothing
\;\middle|\;
\pi^{(t)}
\right)
\ge
1-
\exp
\left(
-\sum_{k=0}^{K}
\theta_kx_k(S,\tau)
\right).
\label{eq:highscore_survival_cor}
\end{equation}
This is the bound displayed as
Theorem~\ref{thm:app_budget_retention} in
Section~\ref{sec:methodology}.
\end{corollary}

The quantity $\sum_{k=0}^{K}\theta_kx_k(S,\tau)$ is the conditional
expected number of retained high-score crossings. Neither result uses
any property specific to maximum-score forests, since composition
retains a uniform random subset within each partially trimmed subgraph
regardless of selector. Both therefore hold unchanged for every
selector in Appendix~\ref{app:selector_family}. These results
characterize what survives budgeted composition. They do not by
themselves imply that weighted ordering is superior to uniform
ordering, or that spreading edges across subgraphs is always beneficial.
Those questions are tested empirically in
Section~\ref{sec:experiments}, with proofs in
Appendix~\ref{app:budget_theory}.

\section{Worked Examples of Structural Selectors}
\label{app:selector-worked-example}
Seven selectors share the same scored input, three extraction rounds,
and six-edge sample budget. Figure~\ref{fig:selector-worked-example}
shows the stored subgraphs and two six-edge samples from the same decomposition.

\begin{figure}[H]
\centering
\includegraphics[width=\linewidth]{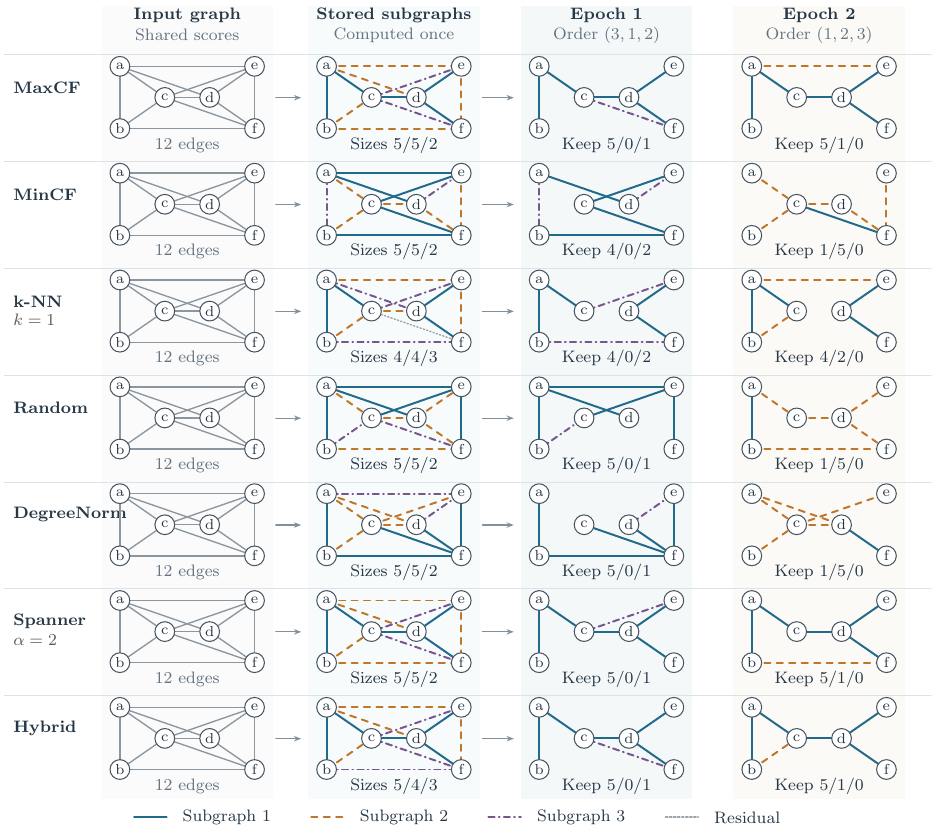}
\caption{\textbf{Structural selection and exact-budget composition in EDiS.}
Colors and line styles identify edge-disjoint subgraphs, extracted in
order $(1,2,3)$ and reused across epochs. The two illustrative epochs
sample subgraph orders $(3,1,2)$ and $(1,2,3)$, adding subgraphs until six
edges are reached and trimming any excess uniformly from the
lowest-weight subgraph. Samples may overlap and are not guaranteed to
differ. These examples illustrate the rules, not predictive rankings.}
\label{fig:selector-worked-example}
\end{figure}

\noindent\textbf{Settings.}
MaxCF selects a residual covering forest with the largest total score.
MinCF selects one with the smallest total score. The $k$-NN example uses
$k=1$ without adjusting the edge count. Random and DegreeNorm select up
to five edges per round. Random gives each residual edge equal weight.
DegreeNorm uses $1/d_H(u)+1/d_H(v)$, held fixed during each round.
The example spanner processes edges from shortest to longest using
length $1/s_e$. It adds an edge when the current path between its
endpoints is longer than $2/s_e$. This illustrative procedure differs
from the randomized benchmark routine. Hybrid applies MaxCF, then
$1$-NN, then DegreeNorm.

The table gives the shared edge scores. Each subgraph weight is the sum
of its edge scores.
\begin{center}
\small
\setlength{\tabcolsep}{5pt}
\begin{tabular}{lrrrrrrrrrrrr}
\toprule
Edge & $ab$ & $ac$ & $bc$ & $de$ & $df$ & $ef$ & $cd$ & $ae$ & $bf$ & $ad$ & $ce$ & $cf$ \\
Score & 12 & 11 & 10 & 9 & 8 & 7 & 6 & 5 & 4 & 3 & 2 & 1 \\
\bottomrule
\end{tabular}
\end{center}
MaxCF and Spanner produce the same subgraphs in this example.
The $k$-NN decomposition leaves one residual edge. No sample needs residual completion. Guarantees for stored
subgraphs may not hold after composition.
Appendix~\ref{app:selector_family} describes each selector's procedure.

\clearpage
\section{Theoretical Understanding and Proofs}
\label{app:proofs}
\begingroup
\small
\setlength{\abovedisplayskip}{3pt plus 1pt minus 1pt}
\setlength{\belowdisplayskip}{3pt plus 1pt minus 1pt}
\setlength{\abovedisplayshortskip}{1pt plus 1pt}
\setlength{\belowdisplayshortskip}{1pt plus 1pt}
\setlength{\topsep}{0pt plus 1pt minus 1pt}
\setlength{\partopsep}{0pt}

A cut's crossing edges are the only channel through which one layer
of message passing can move information between its two sides, so
the results below characterize threshold-qualified crossings in stored
maximum-score forest prefixes and their conditional retention under
composition.

All graphs are finite, undirected, and loopless, with $n\ge2$
vertices, and every undirected edge is counted once. Decomposition
results apply to each realized extraction sequence satisfying the
corresponding hypotheses. Ordinary cut certificates require successive
residual covering forests. Score-threshold, top-$k$ score, and
score-qualified path results require maximum-score residual covering
forests under fixed real scores. Randomized tie-breaking among maximizers
is allowed. Arbitrary randomized or non-forest selectors do not inherit
these guarantees. Sampling results condition on the realized
decomposition, cached weights, and edge-budget.

\subsection{Reuse and Exact-Budget Feasibility}
\label{app:reuse_feasibility}

\begin{proposition}[Cacheable stochastic support generation]
\label{prop:reuse_variation}
Fix a stored decomposition $\mathcal D$. Any number of subsequent
edge-budget--epoch queries can be answered without reevaluating the structural
selector. For fixed $B$, let $\mu_B$ be the support distribution induced
by composition. If epochs are independent draws from this same
distribution and $\mu_B$ is non-degenerate, then for $t\ne t'$,
\begin{equation}
\Pr(\mathcal E_B^{(t)}\ne\mathcal E_B^{(t')})
=1-\sum_{\mathcal S}\mu_B(\mathcal S)^2>0.
\label{eq:variation_identity}
\end{equation}
\end{proposition}
\begin{proof}
The selector is invoked only during decomposition, while composition
operates on the stored partition and weights. Independent draws from
$\mu_B$ coincide with probability $\sum_{\mathcal S}\mu_B(\mathcal S)^2$.
Taking its complement proves the identity. Non-degeneracy makes this
collision probability strictly less than one.
\end{proof}

\begin{proof}[Proof of Proposition~\ref{prop:exact_budget}]
The stored subgraphs and residual form an edge-disjoint partition of
$\mathcal E$, so every returned edge is original and no edge is counted
twice. For $B=0$, composition returns the empty set. If
$0<B\le |U_K|$, the selected subgraphs reach the edge-budget and trimming removes
exactly $|C_B^{(t)}|-B$ edges. If $|U_K|<B\le m$, all structural subgraphs
are retained and the residual supplies the deficit, since
$|\mathcal E_{\mathrm{res}}|=m-|U_K|\ge B-|U_K|$.
These cases include $K=0$ and yield exactly $B$ edges.
\end{proof}

\subsection{Score-Threshold Certificate}
\label{app:score_certificate}

Throughout this proof, $F_i$ denotes the $i$-th extracted subgraph
$\mathcal E_i$ and $U_k$ its stored union $\bigcup_{i\le k}\mathcal E_i$,
matching the maximum-score covering-forest notation of
Fact~\ref{fact:cut_certificate} under the default $\mathcal F_{\max}$
selector.

\begin{proof}[Proof of Theorem~\ref{thm:score_threshold_certificate}]
Fix $\tau$ and let $F_i^\tau:=F_i\cap\mathcal E_{\ge\tau}$.

Suppose $F_i^\tau$ does not span a component of
$\bigl(\mathcal V,\mathcal E_{\mathrm{rem}}^{(i-1)}\cap\mathcal E_{\ge\tau}\bigr)$.
Then some $e\in\mathcal E_{\mathrm{rem}}^{(i-1)}\cap\mathcal E_{\ge\tau}$
joins two components of $F_i^\tau$. Since $F_i$ is a covering forest of
the residual graph, $F_i+e$ contains a unique cycle, and since $e$
creates no cycle in $F_i^\tau$, that cycle contains some
$f\in F_i\setminus F_i^\tau$, so $s_f<\tau\le s_e$. Hence
\begin{equation}
 s(F_i-f+e)=s(F_i)-s_f+s_e>s(F_i),
 \label{eq:score_gain_app}
\end{equation}
contradicting the maximum-score optimality of $F_i$. Therefore
\begin{equation}
 F_i^\tau\in\mathfrak F\bigl(\mathcal V,\
 \mathcal E_{\mathrm{rem}}^{(i-1)}\cap\mathcal E_{\ge\tau}\bigr).
 \label{eq:restriction_is_forest}
\end{equation}

Also
\begin{equation}
 \mathcal E_{\mathrm{rem}}^{(i-1)}\cap\mathcal E_{\ge\tau}
 =\mathcal E_{\ge\tau}\setminus\bigcup_{j<i}F_j^\tau,
 \label{eq:residual_equality_app}
\end{equation}
so $F_1^\tau,\ldots,F_k^\tau$ are successive residual covering forests
of $\mathcal G_{\ge\tau}$ for every completed round $1\le k\le K$.
Applying Fact~\ref{fact:cut_certificate} inside $\mathcal G_{\ge\tau}$
to $U_{k,\ge\tau}=\bigcup_{i\le k}F_i^\tau$,
\begin{equation}
 \boxed{|\delta_{U_{k,\ge\tau}}(S)|\ \ge\ \min\{k,|\delta_{\mathcal G_{\ge\tau}}(S)|\}.}
 \label{eq:score_threshold_certificate_app}
 \qedhere
\end{equation}
\end{proof}

This exchange uses maximum-score optimality alone, so no sign
assumption on $s$ and no tie-breaking rule enter the argument. A
non-maximum-score covering forest still gives the ordinary count
certificate, but not the score-threshold form above.

The guarantee holds simultaneously for every nontrivial cut in the
stored prefix. Its crossing edges provide potential one-hop
message-passing connections between the two sides. Whether the retained
connections carry useful evidence for node classification depends on
the features, task, and model. The certificate itself concerns edge scores
and cut counts.

\subsection{Top-\texorpdfstring{$k$}{k} Cut-Score Corollary}
\label{app:topk_corollaries}

For an edge set $A$, define the sum of its largest crossing scores by
\[
 T_k(S;A)=\sum_{j=1}^{\min\{k,|\delta_A(S)|\}}s_{(j)}(\delta_A(S)).
\]
For a graph, $A$ denotes its edge set. Scores are ordered with
multiplicity, so tied edges contribute separately. Let
$a_1(S)\ge a_2(S)\ge\cdots$ be the scores of the original crossing edges
of a cut $S$ and $b_1(S)\ge b_2(S)\ge\cdots$ those kept in $U_k$.

\begin{corollary}[Top-$k$ cut-score preservation, restated]
\label{cor:topk_cutscore_restated}
For every completed round $1\le k\le K$ and nonempty proper cut $S$,
\begin{equation}
b_j(S)=a_j(S),\qquad 1\le j\le\min\{k,|\delta_{\mathcal G}(S)|\}.
\label{eq:topk_cutscore_restated}
\end{equation}
\end{corollary}

\begin{proof}[Proof of Corollary~\ref{cor:topk_cutscore_restated}]
Fix $S$ and write $M=|\delta_{\mathcal G}(S)|$. For $j\le\min\{k,M\}$
set $\tau=a_j(S)$. Then
\begin{equation}
 |\delta_{\mathcal G_{\ge\tau}}(S)|\ge j,
 \label{eq:topk_step1}
\end{equation}
so by Theorem~\ref{thm:score_threshold_certificate},
$|\delta_{U_{k,\ge\tau}}(S)|\ge j$, giving $b_j(S)\ge a_j(S)$. Since
$U_k\subseteq\mathcal E$, also $b_j(S)\le a_j(S)$. Hence
\begin{equation}
 \boxed{b_j(S)=a_j(S),\qquad j\le\min\{k,M\}.}
 \label{eq:topk_equal}
\end{equation}
Taking $\tau$ below every score gives
$M\ge|\delta_{U_k}(S)|\ge\min\{k,M\}$, so
$\min\{k,|\delta_{U_k}(S)|\}=\min\{k,M\}$, and summing
\eqref{eq:topk_equal} term by term,
\begin{equation}
 \boxed{T_k(S;U_k)=T_k(S;\mathcal G).}
 \label{eq:topk_sum}
 \qedhere
\end{equation}
\end{proof}

The equality preserves the top-$k$ crossing-score values with
multiplicity. With tied scores, it need not preserve particular edge
identities. For the cosine default, this means preserving the largest
feature-similarity scores across each cut. The usefulness of that
preference for prediction is evaluated empirically.

\subsection{Equivalent Interpretations}
\label{app:submodular}

\begin{proposition}[Submodular reading \citep{NemhauserWolseyFisher1978}]
\label{prop:submodular_certificate}
For a fixed cut $S$ and integer $k\ge1$, let $g(A)=|\delta_A(S)|$ and
$\phi(z)=\min\{z,k\}$. Then $f_{S,k}:=\phi\circ g$ is monotone
submodular. If $1\le k\le K$ and $U_k$ is a prefix of successive
residual covering forests, then $f_{S,k}(U_k)=f_{S,k}(\mathcal E)$.
\end{proposition}

\begin{proof}
$g$ is modular and $\phi$ is nondecreasing and concave, so
\begin{equation}
 f_{S,k}(A)=\min\{|\delta_A(S)|,k\}
 \label{eq:fsk_def}
\end{equation}
is monotone submodular, and so is any nonnegative combination
$\sum_S\alpha_Sf_{S,k}$. Fact~\ref{fact:cut_certificate} gives
\begin{equation}
 \boxed{f_{S,k}(U_k)=f_{S,k}(\mathcal E).}
 \label{eq:submod_cert_app}
 \qedhere
\end{equation}
\end{proof}

This is an interpretation of the stored certificate, not the
objective \edis{} optimizes.

\begin{proposition}[Edge-disjoint path reading]
\label{prop:path_certificate}
Under a pure $\mathcal F_{\max}$ schedule with fixed real scores, let
$1\le k\le K$ and $e=(u,v)\in\mathcal E\setminus U_k$. Then $U_k$
contains $k$ edge-disjoint $u$--$v$ paths whose edges all have score
at least $s_e$.
\end{proposition}

\begin{proof}
$e\notin U_k$ gives
\begin{equation}
 e\in\mathcal E_{\mathrm{rem}}^{(i-1)}\qquad\forall\,i\le k.
 \label{eq:path_residual_app}
\end{equation}
By the bottleneck property (Fact~\ref{fact:bottleneck}), each $F_i$
carries a path $P_i(u,v)\subseteq F_i$ with
$\min_{f\in P_i(u,v)}s_f\ge s_e$. Since $F_i\cap F_j=\varnothing$ for
$i\ne j$, the paths $P_1,\ldots,P_k$ are edge-disjoint. Hence $U_k$
holds $k$ edge-disjoint $u$--$v$ paths of score at least $s_e$.
\end{proof}

In plain terms, for any original edge $e$ outside the stored prefix $U_k$,
the stored prefix $U_k$ already contains $k$ different routes between
its two endpoints, no two of which share an edge, and every edge on
every route scores at least as well as $e$ itself. Dropping $e$
cannot disconnect $u$ from $v$ through routes this good unless at
least $k$ of those routes are also removed. This is the path-side
mirror of the cut certificate above. The cut version counts surviving
edges across a fixed cut, while the path version exhibits the actual
routes behind that count. Two of the $k$ paths may still share an
intermediate vertex, since the guarantee covers edges rather than
vertices, and it holds only for the pure $\mathcal F_{\max}$ schedule
with fixed real scores.

These paths establish alternative routes in the stored prefix. The
proposition does not bound their hop lengths or their survival after
composition. A path can support propagation in an $L$-layer
message-passing GNN only if its length is at most $L$ and all its edges
are retained. Even then, reachability alone does not guarantee useful
influence on the prediction.

\subsection{Composition Retention}
\label{app:budget_theory}

All probabilities condition on the realized decomposition, cached subgraph
weights, and edge-budget in addition to the displayed $\pi^{(t)}$.
Trimming uses increasing $(w_k,k)$ order, so this conditioning fixes
all retained counts. For $B=0$, the identities hold trivially. The
displayed ordering may then be regarded as an unused latent weighted
permutation. For $K=0$, this permutation is empty. Write
$A^{(t)}$ for the selected structural indices.
For $n_k=0$, set $r_k=0$ and $\theta_k=0$. We use
$\binom00=1$ and $\binom ab=0$ for integers $b>a\ge0$.
Thus an empty pool contributes a factor of one to cut avoidance.

Fix $\mathcal D$, its weights, and $B$, and condition on $\pi^{(t)}$. Set
$D_0=\mathcal E_{\mathrm{res}}$ and $D_k=\mathcal E_k$ for $1\le k\le K$,
with $n_k=|D_k|$, retained count $r_k$, and $\theta_k=r_k/n_k$ for
$n_k>0$ (zero otherwise). For a cut $S$ and threshold $\tau$, let
$x_k(S,\tau)=|D_k\cap\delta_{\mathcal G}(S)\cap\mathcal E_{\ge\tau}|$ count
its high-score crossing edges.
\smallskip
\begin{theorem}[Conditional retention under composition, restated]
\label{thm:app_budget_retention_restated}
For any fixed additive edge weight $\omega$, cut $S$, and threshold $\tau$,
\begin{equation}
\mathbb E[\omega(\mathcal E_B^{(t)})\mid\pi^{(t)}]
=\sum_{k=0}^{K}\theta_k\,\omega(D_k),
\qquad
\Pr\bigl(\mathcal E_B^{(t)}\cap\delta_{\mathcal G}(S)\cap\mathcal E_{\ge\tau}=\varnothing\mid\pi^{(t)}\bigr)
=\prod_{k=0}^{K}\frac{\binom{n_k-x_k(S,\tau)}{r_k}}{\binom{n_k}{r_k}}.
\label{eq:budget_retention_cut_restated}
\end{equation}
\end{theorem}

\begin{proof}[Proof of Theorem~\ref{thm:app_budget_retention_restated}]
Condition on $\pi^{(t)}$, fixing $A^{(t)}$ and the retained counts
$0\le r_k\le n_k$, $0\le k\le K$, with $\sum_{k=0}^Kr_k=B$.
Depletion removes whole subgraphs before partially trimming at most one
subgraph. Uniform deletion leaves a uniform retained subset of that subgraph.
All other structural pool subsets are deterministic. Residual completion
occurs only when every structural subgraph is fully retained, so only the
residual subset can then be non-degenerate. Thus at most one pool subset
is non-degenerate under the conditioning, and the pool subsets are
conditionally independent. Each pool contributes a uniform $r_k$-subset,
including the unique empty or full subset when $r_k=0$ or $r_k=n_k$. Hence
\begin{equation}
 \Pr(e\in\mathcal E_B^{(t)}\mid\pi^{(t)})=\theta_k\qquad\forall\,e\in D_k.
 \label{eq:retention_prob}
\end{equation}
By linearity of expectation on
$\omega(\mathcal E_B^{(t)})=\sum_e\omega_e\mathbf 1\{e\in\mathcal E_B^{(t)}\}$,
\begin{equation}
 \boxed{\mathbb E\bigl[\omega(\mathcal E_B^{(t)})\mid\pi^{(t)}\bigr]
 =\sum_{k=0}^K\theta_k\,\omega(D_k).}
 \label{eq:budget_retention_score_app}
\end{equation}

For a cut $S$ and threshold $\tau$, mark the
$x_k(S,\tau)=|D_k\cap\delta_{\mathcal G}(S)\cap\mathcal E_{\ge\tau}|$
crossing edges in each $D_k$. Exactly $\binom{n_k-x_k(S,\tau)}{r_k}$
of the $\binom{n_k}{r_k}$ retained subsets avoid them, so
\begin{equation}
 h_k:=h\bigl(n_k,r_k,x_k(S,\tau)\bigr)
 =\frac{\binom{n_k-x_k(S,\tau)}{r_k}}{\binom{n_k}{r_k}}
 \label{eq:hk_def}
\end{equation}
is the probability that subgraph $k$ retains none of them. Independence
across subgraphs gives
\begin{equation}
 \boxed{\Pr\bigl(\text{no retained }\tau\text{-qualified crossing of }S
 \mid\pi^{(t)}\bigr)=\prod_{k=0}^Kh_k.}
 \label{eq:budget_retention_cut_app}
 \qedhere
\end{equation}
\end{proof}

A vertex is isolated exactly when its singleton cut is empty, so
$S=\{v\}$ with $\tau$ below every score gives its isolation
probability. An isolated vertex receives no neighbor messages that
epoch, so this same identity also bounds how often a node's
aggregation degenerates to its own features alone. These identities
bound retention at each cut individually rather than the joint event
of every cut retaining an edge, so connectivity of $\mathcal
G_B^{(t)}$ at a given $B$ is a separate, cut-by-cut question.

\smallskip

\begin{corollary}[High-score cut survival, restated]
\label{cor:highscore_survival_restated}
Under the same conditions,
\begin{equation}
\Pr\bigl(\mathcal E_B^{(t)}\cap\delta_{\mathcal G}(S)\cap\mathcal E_{\ge\tau}\neq\varnothing\mid\pi^{(t)}\bigr)
\ge1-\exp\Bigl(-\textstyle\sum_{k=0}^{K}\theta_k x_k(S,\tau)\Bigr).
\label{eq:highscore_survival_restated}
\end{equation}
\end{corollary}

\begin{proof}[Proof of Corollary~\ref{cor:highscore_survival_restated}]
For $n=0$ or $r=0$ the factor and its exponential bound are one.
For $n>0$ and $r>n-x$,
$h(n,r,x)=0$ and the bound is immediate. Otherwise $0\le r\le n-x$ and
\begin{equation}
 h(n,r,x)=\frac{\binom{n-x}{r}}{\binom nr}
 =\prod_{i=0}^{r-1}\frac{n-x-i}{n-i}.
 \label{eq:h_expand}
\end{equation}
Since $\frac{n-x-i}{n-i}\le1-\frac xn$ for $0\le i<n$,
\begin{equation}
 h(n,r,x)\le\Bigl(1-\frac xn\Bigr)^r\le e^{-rx/n}.
 \label{eq:h_bound}
\end{equation}
Hence
\begin{equation}
 \prod_{k=0}^Kh_k\le\exp\Bigl(-\sum_{k=0}^K\theta_kx_k(S,\tau)\Bigr),
 \label{eq:prod_bound}
\end{equation}
and complementing~\eqref{eq:budget_retention_cut_app},
\begin{equation}
 \boxed{\Pr\bigl(\exists\text{ retained edge crossing }S\text{ with }
 s_e\ge\tau\mid\pi^{(t)}\bigr)\ \ge\ 1-\exp\Bigl(-\sum_{k=0}^K
 \theta_kx_k(S,\tau)\Bigr).}
 \label{eq:highscore_survival_app}
 \qedhere
\end{equation}
\end{proof}

Writing $\mu=\sum_{k=0}^K\theta_kx_k(S,\tau)$, the bound gives
conditional survival probability at least $1-\delta$ when
$\mu\ge\log(1/\delta)$, for $0<\delta<1$. It is weak for small
$\mu$ and equals zero when $\mu=0$. This is a fixed-cut statement,
not a simultaneous connectivity or path-survival guarantee.

\subsection{Optional Composition Variants}
\label{app:trimming}

Two theory-driven variants extend the composition step. Both are
analyzed here as additional theoretical guarantees available within
the framework, alongside rather than in place of the main experiments.

\paragraph{Certificate-core composition.} Under the maximum-score
residual-forest assumptions, set $U_0=\varnothing$. For edge-budget $B$ let
$q(B)=\max\{q\in\{0,\ldots,K\}:|U_q|\le B\}$ and always retain the prefix $U_{q(B)}$,
filling the remaining $B-|U_{q(B)}|$ edges stochastically without
replacement from $\mathcal E\setminus U_{q(B)}$. Since
$\mathcal E_B^{(t)}\supseteq U_{q(B)}$ and $\delta_A(S)$ is monotone in
$A$, Fact~\ref{fact:cut_certificate} and
Theorem~\ref{thm:score_threshold_certificate} give every epoch graph
the deterministic guarantees
\begin{equation}
 \boxed{|\delta_{\mathcal E_B^{(t)}}(S)|\ge\min\{q(B),|\delta_{\mathcal G}(S)|\},
 \qquad
 |\delta_{\mathcal E_B^{(t)}}(S)\cap\mathcal E_{\ge\tau}|
 \ge\min\{q(B),|\delta_{\mathcal G_{\ge\tau}}(S)|\}.}
 \label{eq:certificate_core}
\end{equation}

This trades cross-epoch variability for a deterministic per-epoch
certificate.
\bigskip

\begin{proposition}[Average-score trimming is expectation-optimal]
\label{prop:avg_score_trimming}
Fix available nonempty subgraphs indexed by $k$, let
$\bar s_k=s(\mathcal E_k)/n_k$, and let $0\le B\le\sum_kn_k$ be an
integer. Among allocations $r_k\in\{0,\ldots,n_k\}$ with
$\sum_kr_k=B$, uniform within-subgraph retention maximizes
$\mathbb E\bigl[\sum_ks(\widetilde{\mathcal E}_k)\bigr]=\sum_kr_k\bar s_k$
by filling subgraphs in decreasing order of $\bar s_k$, with
at most one subgraph partially filled. Other allocations can tie when
subgraphs have equal average score. The main method instead trims by
cached total weight $w_k$ rather than this expectation-optimal
average-score rule. That $w_k$-based greedy fill is equally
deterministic in its own retained counts given fixed tie-breaking,
though uniform sampling within the one partially filled subgraph can
still vary which edges are kept.
This score-optimality statement does not establish an accuracy-optimal
allocation. The main method's weighted ordering and total-weight trimming
are evaluated through the empirical ablations.
\end{proposition}

\begin{proof}
Uniform retention gives $\mathbb E[s(\widetilde{\mathcal E}_k)]=r_k\bar s_k$,
so $\sum_kr_k\bar s_k$ is linear in $(r_k)_k$ subject to $0\le r_k\le n_k$
and $\sum_kr_k=B$. If an optimum has $r_a<n_a$, $r_b>0$, and
$\bar s_a>\bar s_b$, moving one unit from $b$ to $a$ changes the
objective by $\bar s_a-\bar s_b>0$, contradicting optimality. Hence
\begin{equation}
 \boxed{\text{an optimum fills subgraphs in decreasing }\bar s_k
 \text{ order, at most one subgraph split.}}
 \qedhere
\end{equation}
\end{proof}

The results establish exact-budget reuse for any cached partition and
deterministic cut and score-qualified path certificates for maximum-score
forest prefixes. They also characterize conditional edge retention and
fixed-cut survival under stochastic composition. The last
property does not transfer the full prefix certificates to every epoch
graph. The optional certificate-core variant retains a prefix chosen
according to $B$ and therefore inherits its deterministic cut guarantees.

These combinatorial results describe available graph connections without
assuming a particular GNN or loss. The benchmark ablations provide
complementary evidence for predictive utility. At $\rho=0.3$, the reported
means favor \edis{}-Lite over Frozen on 13 of 14 datasets, Direct on 11,
and Shuffled groups on 12 (Table~\ref{tab:mechanism-main}). This supports
the practical contribution of decomposition and resampling in the tested
settings, while leaving open how much of the gain is specifically
attributable to the cut and path certificates rather than other changes
in the induced sampling distribution.
\endgroup

\section{Selectors and Computational Cost}
\label{app:implementation}

\subsection{Family of Residual Selectors}
\label{app:selector_family}

Each of the seven selectors below implements the same residual-selector
interface $\mathcal F$ from a residual graph $H$ and its fixed edge
scores $s$, so any one of them can be dropped into \edis{} without
changing anything else in the pipeline. What differs is only how
each one decides which edges to keep, and consequently its round
cost $T_i(n,m_i)$, explained selector by selector below and
summarized in Table~\ref{tab:selector-costs} at the end of this
appendix.

\paragraph{\underline{Maximum/minimum covering forest} (\texttt{maxst}/\texttt{minst}, \textbf{MaxCF}/\textbf{MinCF} in Section~\ref{sec:experiments}).}
The implementation names \texttt{maxst}/\texttt{minst} reflect that
each round's forest spans every component of the current residual
graph. We call it a \emph{covering forest} instead because that
spanning property holds only within one round, not for the
budget-limited graph actually composed for training
(Section~\ref{subsec:theory}).
This selector puts every remaining edge in order by score, either
strongest first or weakest first, and adds edges one at a time down
that list, skipping any edge that would close a loop back to nodes
already connected. This is Kruskal's algorithm
\citep{Kruskal1956,Edmonds1971} run on our scores instead of raw
edge weights, with ties kept in their original order. Because it
only adds loop-free edges, it ends up touching every part of the
residual graph. \texttt{maxst} and \texttt{minst} differ only in
which direction we sort. Each round sorts the remaining edges and
tracks which nodes are already connected, costing a little more
than one plain scan.

\paragraph{\underline{$k$-nearest neighbor} (\texttt{knn}).}
This selector lets every node claim its $k$ best-scoring remaining
edges, or all of them if fewer than $k$ are left, and keeps the
union of every node's claims. Asking for a fixed total $M$ instead
of a fixed $k$ trims or tops up that union by score until it hits
$M$ exactly, which can take away an edge a node had already claimed,
so no node is guaranteed to keep its own top choices. This touches
each remaining edge from both ends.

\paragraph{\underline{Graph spanner} (\texttt{spanner}).}
This selector works on each disconnected piece of the residual graph
on its own, building a smaller subgraph using edge lengths that
shrink as the score grows, so the best-scoring edges look shortest.
The result must keep every pair of nodes within a stretch factor
$\alpha$ of their true distance in these lengths
\citep{Peleg:1989spanners}, a bound on path length rather than hop
count. Construction is randomized, and its cost depends on the
spanner routine used. 

\paragraph{\underline{Degree-normalized} (\texttt{degree\_norm}).}
This selector draws a set number $M$ of edges at random without
repeats, weighting each edge $uv$ by $w_{uv}=1/d_H(u)+1/d_H(v)$ so
edges touching low-degree nodes are more likely to be picked,
recomputing degrees every round. It ignores the score entirely
and guarantees neither connectivity nor a minimum degree. One pass
over the edges is enough to compute the weights.

\paragraph{\underline{Random} (\texttt{random}).}
This selector shuffles the remaining edges uniformly at random and
keeps the first $M$. Ignoring both score and structure, it serves as
the score-free baseline the others are compared against, and needs
only one pass over the edges.

\paragraph{\underline{Scheduled} (Hybrid).}
A schedule chains several selectors together, using a different one
each round on whatever edges remain, and repeating the last one once
the list runs out. Any of the five selectors above can appear, in
any order, and can repeat, so this is a general way to combine
selectors rather than one fixed recipe. The schedule used in our
experiments runs \texttt{maxst}, then \texttt{knn}, then
\texttt{degree\_norm} for every later round, first extracting a
connectivity-preserving backbone, then giving each node its own best
local edges, then spreading coverage to whichever nodes are still
thin. It stops under the same conditions as the other selectors.

\paragraph{Selectors outside this family.}
Effective-resistance sampling \citep{Spielman:2008effres} needs
spectral information recomputed as the residual changes. Forest-fire
sampling \citep{Leskovec:2006sampling}, SCAN \citep{Xu:2007scan}, and
rank-degree sampling \citep{Voudigari:2016rankdegree} likewise need
traversal, clustering, or ranking steps beyond a residual degree
pass. None of these are implemented here as candidate selectors
inside \edis{}'s own construction, though several are compared
against \edis{} directly as full competing methods in
Section~\ref{sec:experiments}.

\clearpage
\subsection{Cost of Construction and Epoch Composition}
\label{app:complexity}

We now analyze where \edis{}'s computational cost is spent, writing
$n=|\mathcal V|$, $m=|\mathcal E|$, and $d_0$ for the input feature
dimension. Every edge is scored once, before construction begins, at
a cost of $T_{\mathrm{score}}=O(md_0)$, and this cost cannot be
skipped even by a score-free selector, since the same scores also
determine subgraph importance.

Building the decomposition runs the selector once per round over
whatever edges remain, so its total cost is just the sum of each
round's selector cost and residual bookkeeping. For the repeated-sorting
\texttt{maxst} construction, writing $m_i$ for the residual edge count
at round $i$, this gives
\[
 T_{\mathrm{construct}}
 =O\!\left(Kn+\sum_{i=1}^K m_i\log(m_i+1)\right)
 \subseteq O\!\left(Kn+Km\log(m+1)\right).
\]
The dependence on $K$ can be absorbed only when the round cap is treated
as a fixed constant. Fixed scores also permit implementations that reuse
an initial ordering. Storing the resulting subgraphs and the residual adds
only $O(m+K)$ extra space, on top of the graph itself and its
features.

Composition reuses the partition and scores without repeating extraction.
An implementation can determine retained counts from cached group sizes
before materializing edges. With an $O(K\log(K+1))$ weighted ordering
and trimming-order sort, random-access pools, and uniform subset sampling
in time proportional to the returned subset size, this yields
$O(K\log(K+1)+B)$ composition time. This is an implementation-dependent
bound. Copying entire candidate pools or generating full-pool permutations
can introduce additional cost depending on their sizes, up to $m$.
Expanding the chosen edges into their reciprocal
directed pairs for message passing adds a further
$T_{\mathrm{reciprocal}}=O(B)$ per epoch.

Summing over all $T$ training epochs, the total cost of a job is the
one-time scoring and construction cost plus the per-epoch
composition, reciprocal-expansion, training, transfer, and
validation costs added up across epochs,
\begin{equation*}
 T_{\mathrm{job}} = T_{\mathrm{score}} + T_{\mathrm{construct}}
 + \sum_{t=1}^{T}\Bigl(T_{\mathrm{compose}}^{(t)}
 + T_{\mathrm{reciprocal}}^{(t)} + T_{\mathrm{train}}^{(t)}
 + T_{\mathrm{transfer}}^{(t)} + T_{\mathrm{val}}^{(t)}\Bigr).
\end{equation*}
Full-graph validation and test inference, where used, process the
original graph $\mathcal G=(\mathcal V,\mathcal E)$ instead of a
composed one, so their cost depends on $(n,m)$ rather than $B$, and
is charged separately from training. The headline point is that
scoring and construction are paid only once, while composition and
reciprocal expansion recur every epoch. Budget-proportional edge
materialization is possible under the implementation conditions above.
Measured runtime is reported separately in Appendix~\ref{app:runtime-memory}.

\begin{center}
\small
\setlength{\tabcolsep}{3pt}
\rowcolors{2}{white}{black!3}
\begin{tabular}{lll}
\toprule
\textbf{Selector} & \textbf{Round cost $T_i(n,m_i)$} & \textbf{Structural guarantee} \\
\midrule
\texttt{maxst} / \texttt{minst} & $O(n+m_i\log(m_i+1))$ & covering forest of every component \\
\texttt{knn} & $O(m_i\log(m_i+1))$ & top-$k$ per node, unless trimmed to hit $M$ \\
\texttt{spanner} & depends on the spanner routine used & distance stretch at most $\alpha$ \\
\texttt{degree\_norm} & $O(m_i)$ & none \\
\texttt{random} & $O(m_i)$ & none \\
Hybrid & cost of whichever selector is active that round & inherits from the active selector \\
\bottomrule
\end{tabular}
\captionof{table}{Round cost and structural guarantee of each
residual selector from Section~\ref{app:selector_family}.}
\label{tab:selector-costs}
\end{center}

\section{Experimental Details}
\label{app:experimental}

\subsection{Dataset}

The experiments use 19 node-classification benchmarks, spanning six
heterophilic small/medium graphs, eight homophilic small/medium
graphs, and five large-scale graphs (Table~\ref{tab:datasets}). Node
and edge counts, and average degree, are computed after
treating each graph as undirected and removing self-loops. Adjusted
homophily measures how often connected nodes share a class after
correcting for class proportions \citep{Platonov:2023critical}. 
Classes/Tasks reports the number of prediction classes for every
dataset except ogbn-proteins, a multi-label dataset with 112 binary
prediction tasks, and Features gives the input feature dimension.
Most datasets are evaluated with accuracy. Minesweeper, Questions,
and ogbn-proteins use AUC instead, since their prediction tasks are
binary.

\begin{table*}[!htbp]
\centering

\small
\setlength{\tabcolsep}{5pt}
\renewcommand{\arraystretch}{1.05}

\begin{adjustbox}{max width=\linewidth}
\begin{tabular}{
@{\extracolsep{\fill}}
l
r
r
r
r
r
r
c
@{}}
\toprule

& \multicolumn{3}{c}{\textbf{Graph size}}
& \multicolumn{1}{c}{\textbf{Structure}}
& \multicolumn{3}{c}{\textbf{Prediction task}} \\

\cmidrule(lr){2-4}
\cmidrule(lr){5-5}
\cmidrule(l){6-8}

\textbf{Dataset}
& \textbf{Nodes}
& \textbf{Edges}
& \textbf{Avg. degree}
& \textbf{Adj. homophily}
& \textbf{Classes/Tasks}
& \textbf{Features}
& \textbf{Metric} \\

\midrule

\rowcolor[gray]{0.93}
\multicolumn{8}{@{}l@{}}{
\textit{Heterophilic small/medium}
~\citep{Platonov:2023critical}
} \\

Squirrel
& 2.22K
& 47.00K
& 42.28
& 0.01
& 5
& 2.09K
& Acc \\

\rowcolor[gray]{0.975}
Chameleon
& 890
& 8.85K
& 19.90
& 0.03
& 5
& 2.33K
& Acc \\

Roman-Empire
& 22.66K
& 32.93K
& 2.91
& -0.05
& 18
& 300
& Acc \\

\rowcolor[gray]{0.975}
Amazon-Ratings
& 24.49K
& 93.05K
& 7.60
& 0.14
& 5
& 300
& Acc \\

Minesweeper
& 10.00K
& 39.40K
& 7.88
& 0.01
& 2
& 7
& AUC \\

\rowcolor[gray]{0.975}
Questions
& 48.92K
& 153.54K
& 6.28
& 0.02
& 2
& 301
& AUC \\

\addlinespace[2pt]

\rowcolor[gray]{0.93}
\multicolumn{8}{@{}l@{}}{
\textit{Homophilic small/medium}
} \\

Cora~\citep{Yang:2016planetoid}
& 2.71K
& 5.28K
& 3.90
& 0.77
& 7
& 1.43K
& Acc \\

\rowcolor[gray]{0.975}
CiteSeer~\citep{Yang:2016planetoid}
& 3.33K
& 4.55K
& 2.74
& 0.67
& 6
& 3.70K
& Acc \\

PubMed~\citep{Yang:2016planetoid}
& 19.72K
& 44.32K
& 4.50
& 0.69
& 3
& 500
& Acc \\

\rowcolor[gray]{0.975}
Amazon-Computer~\citep{Shchur:2018pitfalls}
& 13.75K
& 245.86K
& 35.76
& 0.68
& 10
& 767
& Acc \\

Amazon-Photo~\citep{Shchur:2018pitfalls}
& 7.65K
& 119.08K
& 31.13
& 0.79
& 8
& 745
& Acc \\

\rowcolor[gray]{0.975}
Coauthor-CS~\citep{Shchur:2018pitfalls}
& 18.33K
& 81.89K
& 8.93
& 0.78
& 15
& 6.81K
& Acc \\

Coauthor-Physics~\citep{Shchur:2018pitfalls}
& 34.49K
& 247.96K
& 14.38
& 0.87
& 5
& 8.42K
& Acc \\

\rowcolor[gray]{0.975}
WikiCS~\citep{Mernyei:2020wikics}
& 11.70K
& 215.60K
& 36.85
& 0.58
& 10
& 300
& Acc \\

\addlinespace[2pt]

\rowcolor[gray]{0.93}
\multicolumn{8}{@{}l@{}}{
\textit{Large-scale}
} \\

ogbn-arxiv~\citep{Hu:2020ogb}
& 169.34K
& 1.17M
& 13.77
& 0.42
& 40
& 128
& Acc \\

\rowcolor[gray]{0.975}
ogbn-proteins~\citep{Hu:2020ogb}
& 132.53K
& 39.56M
& 597.00
& 0.05
& 112
& 8
& AUC \\

ogbn-products~\citep{Hu:2020ogb}
& 2.45M
& 61.86M
& 50.52
& 0.46
& 47
& 100
& Acc \\

\rowcolor[gray]{0.975}
Pokec~\citep{Lim:2021nonhomophilous}
& 1.63M
& 22.30M
& 27.32
& 0.42
& 2
& 65
& Acc \\

Reddit~\citep{Hamilton:2017tp}
& 232.97K
& 57.31M
& 492.00
& 0.74
& 41
& 602
& Acc \\

\bottomrule
\end{tabular}
\end{adjustbox}
\caption{
Dataset statistics for the 19 node-classification benchmarks used in
the current experiments.
}
\label{tab:datasets}
\end{table*}

\subsection{Baseline Settings}

We compare \edis{} against baselines drawn from three families,
grouped by the structural signal each relies on rather than by
implementation detail. Anchors bound the comparison
without adapting to the graph, and comprise a feature-only MLP (No
Graph), a Spanning Tree, and Full Graph with the same tuned GNN
configuration \citep{Luo:2024classicgnns}. Semantic baselines learn
which edges to keep jointly with the GNN parameters, and comprise
Unified-LTH \citep{Chen:2021ugs}, AdaGLT \citep{Zhang:2024adaglt},
and MoG \citep{Zhang:2025mog}. Topological baselines score edges
directly from the graph, independent of the downstream task, and
comprise Random, Local Degree \citep{Hamann:2016sparsification},
Rank Degree \citep{Voudigari:2016rankdegree}, Forest Fire
\citep{Leskovec:2006sampling}, SCAN \citep{Xu:2007scan}, L-Spar and
G-Spar \citep{Satuluri:2011lsim}, L-Sim
\citep{Chen:2023sparsifybench}, Spectral (ER)
\citep{Spielman:2008effres}, $t$-Spanner \citep{Peleg:1989spanners},
and DSpar \citep{Liu:2023dspar}.

\textbf{Sampling-based related methods.}\quad Neighborhood- and
subgraph-sampling methods such as GraphSAGE \citep{Hamilton:2017tp},
GraphSAINT \citep{Zeng:2020graphsaint}, and Cluster-GCN
\citep{Chiang:2019clustergcn} control a different resource, the
minibatch or neighborhood size used for scalability, rather than a
global edge-retention budget. They are therefore not part of the
matched-edge-budget comparison, and GraphSAGE appears only as the
alternate backbone in RQ7.

\subsection{Propagation and Evaluation Settings}
\label{app:propagation_details}

For full-batch GCN training, we add each retained undirected edge in
both directions. We add one self-loop per node. The edge-budget $B$
counts only the undirected non-loop edges. All layers use the same epoch
graph, with normalized adjacency
$\hat{\mathbf A}^{(t)}=(\tilde{\mathbf D}^{(t)})^{-1/2}
(\mathbf A^{(t)}+\mathbf I)(\tilde{\mathbf D}^{(t)})^{-1/2}$.
Here $\tilde{\mathbf D}^{(t)}$ holds the degrees of
$\mathbf A^{(t)}+\mathbf I$. We recompute normalization for the
current topology. Edge scores guide selection. They are not
propagation weights. The architecture and training objective stay
the same.

\subsection{Hyperparameter Settings}
\label{app:hyperparameters}

Every baseline and sparsifier method uses the same per-dataset
training setup for a given backbone. This keeps accuracy differences
tied to the sparsification method, not the training setup. Per-dataset
settings follow the defaults of \citet{Luo:2024classicgnns} for every
dataset except Reddit, which that work does not include. We use two
backbones. GCN is the default for every method. GraphSAGE is used
only by Full GraphSAGE and \edis{}-GraphSAGE. In both tables below,
$L$ is the number of GNN layers, $H$ is the hidden dimension, LR is
the learning rate, and WD is the weight decay.

\paragraph{\underline{$K_{\max}$}.} The round cap $K_{\max}$ on subgraph
extraction is itself a hyperparameter, chosen per dataset rather than
fixed globally, and applied the same way regardless of backbone.
Every small- and medium-scale dataset, homophilic or heterophilic,
uses $K_{\max}=10$. The five large-scale graphs instead use
$K_{\max}=80$ for Pokec, $K_{\max}=40$ for ogbn-arxiv, $K_{\max}=20$
for Reddit, and $K_{\max}=10$ for ogbn-products and ogbn-proteins.
These are the same values used throughout the main experiments and
the extended ablations, where Section~\ref{sec:experiments} (RQ6) and
Appendix~\ref{app:kmax-sensitivity} study how sensitive results are
to this choice.

\paragraph{\underline{MLP edge score}.} The learned score used in the
edge-score ablation (Section~\ref{sec:experiments} RQ5,
Appendix~\ref{app:score-profiles}) trains a small MLP on the concatenated
features of each edge's two endpoints to predict an edge score.
Training uses only edges whose two endpoints both carry a training
label, with target 1 if the endpoints share a label and 0 otherwise,
and an equal number of same-label and different-label edges sampled
per batch. PubMed's semi-supervised split labels too few training
nodes to assemble this balanced sample, so the MLP score is marked
N/A for PubMed throughout.

\begin{table*}[!htbp]
\centering

\begin{adjustbox}{max width=\linewidth}
\small
\setlength{\tabcolsep}{4pt}
\renewcommand{\arraystretch}{0.50}
\begin{tabular}{lrrrrlrllc}
\toprule
& \multicolumn{2}{c}{\textbf{Protocol}}
& \multicolumn{3}{c}{\textbf{Architecture}}
& \multicolumn{3}{c}{\textbf{Optimization}}
& \multicolumn{1}{c}{\textbf{Evaluation}} \\
\cmidrule(lr){2-3}\cmidrule(lr){4-6}\cmidrule(lr){7-9}\cmidrule(l){10-10}
\textbf{Dataset} & \textbf{Epochs} & \textbf{Runs} & $L$ & $H$ &
\textbf{Norm./Res.} & \textbf{Dropout} & \textbf{LR} & \textbf{WD} & \textbf{Metric} \\
\midrule
\rowcolor[gray]{0.93}\multicolumn{10}{@{}l@{}}{\textit{Heterophilic small/medium}} \\
Squirrel         & 500  & 10 & 4  & 256 & bn, res       & 0.70 & 1e-2 & 5e-4 & Acc \\
\rowcolor[gray]{0.975}Chameleon      & 200  & 10 & 5  & 512 & --            & 0.20 & 5e-3 & 1e-3 & Acc \\
Roman-Empire     & 2500 & 3  & 9  & 512 & bn, res       & 0.50 & 1e-3 & 0    & Acc \\
\rowcolor[gray]{0.975}Amazon-Ratings & 2500 & 3  & 4  & 512 & bn, res       & 0.50 & 1e-3 & 0    & Acc \\
Minesweeper      & 2000 & 3  & 12 & 64  & bn, res       & 0.20 & 1e-2 & 0    & AUC \\
\rowcolor[gray]{0.975}Questions      & 1500 & 3  & 10 & 512 & res           & 0.30 & 3e-5 & 0    & AUC \\
\rowcolor[gray]{0.93}\multicolumn{10}{@{}l@{}}{\textit{Homophilic small/medium}} \\
Cora             & 500  & 5  & 3  & 512 & --            & 0.70 & 1e-3 & 5e-4 & Acc \\
\rowcolor[gray]{0.975}CiteSeer       & 500  & 5  & 2  & 512 & --            & 0.50 & 1e-3 & 1e-2 & Acc \\
PubMed           & 500  & 5  & 2  & 256 & --            & 0.70 & 5e-3 & 5e-4 & Acc \\
\rowcolor[gray]{0.975}Amazon-Computer & 1000 & 3  & 3  & 512 & ln           & 0.50 & 1e-3 & 5e-5 & Acc \\
Amazon-Photo     & 1000 & 3  & 6  & 256 & ln, res       & 0.50 & 1e-3 & 5e-5 & Acc \\
\rowcolor[gray]{0.975}Coauthor-CS    & 1500 & 3  & 2  & 512 & ln, res       & 0.30 & 1e-3 & 5e-4 & Acc \\
Coauthor-Physics & 1500 & 3  & 2  & 64  & ln, res       & 0.30 & 1e-3 & 5e-4 & Acc \\
\rowcolor[gray]{0.975}WikiCS         & 1000 & 3  & 3  & 256 & ln            & 0.50 & 1e-3 & 0    & Acc \\
\rowcolor[gray]{0.93}\multicolumn{10}{@{}l@{}}{\textit{Large-scale}} \\
ogbn-arxiv       & 2000 & 2 & 5 & 512 & bn, res & 0.50 & 5e-4 & 5e-4 & Acc \\
\rowcolor[gray]{0.975}ogbn-proteins  & 100  & 2 & 3 & 80  & --      & 0.25 & 1e-2 & 0    & AUC \\
ogbn-products    & 300  & 2 & 5 & 200 & ln      & 0.50 & 3e-3 & 0    & Acc \\
\rowcolor[gray]{0.975}Pokec          & 2000 & 2 & 7 & 256 & bn, res & 0.20 & 5e-4 & 0    & Acc \\
Reddit           & 300  & 2 & 5 & 200 & ln      & 0.50 & 3e-3 & 0    & Acc \\
\bottomrule
\end{tabular}
\end{adjustbox}
\caption{Per-dataset training configuration for the GCN backbone.}
\label{tab:presets}
\end{table*}

\begin{table*}[!htbp]
\centering

\begin{adjustbox}{max width=\linewidth}
\small
\setlength{\tabcolsep}{4pt}
\renewcommand{\arraystretch}{0.50}
\begin{tabular}{lrrrrlrllc}
\toprule
& \multicolumn{2}{c}{\textbf{Protocol}}
& \multicolumn{3}{c}{\textbf{Architecture}}
& \multicolumn{3}{c}{\textbf{Optimization}}
& \multicolumn{1}{c}{\textbf{Evaluation}} \\
\cmidrule(lr){2-3}\cmidrule(lr){4-6}\cmidrule(lr){7-9}\cmidrule(l){10-10}
\textbf{Dataset} & \textbf{Epochs} & \textbf{Runs} & $L$ & $H$ &
\textbf{Norm./Res.} & \textbf{Dropout} & \textbf{LR} & \textbf{WD} & \textbf{Metric} \\
\midrule
\rowcolor[gray]{0.93}\multicolumn{10}{@{}l@{}}{\textit{Heterophilic small/medium}} \\
Squirrel         & 500  & 10 & 3  & 256 & bn, res & 0.70 & 1e-2 & --   & Acc \\
\rowcolor[gray]{0.975}Chameleon      & 200  & 10 & 4  & 256 & bn, res & 0.70 & 1e-2 & --   & Acc \\
Roman-Empire     & 2500 & 3  & 9  & 256 & bn      & 0.30 & 1e-3 & --   & Acc \\
\rowcolor[gray]{0.975}Amazon-Ratings & 2500 & 3  & 9  & 512 & bn, res & 0.50 & 1e-3 & --   & Acc \\
Minesweeper      & 2000 & 3  & 15 & 64  & bn, res & 0.20 & 1e-2 & --   & AUC \\
\rowcolor[gray]{0.975}Questions      & 1500 & 3  & 6  & 512 & ln      & 0.20 & 1e-3 & --   & AUC \\
\rowcolor[gray]{0.93}\multicolumn{10}{@{}l@{}}{\textit{Homophilic small/medium}} \\
Cora             & 500  & 5  & 3 & 256 & --      & 0.70 & 1e-3 & --   & Acc \\
\rowcolor[gray]{0.975}CiteSeer       & 500  & 5  & 3 & 512 & --      & 0.20 & 1e-3 & --   & Acc \\
PubMed           & 500  & 5  & 4 & 512 & --      & 0.70 & 5e-3 & --   & Acc \\
\rowcolor[gray]{0.975}Amazon-Computer & 1000 & 3  & 4 & 64  & ln      & 0.30 & 1e-3 & --   & Acc \\
Amazon-Photo     & 1000 & 3  & 6 & 64  & ln, res & 0.20 & 1e-3 & --   & Acc \\
\rowcolor[gray]{0.975}Coauthor-CS    & 1500 & 3  & 2 & 512 & ln, res & 0.50 & 1e-3 & --   & Acc \\
Coauthor-Physics & 1500 & 3  & 2 & 64  & bn, res & 0.70 & 1e-3 & --   & Acc \\
\rowcolor[gray]{0.975}WikiCS         & 1000 & 3  & 2 & 256 & ln      & 0.70 & 1e-3 & --   & Acc \\
\rowcolor[gray]{0.93}\multicolumn{10}{@{}l@{}}{\textit{Large-scale}} \\
ogbn-arxiv       & 2000 & 2 & 4 & 256 & bn, res & 0.50 & 5e-4 & --   & Acc \\
\rowcolor[gray]{0.975}ogbn-proteins  & 1000 & 2 & 6 & 512 & bn, res & 0.30 & 1e-2 & --   & AUC \\
ogbn-products    & 1000 & 2 & 5 & 256 & ln      & 0.50 & 3e-3 & --   & Acc \\
\rowcolor[gray]{0.975}Pokec          & 2000 & 2 & 7 & 256 & bn, res & 0.20 & 5e-4 & --   & Acc \\
Reddit           & 300  & 2 & 5 & 200 & ln      & 0.50 & 3e-3 & --   & Acc \\
\bottomrule
\end{tabular}
\end{adjustbox}
\caption{Per-dataset training configuration for the GraphSAGE backbone.}
\label{tab:graphsage_hyperparameters}
\end{table*}

\FloatBarrier

\subsection{Complete numerical results}
\label{app:complete-numerical-results}

Tables~\ref{tab:main-hom-acc}--\ref{tab:main-large-acc} report full
results (mean $\pm$ standard deviation) on the homophilic,
heterophilic, and large-scale datasets.

\begingroup
\definecolor{homBest}{RGB}{183,220,228}
\definecolor{homBand}{RGB}{235,235,235}
\definecolor{homHeader}{RGB}{222,226,230}
\definecolor{homInk}{RGB}{20,20,20}
\definecolor{homSecond}{RGB}{249,231,191}
\definecolor{homThird}{RGB}{231,221,242}
\newcommand{\homValue}[2]{\shortstack{#1\\[-1pt]{\fontsize{5.5pt}{6pt}\selectfont\color{black!80}\textpm#2}}}
\newlength{\homPageHeight}
\setlength{\homPageHeight}{0.98\textwidth}
\begin{sidewaystable}[p]
\centering
\begin{adjustbox}{max totalheight=\homPageHeight}
\begin{minipage}{0.98\textheight}
\centering
\caption{\textbf{Test accuracy on homophilic node-classification benchmarks.} Results are reported as mean and standard deviation (\%) at retained-edge fractions $\rho\in\{0.3,0.5,0.7\}$; higher values indicate better performance.}
\label{tab:main-hom-acc}
\label{app:full_hom_acc}
\vspace{3pt}
\fontfamily{ptm}\fontsize{7.5pt}{8pt}\selectfont
\color{homInk}
\setlength{\tabcolsep}{1.5pt}
\renewcommand{\arraystretch}{1.0}
\begin{tabular*}{0.98\textheight}{@{\extracolsep{\fill}}lc*{19}{c}@{}}
\toprule
\rowcolor{homHeader}[3pt][3pt]
\textbf{Dataset} & $\boldsymbol{\rho}$ & \multicolumn{4}{c}{\textbf{Reference anchors}} & \multicolumn{11}{c}{\textbf{Topology-based}} & \multicolumn{3}{c}{\textbf{Semantic-based}} & \multicolumn{1}{c}{\textbf{EDiS (Ours)}} \\
\cmidrule(lr){1-2} \cmidrule(lr){3-6} \cmidrule(lr){7-17} \cmidrule(lr){18-20} \cmidrule(lr){21-21}
 &  & \shortstack{MLP\\(no graph)} & \shortstack{Spanning\\Tree} & \shortstack{Full\\GCN} & \shortstack{EDiS-\\Full} & \shortstack{Random} & \shortstack{Local\\Degree} & \shortstack{Rank\\Degree} & \shortstack{Forest\\Fire} & \shortstack{SCAN} & \shortstack{L-Spar} & \shortstack{G-Spar} & \shortstack{L-Sim} & \shortstack{Spectral} & \shortstack{$t$-Spanner} & \shortstack{DSpar} & \shortstack{Unified-\\LTH} & \shortstack{AdaGLT} & \shortstack{MoG} & \shortstack{EDiS-\\Lite} \\
\midrule
\rowcolor{white}[3pt][3pt]  \multicolumn{1}{l}{\cellcolor{white}} & 0.3 & \multicolumn{1}{c}{\cellcolor{white}} & \multicolumn{1}{c}{\cellcolor{white}} & \multicolumn{1}{c}{\cellcolor{white}} & \homValue{84.08}{0.83} & \homValue{73.40}{0.51} & \homValue{73.52}{0.47} & \cellcolor{homSecond}\homValue{\underline{76.34}}{0.78} & \homValue{71.98}{0.39} & \homValue{71.38}{0.67} & \homValue{71.38}{1.48} & \homValue{72.02}{1.05} & \cellcolor{homThird}\homValue{\textit{75.18}}{0.72} & \homValue{72.86}{0.85} & \homValue{72.60}{0.90} & \homValue{73.90}{0.36} & \homValue{72.12}{3.47} & \homValue{72.74}{0.63} & \homValue{73.66}{0.77} & \cellcolor{homBest}\homValue{\textbf{77.16}}{0.98} \\
\rowcolor{homBand}[3pt][3pt]  \multicolumn{1}{l}{\cellcolor{white}\textbf{Cora}} & 0.5 & \multicolumn{1}{c}{\cellcolor{white}\homValue{60.14}{2.62}} & \multicolumn{1}{c}{\cellcolor{white}\homValue{80.86}{1.17}} & \multicolumn{1}{c}{\cellcolor{white}\homValue{84.26}{0.66}} & \homValue{83.30}{0.61} & \homValue{79.10}{0.80} & \cellcolor{homThird}\homValue{\textit{80.68}}{0.56} & \homValue{80.60}{0.76} & \cellcolor{homSecond}\homValue{\underline{80.98}}{0.58} & \homValue{78.32}{0.83} & \homValue{74.68}{0.82} & \homValue{74.68}{0.82} & \cellcolor{homBest}\homValue{\textbf{81.14}}{0.72} & \homValue{75.37}{2.53} & \homValue{75.67}{1.95} & \homValue{77.33}{0.75} & \homValue{73.57}{1.92} & \homValue{76.80}{0.86} & \homValue{75.90}{1.47} & \homValue{80.28}{0.75} \\
\rowcolor{white}[3pt][3pt]  \multicolumn{1}{l}{\cellcolor{white}} & 0.7 & \multicolumn{1}{c}{\cellcolor{white}} & \multicolumn{1}{c}{\cellcolor{white}} & \multicolumn{1}{c}{\cellcolor{white}} & \homValue{84.28}{0.55} & \homValue{81.10}{0.72} & \cellcolor{homThird}\homValue{\textit{82.74}}{0.68} & \homValue{82.46}{0.57} & \homValue{80.60}{0.40} & \homValue{81.82}{0.63} & \homValue{78.02}{0.67} & \homValue{77.92}{0.74} & \homValue{80.42}{0.35} & \cellcolor{homSecond}\homValue{\underline{82.78}}{0.72} & \homValue{80.56}{0.27} & \homValue{81.20}{0.86} & \homValue{79.72}{2.26} & \homValue{80.20}{0.95} & \homValue{81.50}{0.81} & \cellcolor{homBest}\homValue{\textbf{82.88}}{0.50} \\
\midrule[0.25pt]
\rowcolor{homBand}[3pt][3pt]  \multicolumn{1}{l}{\cellcolor{white}} & 0.3 & \multicolumn{1}{c}{\cellcolor{white}} & \multicolumn{1}{c}{\cellcolor{white}} & \multicolumn{1}{c}{\cellcolor{white}} & \homValue{70.88}{0.78} & \homValue{65.86}{0.95} & \homValue{64.76}{0.84} & \homValue{65.10}{0.24} & \homValue{63.82}{1.05} & \homValue{64.54}{0.37} & \homValue{64.32}{0.57} & \homValue{64.28}{0.57} & \cellcolor{homThird}\homValue{\textit{66.06}}{0.42} & \cellcolor{homBest}\homValue{\textbf{66.82}}{0.44} & \homValue{65.70}{0.63} & \homValue{64.70}{1.22} & \homValue{63.76}{1.19} & \homValue{65.26}{1.00} & \homValue{65.28}{1.50} & \cellcolor{homSecond}\homValue{\underline{66.08}}{1.23} \\
\rowcolor{white}[3pt][3pt]  \multicolumn{1}{l}{\cellcolor{white}\textbf{CiteSeer}} & 0.5 & \multicolumn{1}{c}{\cellcolor{white}\homValue{57.00}{1.13}} & \multicolumn{1}{c}{\cellcolor{white}\homValue{70.50}{0.64}} & \multicolumn{1}{c}{\cellcolor{white}\homValue{72.50}{0.36}} & \homValue{70.88}{0.47} & \homValue{67.34}{0.60} & \cellcolor{homThird}\homValue{\textit{69.04}}{0.58} & \cellcolor{homBest}\homValue{\textbf{69.74}}{0.33} & \homValue{68.60}{1.10} & \homValue{68.74}{0.59} & \homValue{65.80}{0.72} & \homValue{65.80}{0.72} & \cellcolor{homSecond}\homValue{\underline{69.42}}{0.55} & \homValue{66.30}{0.22} & \homValue{66.90}{0.71} & \homValue{68.63}{0.52} & \homValue{66.23}{0.59} & \homValue{66.03}{0.54} & \homValue{67.73}{0.97} & \homValue{68.20}{0.61} \\
\rowcolor{homBand}[3pt][3pt]  \multicolumn{1}{l}{\cellcolor{white}} & 0.7 & \multicolumn{1}{c}{\cellcolor{white}} & \multicolumn{1}{c}{\cellcolor{white}} & \multicolumn{1}{c}{\cellcolor{white}} & \homValue{71.14}{0.29} & \homValue{70.16}{0.53} & \homValue{70.20}{0.55} & \homValue{70.40}{0.45} & \cellcolor{homBest}\homValue{\textbf{71.50}}{0.48} & \homValue{70.08}{0.24} & \homValue{69.64}{0.79} & \homValue{69.66}{0.77} & \homValue{70.56}{1.18} & \homValue{70.34}{0.82} & \cellcolor{homSecond}\homValue{\underline{70.72}}{0.61} & \homValue{70.44}{0.37} & \homValue{69.18}{0.95} & \homValue{69.46}{0.43} & \homValue{70.14}{0.99} & \cellcolor{homThird}\homValue{\textit{70.64}}{0.84} \\
\midrule[0.25pt]
\rowcolor{white}[3pt][3pt]  \multicolumn{1}{l}{\cellcolor{white}} & 0.3 & \multicolumn{1}{c}{\cellcolor{white}} & \multicolumn{1}{c}{\cellcolor{white}} & \multicolumn{1}{c}{\cellcolor{white}} & \homValue{80.20}{0.46} & \homValue{74.26}{0.62} & \cellcolor{homSecond}\homValue{\underline{77.32}}{1.08} & \homValue{76.38}{0.46} & \homValue{73.46}{0.85} & \homValue{74.32}{0.50} & \homValue{73.48}{0.90} & \homValue{73.48}{0.90} & \homValue{76.98}{0.74} & \homValue{75.04}{1.43} & \homValue{74.74}{0.59} & \cellcolor{homBest}\homValue{\textbf{77.86}}{0.46} & \homValue{75.26}{0.67} & \homValue{75.84}{0.56} & \homValue{71.56}{0.54} & \cellcolor{homThird}\homValue{\textit{77.28}}{0.57} \\
\rowcolor{homBand}[3pt][3pt]  \multicolumn{1}{l}{\cellcolor{white}\textbf{PubMed}} & 0.5 & \multicolumn{1}{c}{\cellcolor{white}\homValue{69.82}{1.86}} & \multicolumn{1}{c}{\cellcolor{white}\homValue{80.40}{0.99}} & \multicolumn{1}{c}{\cellcolor{white}\homValue{80.76}{0.55}} & \homValue{80.24}{0.61} & \homValue{76.52}{0.71} & \cellcolor{homSecond}\homValue{\underline{79.64}}{0.77} & \cellcolor{homThird}\homValue{\textit{79.56}}{0.40} & \homValue{78.86}{0.51} & \homValue{77.10}{0.66} & \homValue{74.76}{0.72} & \homValue{74.76}{0.72} & \homValue{79.12}{0.66} & \homValue{77.53}{0.75} & \homValue{77.17}{0.56} & \homValue{76.73}{0.45} & \homValue{73.53}{0.45} & \homValue{76.83}{0.63} & \homValue{76.93}{0.53} & \cellcolor{homBest}\homValue{\textbf{80.32}}{0.72} \\
\rowcolor{white}[3pt][3pt]  \multicolumn{1}{l}{\cellcolor{white}} & 0.7 & \multicolumn{1}{c}{\cellcolor{white}} & \multicolumn{1}{c}{\cellcolor{white}} & \multicolumn{1}{c}{\cellcolor{white}} & \homValue{81.02}{0.38} & \homValue{76.82}{0.55} & \cellcolor{homThird}\homValue{\textit{80.56}}{0.83} & \homValue{80.18}{0.48} & \homValue{78.52}{0.99} & \homValue{79.02}{0.63} & \homValue{77.90}{1.08} & \homValue{77.90}{1.08} & \homValue{80.42}{0.26} & \homValue{79.04}{0.41} & \homValue{78.20}{0.58} & \cellcolor{homBest}\homValue{\textbf{81.00}}{0.59} & \homValue{77.00}{1.77} & \homValue{77.94}{1.72} & \homValue{76.72}{1.62} & \cellcolor{homSecond}\homValue{\underline{80.62}}{0.54} \\
\midrule[0.25pt]
\rowcolor{homBand}[3pt][3pt]  \multicolumn{1}{l}{\cellcolor{white}} & 0.3 & \multicolumn{1}{c}{\cellcolor{white}} & \multicolumn{1}{c}{\cellcolor{white}} & \multicolumn{1}{c}{\cellcolor{white}} & \homValue{93.59}{0.05} & \homValue{91.05}{0.15} & \homValue{89.28}{0.33} & \homValue{88.04}{0.28} & \homValue{92.57}{0.12} & \homValue{91.60}{0.18} & \homValue{90.95}{0.22} & \homValue{91.02}{0.21} & \cellcolor{homThird}\homValue{\textit{93.04}}{0.07} & \homValue{92.13}{0.11} & \homValue{91.51}{0.11} & \cellcolor{homSecond}\homValue{\underline{93.17}}{0.08} & \homValue{86.88}{0.74} & \homValue{91.53}{0.42} & \homValue{92.55}{0.05} & \cellcolor{homBest}\homValue{\textbf{93.35}}{0.21} \\
\rowcolor{white}[3pt][3pt]  \multicolumn{1}{l}{\cellcolor{white}\textbf{Amazon-Computer}} & 0.5 & \multicolumn{1}{c}{\cellcolor{white}\homValue{86.99}{0.39}} & \multicolumn{1}{c}{\cellcolor{white}\homValue{89.76}{0.30}} & \multicolumn{1}{c}{\cellcolor{white}\homValue{93.70}{0.31}} & \homValue{93.75}{0.17} & \homValue{92.17}{0.30} & \homValue{90.59}{0.40} & \homValue{89.06}{0.08} & \cellcolor{homThird}\homValue{\textit{93.55}}{0.17} & \homValue{92.37}{0.05} & \homValue{91.98}{0.14} & \homValue{92.03}{0.20} & \homValue{93.34}{0.15} & \homValue{93.32}{0.06} & \homValue{92.39}{0.11} & \cellcolor{homBest}\homValue{\textbf{93.75}}{0.13} & \homValue{88.76}{0.55} & \homValue{91.45}{0.06} & \homValue{92.23}{0.24} & \cellcolor{homSecond}\homValue{\underline{93.60}}{0.14} \\
\rowcolor{homBand}[3pt][3pt]  \multicolumn{1}{l}{\cellcolor{white}} & 0.7 & \multicolumn{1}{c}{\cellcolor{white}} & \multicolumn{1}{c}{\cellcolor{white}} & \multicolumn{1}{c}{\cellcolor{white}} & \homValue{93.87}{0.24} & \homValue{93.30}{0.23} & \homValue{91.48}{0.13} & \homValue{90.63}{0.25} & \homValue{93.29}{0.36} & \homValue{93.19}{0.11} & \homValue{92.58}{0.22} & \homValue{92.34}{0.05} & \cellcolor{homThird}\homValue{\textit{93.58}}{0.22} & \homValue{93.41}{0.17} & \homValue{93.51}{0.18} & \cellcolor{homBest}\homValue{\textbf{94.07}}{0.08} & \homValue{91.08}{0.07} & \homValue{92.32}{0.21} & \homValue{93.55}{0.14} & \cellcolor{homSecond}\homValue{\underline{93.88}}{0.11} \\
\midrule[0.25pt]
\rowcolor{white}[3pt][3pt]  \multicolumn{1}{l}{\cellcolor{white}} & 0.3 & \multicolumn{1}{c}{\cellcolor{white}} & \multicolumn{1}{c}{\cellcolor{white}} & \multicolumn{1}{c}{\cellcolor{white}} & \homValue{96.14}{0.19} & \homValue{94.38}{0.16} & \homValue{94.49}{0.03} & \homValue{93.79}{0.05} & \homValue{95.40}{0.22} & \homValue{94.01}{0.22} & \homValue{94.07}{0.20} & \homValue{93.42}{0.37} & \homValue{94.88}{0.13} & \homValue{95.24}{0.29} & \homValue{94.95}{0.24} & \cellcolor{homBest}\homValue{\textbf{95.62}}{0.33} & \homValue{91.94}{0.36} & \homValue{94.16}{0.19} & \cellcolor{homThird}\homValue{\textit{95.45}}{0.29} & \cellcolor{homSecond}\homValue{\underline{95.60}}{0.15} \\
\rowcolor{homBand}[3pt][3pt]  \multicolumn{1}{l}{\cellcolor{white}\textbf{Amazon-Photo}} & 0.5 & \multicolumn{1}{c}{\cellcolor{white}\homValue{92.16}{0.24}} & \multicolumn{1}{c}{\cellcolor{white}\homValue{92.46}{0.12}} & \multicolumn{1}{c}{\cellcolor{white}\homValue{95.95}{0.14}} & \homValue{96.06}{0.31} & \homValue{95.45}{0.16} & \homValue{94.42}{0.45} & \homValue{94.29}{0.11} & \cellcolor{homSecond}\homValue{\underline{95.62}}{0.16} & \homValue{94.60}{0.13} & \homValue{94.23}{0.27} & \homValue{94.55}{0.26} & \homValue{95.25}{0.27} & \homValue{95.19}{0.29} & \homValue{95.14}{0.29} & \cellcolor{homThird}\homValue{\textit{95.51}}{0.11} & \homValue{93.27}{0.70} & \homValue{94.99}{0.30} & \homValue{95.45}{0.36} & \cellcolor{homBest}\homValue{\textbf{95.82}}{0.19} \\
\rowcolor{white}[3pt][3pt]  \multicolumn{1}{l}{\cellcolor{white}} & 0.7 & \multicolumn{1}{c}{\cellcolor{white}} & \multicolumn{1}{c}{\cellcolor{white}} & \multicolumn{1}{c}{\cellcolor{white}} & \homValue{95.71}{0.11} & \cellcolor{homThird}\homValue{\textit{95.84}}{0.25} & \homValue{94.75}{0.29} & \homValue{94.92}{0.20} & \homValue{95.64}{0.06} & \homValue{95.16}{0.11} & \homValue{95.10}{0.21} & \homValue{94.86}{0.43} & \homValue{95.51}{0.11} & \cellcolor{homBest}\homValue{\textbf{96.33}}{0.32} & \cellcolor{homThird}\homValue{\textit{95.84}}{0.32} & \homValue{95.73}{0.13} & \homValue{93.97}{0.32} & \homValue{94.95}{0.29} & \homValue{95.64}{0.08} & \cellcolor{homSecond}\homValue{\underline{95.88}}{0.05} \\
\midrule[0.25pt]
\rowcolor{homBand}[3pt][3pt]  \multicolumn{1}{l}{\cellcolor{white}} & 0.3 & \multicolumn{1}{c}{\cellcolor{white}} & \multicolumn{1}{c}{\cellcolor{white}} & \multicolumn{1}{c}{\cellcolor{white}} & \homValue{95.98}{0.15} & \homValue{94.71}{0.04} & \homValue{94.93}{0.22} & \homValue{94.68}{0.06} & \homValue{94.67}{0.33} & \cellcolor{homThird}\homValue{\textit{95.09}}{0.02} & \homValue{95.03}{0.22} & \homValue{94.78}{0.08} & \homValue{95.02}{0.29} & \homValue{94.68}{0.13} & \homValue{94.64}{0.27} & \homValue{94.52}{0.12} & \homValue{94.99}{0.37} & \cellcolor{homSecond}\homValue{\underline{95.13}}{0.08} & \homValue{94.94}{0.20} & \cellcolor{homBest}\homValue{\textbf{95.72}}{0.06} \\
\rowcolor{white}[3pt][3pt]  \multicolumn{1}{l}{\cellcolor{white}\textbf{Coauthor-CS}} & 0.5 & \multicolumn{1}{c}{\cellcolor{white}\homValue{95.11}{0.05}} & \multicolumn{1}{c}{\cellcolor{white}\homValue{94.68}{0.14}} & \multicolumn{1}{c}{\cellcolor{white}\homValue{95.61}{0.36}} & \homValue{95.95}{0.13} & \homValue{95.42}{0.16} & \homValue{94.97}{0.11} & \homValue{94.97}{0.07} & \homValue{95.29}{0.22} & \homValue{95.04}{0.20} & \homValue{95.28}{0.12} & \homValue{95.12}{0.14} & \homValue{95.18}{0.14} & \cellcolor{homSecond}\homValue{\underline{95.45}}{0.16} & \homValue{95.18}{0.15} & \homValue{95.06}{0.08} & \homValue{95.04}{0.60} & \cellcolor{homThird}\homValue{\textit{95.44}}{0.20} & \homValue{95.41}{0.17} & \cellcolor{homBest}\homValue{\textbf{95.86}}{0.15} \\
\rowcolor{homBand}[3pt][3pt]  \multicolumn{1}{l}{\cellcolor{white}} & 0.7 & \multicolumn{1}{c}{\cellcolor{white}} & \multicolumn{1}{c}{\cellcolor{white}} & \multicolumn{1}{c}{\cellcolor{white}} & \homValue{95.85}{0.05} & \homValue{95.57}{0.27} & \homValue{95.48}{0.01} & \homValue{95.63}{0.08} & \homValue{95.55}{0.11} & \homValue{95.60}{0.19} & \cellcolor{homThird}\homValue{\textit{95.70}}{0.10} & \homValue{95.47}{0.04} & \homValue{95.59}{0.12} & \homValue{95.46}{0.13} & \homValue{95.31}{0.09} & \cellcolor{homSecond}\homValue{\underline{95.76}}{0.14} & \homValue{95.26}{0.14} & \homValue{95.53}{0.10} & \homValue{95.56}{0.10} & \cellcolor{homBest}\homValue{\textbf{95.95}}{0.11} \\
\midrule[0.25pt]
\rowcolor{white}[3pt][3pt]  \multicolumn{1}{l}{\cellcolor{white}} & 0.3 & \multicolumn{1}{c}{\cellcolor{white}} & \multicolumn{1}{c}{\cellcolor{white}} & \multicolumn{1}{c}{\cellcolor{white}} & \homValue{97.45}{0.03} & \homValue{96.66}{0.05} & \homValue{96.30}{0.18} & \homValue{96.40}{0.17} & \cellcolor{homThird}\homValue{\textit{96.91}}{0.03} & \homValue{96.84}{0.03} & \homValue{96.74}{0.05} & \homValue{96.73}{0.06} & \homValue{96.83}{0.06} & \homValue{96.74}{0.08} & \homValue{96.75}{0.10} & \homValue{96.84}{0.05} & \homValue{96.49}{0.00} & \homValue{96.65}{0.00} & \cellcolor{homSecond}\homValue{\underline{96.93}}{0.10} & \cellcolor{homBest}\homValue{\textbf{97.06}}{0.09} \\
\rowcolor{homBand}[3pt][3pt]  \multicolumn{1}{l}{\cellcolor{white}\textbf{Coauthor-Physics}} & 0.5 & \multicolumn{1}{c}{\cellcolor{white}\homValue{96.36}{0.01}} & \multicolumn{1}{c}{\cellcolor{white}\homValue{96.94}{0.02}} & \multicolumn{1}{c}{\cellcolor{white}\homValue{97.51}{0.07}} & \homValue{97.43}{0.10} & \homValue{96.86}{0.05} & \homValue{96.85}{0.07} & \homValue{96.77}{0.14} & \cellcolor{homSecond}\homValue{\underline{97.22}}{0.04} & \homValue{97.02}{0.06} & \homValue{96.92}{0.07} & \homValue{96.95}{0.05} & \homValue{97.01}{0.04} & \homValue{97.15}{0.07} & \homValue{97.00}{0.09} & \cellcolor{homThird}\homValue{\textit{97.21}}{0.12} & \homValue{96.74}{0.00} & \homValue{96.81}{0.00} & \homValue{97.01}{0.12} & \cellcolor{homBest}\homValue{\textbf{97.34}}{0.05} \\
\rowcolor{white}[3pt][3pt]  \multicolumn{1}{l}{\cellcolor{white}} & 0.7 & \multicolumn{1}{c}{\cellcolor{white}} & \multicolumn{1}{c}{\cellcolor{white}} & \multicolumn{1}{c}{\cellcolor{white}} & \homValue{97.47}{0.03} & \homValue{97.21}{0.05} & \homValue{97.04}{0.07} & \homValue{96.77}{0.04} & \cellcolor{homThird}\homValue{\textit{97.29}}{0.08} & \homValue{97.21}{0.08} & \homValue{97.10}{0.04} & \homValue{97.01}{0.09} & \homValue{97.24}{0.04} & \homValue{97.11}{0.05} & \homValue{97.27}{0.04} & \cellcolor{homSecond}\homValue{\underline{97.37}}{0.11} & \homValue{96.67}{0.00} & \homValue{97.01}{0.00} & \cellcolor{homThird}\homValue{\textit{97.29}}{0.09} & \cellcolor{homBest}\homValue{\textbf{97.47}}{0.07} \\
\midrule[0.25pt]
\rowcolor{homBand}[3pt][3pt]  \multicolumn{1}{l}{\cellcolor{white}} & 0.3 & \multicolumn{1}{c}{\cellcolor{white}} & \multicolumn{1}{c}{\cellcolor{white}} & \multicolumn{1}{c}{\cellcolor{white}} & \homValue{80.39}{0.29} & \homValue{77.04}{0.63} & \homValue{75.69}{0.48} & \homValue{76.13}{0.53} & \cellcolor{homThird}\homValue{\textit{79.20}}{0.58} & \homValue{75.08}{0.44} & \homValue{75.10}{0.25} & \homValue{74.91}{0.39} & \homValue{78.71}{0.96} & \homValue{78.80}{0.47} & \homValue{78.07}{0.23} & \cellcolor{homSecond}\homValue{\underline{79.47}}{0.68} & \homValue{75.49}{0.17} & \homValue{76.48}{0.13} & \homValue{77.48}{0.66} & \cellcolor{homBest}\homValue{\textbf{80.25}}{0.79} \\
\rowcolor{white}[3pt][3pt]  \multicolumn{1}{l}{\cellcolor{white}\textbf{WikiCS}} & 0.5 & \multicolumn{1}{c}{\cellcolor{white}\homValue{73.10}{0.19}} & \multicolumn{1}{c}{\cellcolor{white}\homValue{75.05}{0.23}} & \multicolumn{1}{c}{\cellcolor{white}\homValue{80.57}{0.06}} & \homValue{80.33}{0.32} & \homValue{78.99}{0.09} & \homValue{76.59}{0.15} & \homValue{76.44}{0.30} & \cellcolor{homBest}\homValue{\textbf{80.67}}{0.16} & \homValue{75.02}{0.41} & \homValue{75.40}{0.49} & \homValue{75.34}{0.62} & \homValue{79.32}{0.21} & \homValue{79.49}{0.62} & \homValue{78.38}{0.50} & \cellcolor{homThird}\homValue{\textit{80.13}}{0.51} & \homValue{76.73}{0.57} & \homValue{77.85}{0.32} & \homValue{78.33}{0.66} & \cellcolor{homSecond}\homValue{\underline{80.19}}{0.33} \\
\rowcolor{homBand}[3pt][3pt]  \multicolumn{1}{l}{\cellcolor{white}} & 0.7 & \multicolumn{1}{c}{\cellcolor{white}} & \multicolumn{1}{c}{\cellcolor{white}} & \multicolumn{1}{c}{\cellcolor{white}} & \homValue{80.17}{0.40} & \homValue{78.57}{0.54} & \homValue{78.38}{0.12} & \homValue{77.26}{0.32} & \homValue{79.54}{0.66} & \homValue{77.50}{0.39} & \homValue{77.15}{0.35} & \homValue{77.21}{0.24} & \homValue{79.45}{0.24} & \homValue{79.51}{0.34} & \cellcolor{homThird}\homValue{\textit{79.75}}{0.55} & \cellcolor{homBest}\homValue{\textbf{80.31}}{0.49} & \homValue{77.74}{0.60} & \homValue{78.40}{0.32} & \homValue{79.53}{0.62} & \cellcolor{homSecond}\homValue{\underline{80.17}}{0.43} \\
\bottomrule
\end{tabular*}
\par\vspace{4pt}
{\fontsize{7pt}{8pt}\selectfont\raggedright  Each cell reports the mean above $\pm$ standard deviation. \textbf{Bold} (teal), \underline{underline} (gold), and \textit{italic} (lavender) identify the highest, second-highest, and third-highest distinct means within each row, respectively; ties share a rank. Reference anchors are excluded from ranking. EDiS-Lite and EDiS-Full use sampled-graph and full-graph evaluation, respectively.\par}
\end{minipage}
\end{adjustbox}
\end{sidewaystable}
\endgroup

\begingroup
\definecolor{homBest}{RGB}{183,220,228}
\definecolor{homBand}{RGB}{235,235,235}
\definecolor{homHeader}{RGB}{222,226,230}
\definecolor{homInk}{RGB}{20,20,20}
\definecolor{homSecond}{RGB}{249,231,191}
\definecolor{homThird}{RGB}{231,221,242}
\newcommand{\homValue}[2]{\shortstack{#1\\[-1pt]{\fontsize{5.5pt}{6pt}\selectfont\color{black!80}\textpm#2}}}
\newlength{\hetPageHeight}
\setlength{\hetPageHeight}{0.98\textwidth}
\begin{sidewaystable}[p]
\centering
\begin{adjustbox}{max totalheight=\hetPageHeight}
\begin{minipage}{0.98\textheight}
\centering
\caption{\textbf{Predictive performance on heterophilic node-classification benchmarks.} The metric is accuracy for Squirrel, Chameleon, Roman-Empire, and Amazon-Ratings, and ROC-AUC for Minesweeper and Questions. Results are reported as mean and standard deviation (\%) at retained-edge fractions $\rho\in\{0.3,0.5,0.7\}$; higher values indicate better performance.}
\label{tab:main-het-acc}
\label{app:full_het_acc}
\label{tab:main-het-rocauc}
\label{app:full_het_rocauc}
\vspace{3pt}
\fontfamily{ptm}\fontsize{7.5pt}{8pt}\selectfont
\color{homInk}
\setlength{\tabcolsep}{1.5pt}
\renewcommand{\arraystretch}{1.0}
\begin{tabular*}{0.98\textheight}{@{\extracolsep{\fill}}lc*{19}{c}@{}}
\toprule
\rowcolor{homHeader}[3pt][3pt]
\textbf{Dataset} & $\boldsymbol{\rho}$ & \multicolumn{4}{c}{\textbf{Reference anchors}} & \multicolumn{11}{c}{\textbf{Topology-based}} & \multicolumn{3}{c}{\textbf{Semantic-based}} & \multicolumn{1}{c}{\textbf{EDiS (Ours)}} \\
\cmidrule(lr){1-2} \cmidrule(lr){3-6} \cmidrule(lr){7-17} \cmidrule(lr){18-20} \cmidrule(lr){21-21}
 & & \shortstack{MLP\\(no graph)} & \shortstack{Spanning\\Tree} & \shortstack{Full\\GCN} & \shortstack{EDiS-\\Full} & \shortstack{Random} & \shortstack{Local\\Degree} & \shortstack{Rank\\Degree} & \shortstack{Forest\\Fire} & \shortstack{SCAN} & \shortstack{L-Spar} & \shortstack{G-Spar} & \shortstack{L-Sim} & \shortstack{Spectral} & \shortstack{$t$-Spanner} & \shortstack{DSpar} & \shortstack{Unified-\\LTH} & \shortstack{AdaGLT} & \shortstack{MoG} & \shortstack{EDiS-\\Lite} \\
\midrule
\rowcolor{white}[3pt][3pt] \multicolumn{1}{l}{\cellcolor{white}} & 0.3 & \multicolumn{1}{c}{\cellcolor{white}} & \multicolumn{1}{c}{\cellcolor{white}} & \multicolumn{1}{c}{\cellcolor{white}} & \homValue{43.26}{2.43} & \homValue{40.90}{1.70} & \homValue{41.80}{1.66} & \homValue{40.17}{1.89} & \homValue{43.15}{1.90} & \homValue{39.50}{1.10} & \homValue{39.17}{1.31} & \homValue{39.77}{1.57} & \homValue{39.23}{2.38} & \cellcolor{homBest}\homValue{\textbf{43.94}}{1.18} & \homValue{42.65}{2.01} & \cellcolor{homSecond}\homValue{\underline{43.48}}{1.56} & \homValue{38.17}{2.09} & \homValue{42.30}{3.91} & \cellcolor{homThird}\homValue{\textit{43.25}}{1.12} & \homValue{42.63}{2.04} \\
\rowcolor{homBand}[3pt][3pt] \multicolumn{1}{l}{\cellcolor{white}\textbf{Squirrel}} & 0.5 & \multicolumn{1}{c}{\cellcolor{white}\homValue{39.77}{1.59}} & \multicolumn{1}{c}{\cellcolor{white}\homValue{36.93}{1.58}} & \multicolumn{1}{c}{\cellcolor{white}\homValue{44.79}{1.82}} & \multicolumn{1}{c}{\cellcolor{white}\homValue{43.89}{2.09}} & \homValue{42.54}{2.39} & \homValue{42.49}{1.69} & \homValue{41.17}{2.36} & \homValue{44.61}{1.98} & \homValue{39.08}{1.64} & \homValue{38.92}{1.59} & \homValue{38.96}{1.62} & \homValue{40.07}{1.84} & \homValue{44.08}{2.02} & \cellcolor{homThird}\homValue{\textit{44.75}}{1.84} & \cellcolor{homBest}\homValue{\textbf{46.83}}{1.62} & \homValue{39.26}{0.38} & \cellcolor{homSecond}\homValue{\underline{45.42}}{1.27} & \homValue{41.69}{2.58} & \homValue{44.50}{1.69} \\
\rowcolor{white}[3pt][3pt] \multicolumn{1}{l}{\cellcolor{white}} & 0.7 & \multicolumn{1}{c}{\cellcolor{white}} & \multicolumn{1}{c}{\cellcolor{white}} & \multicolumn{1}{c}{\cellcolor{white}} & \multicolumn{1}{c}{\cellcolor{white}\homValue{44.03}{2.18}} & \homValue{43.06}{1.63} & \homValue{43.45}{1.79} & \homValue{42.02}{2.03} & \homValue{43.15}{2.40} & \homValue{38.80}{1.70} & \homValue{38.22}{1.83} & \homValue{38.60}{1.71} & \homValue{43.31}{1.07} & \cellcolor{homBest}\homValue{\textbf{45.58}}{1.32} & \homValue{43.11}{2.30} & \cellcolor{homSecond}\homValue{\underline{44.55}}{1.55} & \homValue{36.66}{2.89} & \homValue{43.25}{1.78} & \cellcolor{homThird}\homValue{\textit{44.19}}{2.22} & \homValue{44.15}{1.65} \\
\midrule[0.25pt]
\rowcolor{homBand}[3pt][3pt] \multicolumn{1}{l}{\cellcolor{white}} & 0.3 & \multicolumn{1}{c}{\cellcolor{white}} & \multicolumn{1}{c}{\cellcolor{white}} & \multicolumn{1}{c}{\cellcolor{white}} & \homValue{43.13}{2.53} & \homValue{37.85}{3.35} & \homValue{41.37}{2.59} & \homValue{40.35}{3.82} & \homValue{41.28}{1.96} & \homValue{36.05}{3.41} & \homValue{36.11}{3.06} & \homValue{38.62}{4.13} & \homValue{39.71}{3.97} & \cellcolor{homThird}\homValue{\textit{42.98}}{3.60} & \homValue{38.71}{3.72} & \homValue{39.73}{3.87} & \homValue{37.61}{3.67} & \homValue{38.03}{4.12} & \cellcolor{homSecond}\homValue{\underline{43.40}}{3.13} & \cellcolor{homBest}\homValue{\textbf{43.75}}{4.14} \\
\rowcolor{white}[3pt][3pt] \multicolumn{1}{l}{\cellcolor{white}\textbf{Chameleon}} & 0.5 & \multicolumn{1}{c}{\cellcolor{white}\homValue{36.59}{3.32}} & \multicolumn{1}{c}{\cellcolor{white}\homValue{34.72}{2.81}} & \multicolumn{1}{c}{\cellcolor{white}\homValue{44.67}{3.84}} & \multicolumn{1}{c}{\cellcolor{white}\homValue{44.14}{4.15}} & \homValue{39.70}{3.51} & \homValue{43.58}{4.85} & \homValue{42.76}{2.99} & \cellcolor{homThird}\homValue{\textit{44.82}}{3.43} & \homValue{37.69}{3.95} & \homValue{38.99}{4.06} & \homValue{38.88}{4.64} & \homValue{41.40}{3.50} & \homValue{42.12}{3.40} & \homValue{40.71}{2.17} & \cellcolor{homBest}\homValue{\textbf{46.14}}{0.86} & \homValue{34.25}{3.09} & \homValue{36.67}{1.40} & \homValue{42.78}{3.18} & \cellcolor{homSecond}\homValue{\underline{45.64}}{2.72} \\
\rowcolor{homBand}[3pt][3pt] \multicolumn{1}{l}{\cellcolor{white}} & 0.7 & \multicolumn{1}{c}{\cellcolor{white}} & \multicolumn{1}{c}{\cellcolor{white}} & \multicolumn{1}{c}{\cellcolor{white}} & \multicolumn{1}{c}{\cellcolor{white}\homValue{43.50}{4.27}} & \homValue{44.08}{3.81} & \homValue{44.37}{4.33} & \homValue{43.10}{2.08} & \homValue{42.93}{4.54} & \homValue{35.87}{2.91} & \homValue{36.92}{3.22} & \homValue{36.09}{3.61} & \homValue{41.21}{2.69} & \cellcolor{homBest}\homValue{\textbf{45.95}}{1.83} & \cellcolor{homSecond}\homValue{\underline{45.47}}{3.95} & \cellcolor{homBest}\homValue{\textbf{45.95}}{4.48} & \homValue{36.55}{4.18} & \homValue{36.72}{3.10} & \homValue{44.64}{2.60} & \cellcolor{homThird}\homValue{\textit{44.88}}{4.25} \\
\midrule[0.25pt]
\rowcolor{white}[3pt][3pt] \multicolumn{1}{l}{\cellcolor{white}} & 0.3 & \multicolumn{1}{c}{\cellcolor{white}} & \multicolumn{1}{c}{\cellcolor{white}} & \multicolumn{1}{c}{\cellcolor{white}} & \homValue{75.00}{0.51} & \homValue{67.07}{0.24} & \homValue{70.79}{0.31} & \homValue{70.53}{0.33} & \homValue{67.43}{0.62} & \cellcolor{homSecond}\homValue{\underline{73.37}}{0.23} & \cellcolor{homThird}\homValue{\textit{72.24}}{0.66} & \homValue{72.20}{0.42} & \homValue{71.33}{0.21} & \homValue{66.72}{0.42} & \homValue{67.09}{0.45} & \homValue{67.07}{0.09} & \homValue{65.43}{0.58} & \cellcolor{homBest}\homValue{\textbf{81.43}}{0.00} & \homValue{70.44}{0.60} & \homValue{68.58}{0.12} \\
\rowcolor{homBand}[3pt][3pt] \multicolumn{1}{l}{\cellcolor{white}\textbf{Roman-Empire}} & 0.5 & \multicolumn{1}{c}{\cellcolor{white}\homValue{66.23}{0.07}} & \multicolumn{1}{c}{\cellcolor{white}\homValue{87.36}{0.23}} & \multicolumn{1}{c}{\cellcolor{white}\homValue{90.74}{0.03}} & \multicolumn{1}{c}{\cellcolor{white}\homValue{79.61}{0.28}} & \homValue{71.38}{0.29} & \cellcolor{homThird}\homValue{\textit{76.52}}{0.15} & \homValue{75.72}{0.28} & \homValue{72.66}{0.21} & \cellcolor{homSecond}\homValue{\underline{76.53}}{0.41} & \homValue{76.17}{0.36} & \homValue{76.04}{0.14} & \cellcolor{homBest}\homValue{\textbf{78.39}}{0.17} & \homValue{70.83}{0.25} & \homValue{71.18}{0.38} & \homValue{70.91}{0.14} & \homValue{61.90}{0.64} & \homValue{68.13}{0.00} & \homValue{73.91}{0.63} & \homValue{73.71}{0.57} \\
\rowcolor{white}[3pt][3pt] \multicolumn{1}{l}{\cellcolor{white}} & 0.7 & \multicolumn{1}{c}{\cellcolor{white}} & \multicolumn{1}{c}{\cellcolor{white}} & \multicolumn{1}{c}{\cellcolor{white}} & \multicolumn{1}{c}{\cellcolor{white}\homValue{86.19}{0.53}} & \homValue{77.92}{0.41} & \cellcolor{homSecond}\homValue{\underline{83.35}}{0.22} & \homValue{82.73}{0.19} & \homValue{78.76}{0.17} & \homValue{80.49}{0.74} & \homValue{79.97}{0.27} & \homValue{80.46}{0.18} & \homValue{82.33}{0.25} & \homValue{76.26}{0.63} & \homValue{78.17}{0.19} & \homValue{77.26}{0.66} & \homValue{70.81}{0.37} & \cellcolor{homBest}\homValue{\textbf{90.08}}{0.00} & \homValue{79.65}{0.37} & \cellcolor{homThird}\homValue{\textit{83.03}}{0.32} \\
\midrule[0.25pt]
\rowcolor{homBand}[3pt][3pt] \multicolumn{1}{l}{\cellcolor{white}} & 0.3 & \multicolumn{1}{c}{\cellcolor{white}} & \multicolumn{1}{c}{\cellcolor{white}} & \multicolumn{1}{c}{\cellcolor{white}} & \homValue{53.16}{0.68} & \homValue{48.40}{0.43} & \cellcolor{homThird}\homValue{\textit{51.47}}{0.71} & \homValue{51.23}{0.52} & \homValue{48.37}{0.22} & \homValue{50.25}{0.61} & \homValue{49.84}{0.19} & \homValue{49.64}{0.61} & \cellcolor{homSecond}\homValue{\underline{51.66}}{0.09} & \homValue{46.86}{0.41} & \homValue{47.96}{1.01} & \homValue{47.71}{0.61} & \homValue{40.49}{0.52} & \homValue{48.05}{0.35} & \homValue{51.24}{0.33} & \cellcolor{homBest}\homValue{\textbf{52.31}}{0.85} \\
\rowcolor{white}[3pt][3pt] \multicolumn{1}{l}{\cellcolor{white}\textbf{Amazon-Ratings}} & 0.5 & \multicolumn{1}{c}{\cellcolor{white}\homValue{50.32}{0.24}} & \multicolumn{1}{c}{\cellcolor{white}\homValue{51.73}{0.37}} & \multicolumn{1}{c}{\cellcolor{white}\homValue{53.52}{0.42}} & \multicolumn{1}{c}{\cellcolor{white}\homValue{54.12}{0.77}} & \homValue{51.43}{0.14} & \homValue{52.06}{0.17} & \homValue{51.38}{0.48} & \homValue{51.16}{0.28} & \homValue{51.80}{0.41} & \homValue{51.95}{0.21} & \homValue{51.75}{0.31} & \cellcolor{homThird}\homValue{\textit{52.16}}{0.20} & \homValue{50.16}{0.25} & \homValue{50.59}{0.18} & \homValue{50.63}{0.33} & \homValue{43.36}{0.79} & \homValue{48.66}{0.22} & \cellcolor{homBest}\homValue{\textbf{53.86}}{0.66} & \cellcolor{homSecond}\homValue{\underline{53.62}}{0.58} \\
\rowcolor{homBand}[3pt][3pt] \multicolumn{1}{l}{\cellcolor{white}} & 0.7 & \multicolumn{1}{c}{\cellcolor{white}} & \multicolumn{1}{c}{\cellcolor{white}} & \multicolumn{1}{c}{\cellcolor{white}} & \multicolumn{1}{c}{\cellcolor{white}\homValue{53.27}{0.84}} & \homValue{52.26}{0.09} & \homValue{52.14}{0.69} & \homValue{51.59}{0.35} & \homValue{52.08}{0.32} & \homValue{52.34}{0.12} & \homValue{52.60}{0.20} & \homValue{52.49}{0.11} & \homValue{52.45}{0.26} & \homValue{51.40}{0.40} & \cellcolor{homThird}\homValue{\textit{52.65}}{0.61} & \homValue{52.38}{0.20} & \homValue{47.15}{0.27} & \homValue{51.64}{0.33} & \cellcolor{homBest}\homValue{\textbf{53.28}}{0.34} & \cellcolor{homSecond}\homValue{\underline{53.14}}{0.27} \\
\midrule[0.25pt]
\rowcolor{white}[3pt][3pt] \multicolumn{1}{l}{\cellcolor{white}} & 0.3 & \multicolumn{1}{c}{\cellcolor{white}} & \multicolumn{1}{c}{\cellcolor{white}} & \multicolumn{1}{c}{\cellcolor{white}} & \homValue{82.62}{1.09} & \homValue{72.39}{1.04} & \homValue{72.53}{1.15} & \homValue{71.09}{0.21} & \homValue{70.76}{0.29} & \cellcolor{homSecond}\homValue{\underline{78.53}}{0.35} & \cellcolor{homThird}\homValue{\textit{78.45}}{0.86} & \cellcolor{homBest}\homValue{\textbf{78.97}}{0.51} & \homValue{77.17}{0.41} & \homValue{70.48}{1.94} & \homValue{71.13}{1.03} & \homValue{72.56}{2.00} & \homValue{69.70}{1.12} & \homValue{62.60}{1.27} & \homValue{74.02}{0.32} & \homValue{73.82}{0.97} \\
\rowcolor{homBand}[3pt][3pt] \multicolumn{1}{l}{\cellcolor{white}\textbf{Minesweeper}} & 0.5 & \multicolumn{1}{c}{\cellcolor{white}\homValue{51.34}{0.77}} & \multicolumn{1}{c}{\cellcolor{white}\homValue{70.73}{0.21}} & \multicolumn{1}{c}{\cellcolor{white}\homValue{97.49}{0.18}} & \multicolumn{1}{c}{\cellcolor{white}\homValue{87.54}{0.32}} & \homValue{80.23}{0.11} & \homValue{80.12}{0.46} & \homValue{78.67}{0.76} & \homValue{78.94}{0.63} & \homValue{90.61}{0.12} & \cellcolor{homThird}\homValue{\textit{90.87}}{0.12} & \cellcolor{homSecond}\homValue{\underline{90.98}}{0.20} & \cellcolor{homBest}\homValue{\textbf{91.02}}{0.31} & \homValue{78.34}{0.60} & \homValue{79.47}{1.41} & \homValue{80.79}{0.54} & \homValue{73.37}{2.39} & \homValue{69.93}{1.52} & \homValue{82.53}{0.43} & \homValue{82.21}{0.60} \\
\rowcolor{white}[3pt][3pt] \multicolumn{1}{l}{\cellcolor{white}} & 0.7 & \multicolumn{1}{c}{\cellcolor{white}} & \multicolumn{1}{c}{\cellcolor{white}} & \multicolumn{1}{c}{\cellcolor{white}} & \multicolumn{1}{c}{\cellcolor{white}\homValue{89.84}{0.28}} & \homValue{85.18}{0.50} & \homValue{86.30}{0.37} & \homValue{86.26}{0.81} & \homValue{85.26}{0.54} & \cellcolor{homThird}\homValue{\textit{90.20}}{0.33} & \cellcolor{homSecond}\homValue{\underline{90.43}}{0.22} & \cellcolor{homBest}\homValue{\textbf{90.46}}{0.10} & \homValue{89.28}{0.07} & \homValue{83.95}{0.53} & \homValue{85.22}{0.72} & \homValue{85.51}{0.53} & \homValue{82.40}{0.81} & \homValue{81.28}{1.45} & \homValue{87.53}{0.14} & \homValue{85.37}{0.59} \\
\midrule[0.25pt]
\rowcolor{homBand}[3pt][3pt] \multicolumn{1}{l}{\cellcolor{white}} & 0.3 & \multicolumn{1}{c}{\cellcolor{white}} & \multicolumn{1}{c}{\cellcolor{white}} & \multicolumn{1}{c}{\cellcolor{white}} & \homValue{77.70}{0.53} & \homValue{71.89}{0.90} & \cellcolor{homSecond}\homValue{\underline{74.39}}{0.83} & \homValue{71.49}{0.48} & \homValue{71.34}{0.46} & \homValue{72.72}{0.45} & \homValue{70.29}{1.82} & \homValue{70.35}{1.77} & \cellcolor{homBest}\homValue{\textbf{75.39}}{0.60} & \homValue{74.00}{0.88} & \homValue{72.42}{0.61} & \homValue{72.77}{0.47} & \homValue{72.42}{0.64} & \homValue{72.34}{0.33} & \homValue{68.90}{1.68} & \cellcolor{homThird}\homValue{\textit{74.12}}{0.88} \\
\rowcolor{white}[3pt][3pt] \multicolumn{1}{l}{\cellcolor{white}\textbf{Questions}} & 0.5 & \multicolumn{1}{c}{\cellcolor{white}\homValue{69.69}{0.35}} & \multicolumn{1}{c}{\cellcolor{white}\homValue{74.85}{0.24}} & \multicolumn{1}{c}{\cellcolor{white}\homValue{78.49}{0.18}} & \multicolumn{1}{c}{\cellcolor{white}\homValue{78.07}{0.13}} & \homValue{71.92}{0.48} & \homValue{77.16}{0.22} & \homValue{76.28}{0.31} & \homValue{73.45}{0.96} & \homValue{75.36}{0.29} & \homValue{70.87}{1.22} & \homValue{69.46}{0.39} & \cellcolor{homBest}\homValue{\textbf{78.12}}{0.20} & \cellcolor{homSecond}\homValue{\underline{77.35}}{0.69} & \homValue{73.25}{0.49} & \homValue{77.07}{0.59} & \homValue{74.89}{0.84} & \homValue{73.87}{0.73} & \homValue{73.13}{0.14} & \cellcolor{homThird}\homValue{\textit{77.17}}{0.71} \\
\rowcolor{homBand}[3pt][3pt] \multicolumn{1}{l}{\cellcolor{white}} & 0.7 & \multicolumn{1}{c}{\cellcolor{white}} & \multicolumn{1}{c}{\cellcolor{white}} & \multicolumn{1}{c}{\cellcolor{white}} & \multicolumn{1}{c}{\cellcolor{white}\homValue{78.53}{0.38}} & \homValue{75.65}{0.28} & \homValue{77.98}{0.38} & \homValue{77.15}{0.34} & \homValue{75.37}{0.65} & \homValue{77.17}{0.68} & \homValue{73.56}{0.69} & \homValue{73.57}{0.70} & \cellcolor{homSecond}\homValue{\underline{78.37}}{0.44} & \homValue{78.10}{0.54} & \homValue{76.61}{0.69} & \cellcolor{homBest}\homValue{\textbf{78.47}}{0.16} & \homValue{75.97}{0.29} & \homValue{73.36}{0.59} & \homValue{76.00}{0.24} & \cellcolor{homThird}\homValue{\textit{78.21}}{0.46} \\
\bottomrule
\end{tabular*}
\par\vspace{4pt}
{\fontsize{7pt}{8pt}\selectfont\raggedright  Each cell reports the mean above $\pm$ standard deviation. \textbf{Bold} (teal), \underline{underline} (gold), and \textit{italic} (lavender) identify the highest, second-highest, and third-highest distinct means within each row, respectively ties share a rank. Reference anchors are excluded from ranking. EDiS-Full is reported as a full-graph inference reference, while EDiS-Lite uses sampled-graph evaluation.\par}
\end{minipage}
\end{adjustbox}
\end{sidewaystable}
\endgroup

\begingroup
\definecolor{homBest}{RGB}{183,220,228}
\definecolor{homBand}{RGB}{235,235,235}
\definecolor{homHeader}{RGB}{222,226,230}
\definecolor{homInk}{RGB}{20,20,20}
\definecolor{homSecond}{RGB}{249,231,191}
\definecolor{homThird}{RGB}{231,221,242}
\newcommand{\homValue}[2]{\shortstack{#1\\[-1pt]{\fontsize{5.5pt}{6pt}\selectfont\color{black!80}\textpm#2}}}
\newcommand{\homSingle}[1]{#1\textsuperscript{\textdagger}}
\newlength{\largePageHeight}
\setlength{\largePageHeight}{0.98\textwidth}
\begin{sidewaystable}[p]
\centering
\begin{adjustbox}{max totalheight=\largePageHeight}
\begin{minipage}{0.98\textheight}
\centering
\caption{\textbf{Predictive performance on large-scale node-classification benchmarks.} The metric is accuracy for Reddit, ogbn-products, ogbn-arxiv, and Pokec, and ROC-AUC for ogbn-proteins. Results are reported as mean and standard deviation (\%) at retained-edge fractions $\rho\in\{0.1,0.3,0.5\}$; higher values indicate better performance.}
\label{tab:main-large-acc}
\label{app:full_large_acc}
\vspace{3pt}
\fontfamily{ptm}\fontsize{7.5pt}{8pt}\selectfont
\color{homInk}
\setlength{\tabcolsep}{1.5pt}
\renewcommand{\arraystretch}{1.0}
\begin{tabular*}{0.98\textheight}{@{\extracolsep{\fill}}lc*{19}{c}@{}}
\toprule
\rowcolor{homHeader}[3pt][3pt]
\textbf{Dataset} & $\boldsymbol{\rho}$ & \multicolumn{4}{c}{\textbf{Reference anchors}} & \multicolumn{11}{c}{\textbf{Topology-based}} & \multicolumn{3}{c}{\textbf{Semantic-based}} & \multicolumn{1}{c}{\textbf{EDiS (Ours)}} \\
\cmidrule(lr){1-2} \cmidrule(lr){3-6} \cmidrule(lr){7-17} \cmidrule(lr){18-20} \cmidrule(lr){21-21}
 & & \shortstack{MLP\\(no graph)} & \shortstack{Spanning\\Tree} & \shortstack{Full\\GCN} & \shortstack{EDiS-\\Full} & \shortstack{Random} & \shortstack{Local\\Degree} & \shortstack{Rank\\Degree} & \shortstack{Forest\\Fire} & \shortstack{SCAN} & \shortstack{L-Spar} & \shortstack{G-Spar} & \shortstack{L-Sim} & \shortstack{Spectral} & \shortstack{$t$-Spanner} & \shortstack{DSpar} & \shortstack{Unified-\\LTH} & \shortstack{AdaGLT} & \shortstack{MoG} & \shortstack{EDiS-\\Lite} \\
\midrule
\rowcolor{white}[3pt][3pt] \multicolumn{1}{l}{\cellcolor{white}} & 0.1 & \multicolumn{1}{c}{\cellcolor{white}} & \multicolumn{1}{c}{\cellcolor{white}} & \multicolumn{1}{c}{\cellcolor{white}} & \multicolumn{1}{c}{\cellcolor{white}\homValue{95.01}{0.00}} & \homValue{93.13}{0.08} & \homValue{81.85}{0.16} & \homValue{71.85}{0.45} & \homValue{94.97}{0.04} & \homValue{83.96}{0.02} & \homValue{83.61}{0.03} & \homValue{83.54}{0.04} & \homValue{93.39}{0.01} & \cellcolor{homThird}\homValue{\textit{95.06}}{0.01} & \homValue{93.47}{0.00} & \cellcolor{homSecond}\homValue{\underline{95.09}}{0.02} & \homValue{91.90}{0.00} & \homValue{71.45}{0.00} & \homValue{85.15}{0.07} & \cellcolor{homBest}\homValue{\textbf{95.73}}{0.02} \\
\rowcolor{homBand}[3pt][3pt] \multicolumn{1}{l}{\cellcolor{white}\textbf{Reddit}} & 0.3 & \multicolumn{1}{c}{\cellcolor{white}\homValue{69.66}{0.09}} & \multicolumn{1}{c}{\cellcolor{white}\homValue{86.07}{0.14}} & \multicolumn{1}{c}{\cellcolor{white}\homValue{95.44}{0.05}} & \homValue{95.16}{0.01} & \homValue{94.69}{0.02} & \homValue{90.82}{0.15} & \homValue{77.28}{0.04} & \cellcolor{homSecond}\homValue{\underline{95.41}}{0.00} & \homValue{90.28}{0.01} & \homValue{90.06}{0.01} & \homValue{90.05}{0.01} & \homValue{94.35}{0.05} & \cellcolor{homSecond}\homValue{\underline{95.41}}{0.03} & \homValue{94.72}{0.01} & \cellcolor{homThird}\homValue{\textit{95.22}}{0.01} & \homValue{92.68}{0.00} & \homValue{85.09}{0.00} & \homValue{90.13}{0.08} & \cellcolor{homBest}\homValue{\textbf{95.61}}{0.03} \\
\rowcolor{white}[3pt][3pt] \multicolumn{1}{l}{\cellcolor{white}} & 0.5 & \multicolumn{1}{c}{\cellcolor{white}} & \multicolumn{1}{c}{\cellcolor{white}} & \multicolumn{1}{c}{\cellcolor{white}} & \multicolumn{1}{c}{\cellcolor{white}\homValue{95.15}{0.01}} & \homValue{95.04}{0.01} & \homValue{93.32}{0.05} & \homValue{79.04}{0.10} & \cellcolor{homBest}\homValue{\textbf{95.50}}{0.04} & \homValue{92.13}{0.01} & \homValue{91.84}{0.00} & \homValue{91.81}{0.03} & \homValue{94.82}{0.02} & \cellcolor{homSecond}\homValue{\underline{95.44}}{0.05} & \homValue{95.27}{0.02} & \homValue{95.26}{0.01} & \homValue{92.96}{0.00} & \homValue{89.44}{0.00} & \homValue{91.61}{0.09} & \cellcolor{homThird}\homValue{\textit{95.42}}{0.03} \\
\midrule[0.25pt]
\rowcolor{homBand}[3pt][3pt] \multicolumn{1}{l}{\cellcolor{white}} & 0.1 & \multicolumn{1}{c}{\cellcolor{white}} & \multicolumn{1}{c}{\cellcolor{white}} & \multicolumn{1}{c}{\cellcolor{white}} & \multicolumn{1}{c}{\cellcolor{white}\homValue{80.76}{0.00}} & \homValue{71.93}{0.03} & \homValue{73.40}{0.08} & \homValue{61.09}{0.01} & \homValue{73.72}{0.13} & \homValue{65.85}{0.04} & \homValue{65.85}{0.08} & \homValue{65.76}{0.10} & \homValue{73.66}{0.05} & \cellcolor{homThird}\homValue{\textit{73.86}}{0.14} & \homValue{71.84}{0.17} & \cellcolor{homSecond}\homValue{\underline{77.18}}{0.26} & \homValue{71.72}{0.00} & \homValue{61.17}{0.00} & \homValue{63.45}{0.09} & \cellcolor{homBest}\homValue{\textbf{77.50}}{0.00} \\
\rowcolor{white}[3pt][3pt] \multicolumn{1}{l}{\cellcolor{white}\textbf{ogbn-products}} & 0.3 & \multicolumn{1}{c}{\cellcolor{white}\homValue{45.46}{0.19}} & \multicolumn{1}{c}{\cellcolor{white}\homValue{67.79}{0.09}} & \multicolumn{1}{c}{\cellcolor{white}\homValue{80.67}{0.10}} & \homValue{81.38}{0.13} & \homValue{79.08}{0.06} & \homValue{77.62}{0.16} & \homValue{63.13}{0.02} & \cellcolor{homSecond}\homValue{\underline{80.65}}{0.05} & \homValue{72.56}{0.15} & \homValue{72.64}{0.02} & \homValue{72.55}{0.05} & \homValue{79.20}{0.09} & \cellcolor{homThird}\homValue{\textit{80.00}}{0.15} & \homValue{79.30}{0.01} & \homValue{77.43}{0.08} & \homValue{75.26}{0.00} & \homValue{66.74}{0.00} & \homValue{68.27}{0.05} & \cellcolor{homBest}\homValue{\textbf{80.78}}{0.42} \\
\rowcolor{homBand}[3pt][3pt] \multicolumn{1}{l}{\cellcolor{white}} & 0.5 & \multicolumn{1}{c}{\cellcolor{white}} & \multicolumn{1}{c}{\cellcolor{white}} & \multicolumn{1}{c}{\cellcolor{white}} & \multicolumn{1}{c}{\cellcolor{white}\homValue{80.76}{0.00}} & \homValue{80.66}{0.09} & \homValue{80.12}{0.03} & \homValue{65.47}{0.17} & \cellcolor{homBest}\homValue{\textbf{82.10}}{0.22} & \homValue{77.19}{0.15} & \homValue{76.55}{0.08} & \homValue{76.49}{0.10} & \homValue{80.68}{0.07} & \cellcolor{homSecond}\homValue{\underline{81.89}}{0.02} & \cellcolor{homThird}\homValue{\textit{80.70}}{0.11} & \homValue{76.90}{0.23} & \homValue{77.05}{0.00} & \homValue{72.57}{0.00} & \homValue{70.97}{0.17} & \homValue{80.51}{0.00} \\
\midrule[0.25pt]
\rowcolor{white}[3pt][3pt] \multicolumn{1}{l}{\cellcolor{white}} & 0.1 & \multicolumn{1}{c}{\cellcolor{white}} & \multicolumn{1}{c}{\cellcolor{white}} & \multicolumn{1}{c}{\cellcolor{white}} & \multicolumn{1}{c}{\cellcolor{white}\homValue{71.39}{0.12}} & \homValue{61.64}{0.04} & \homValue{61.87}{0.05} & \homValue{58.47}{0.26} & \cellcolor{homThird}\homValue{\textit{62.41}}{0.18} & \homValue{61.59}{0.07} & \homValue{60.90}{0.02} & \homValue{60.93}{0.08} & \cellcolor{homBest}\homValue{\textbf{64.56}}{0.08} & \homValue{62.13}{0.19} & \homValue{61.80}{0.02} & \homValue{54.85}{0.08} & \homValue{55.53}{0.00} & \homValue{55.94}{0.00} & \homValue{56.25}{0.15} & \cellcolor{homSecond}\homValue{\underline{64.34}}{0.07} \\
\rowcolor{homBand}[3pt][3pt] \multicolumn{1}{l}{\cellcolor{white}\textbf{ogbn-arxiv}} & 0.3 & \multicolumn{1}{c}{\cellcolor{white}\homValue{55.60}{0.16}} & \multicolumn{1}{c}{\cellcolor{white}\homValue{67.27}{0.12}} & \multicolumn{1}{c}{\cellcolor{white}\homValue{72.83}{0.20}} & \homValue{73.29}{0.30} & \homValue{67.77}{0.18} & \homValue{68.34}{0.17} & \homValue{62.59}{0.07} & \homValue{68.75}{0.21} & \homValue{67.36}{0.37} & \homValue{65.82}{0.05} & \homValue{65.86}{0.06} & \cellcolor{homSecond}\homValue{\underline{70.69}}{0.08} & \cellcolor{homThird}\homValue{\textit{69.88}}{0.15} & \homValue{67.51}{0.09} & \homValue{55.19}{0.13} & \homValue{56.37}{0.00} & \homValue{55.93}{0.00} & \homValue{56.42}{0.16} & \cellcolor{homBest}\homValue{\textbf{71.41}}{0.26} \\
\rowcolor{white}[3pt][3pt] \multicolumn{1}{l}{\cellcolor{white}} & 0.5 & \multicolumn{1}{c}{\cellcolor{white}} & \multicolumn{1}{c}{\cellcolor{white}} & \multicolumn{1}{c}{\cellcolor{white}} & \multicolumn{1}{c}{\cellcolor{white}\homValue{72.82}{0.01}} & \homValue{70.55}{0.42} & \homValue{69.74}{0.37} & \homValue{63.14}{0.12} & \homValue{71.10}{0.10} & \homValue{70.64}{0.14} & \homValue{67.68}{0.02} & \homValue{67.64}{0.05} & \cellcolor{homThird}\homValue{\textit{71.67}}{0.18} & \cellcolor{homSecond}\homValue{\underline{71.93}}{0.02} & \homValue{70.25}{0.20} & \homValue{55.95}{0.11} & \homValue{57.01}{0.00} & \homValue{56.62}{0.00} & \homValue{56.65}{0.15} & \cellcolor{homBest}\homValue{\textbf{72.50}}{0.05} \\
\midrule[0.25pt]
\rowcolor{homBand}[3pt][3pt] \multicolumn{1}{l}{\cellcolor{white}} & 0.1 & \multicolumn{1}{c}{\cellcolor{white}} & \multicolumn{1}{c}{\cellcolor{white}} & \multicolumn{1}{c}{\cellcolor{white}} & \multicolumn{1}{c}{\cellcolor{white}\homValue{75.98}{0.62}} & \homValue{76.50}{0.72} & \homValue{69.09}{0.16} & \homValue{55.64}{0.91} & \cellcolor{homBest}\homValue{\textbf{78.20}}{0.14} & \homValue{64.55}{0.28} & \homValue{64.55}{0.21} & \homValue{64.43}{0.20} & \homValue{76.38}{0.48} & \cellcolor{homThird}\homValue{\textit{77.23}}{0.35} & \cellcolor{homSecond}\homValue{\underline{78.01}}{0.50} & \homValue{53.43}{0.01} & \homValue{66.46}{0.00} & \homValue{63.58}{0.00} & \homValue{71.12}{0.08} & \homValue{76.46}{0.45} \\
\rowcolor{white}[3pt][3pt] \multicolumn{1}{l}{\cellcolor{white}\textbf{ogbn-proteins}} & 0.3 & \multicolumn{1}{c}{\cellcolor{white}\homValue{72.91}{0.17}} & \multicolumn{1}{c}{\cellcolor{white}\homValue{74.13}{0.16}} & \multicolumn{1}{c}{\cellcolor{white}\homValue{77.48}{0.06}} & \homValue{76.56}{0.22} & \homValue{76.83}{0.08} & \homValue{71.10}{0.00} & \homValue{66.01}{0.78} & \cellcolor{homBest}\homValue{\textbf{78.25}}{0.32} & \homValue{70.62}{0.42} & \homValue{70.80}{0.39} & \homValue{70.39}{0.56} & \cellcolor{homSecond}\homValue{\underline{77.95}}{0.11} & \homValue{77.09}{0.66} & \cellcolor{homThird}\homValue{\textit{77.43}}{0.12} & \homValue{51.94}{1.16} & \homValue{67.15}{0.00} & \homValue{52.31}{0.00} & \homValue{72.72}{1.75} & \homValue{76.43}{0.22} \\
\rowcolor{homBand}[3pt][3pt] \multicolumn{1}{l}{\cellcolor{white}} & 0.5 & \multicolumn{1}{c}{\cellcolor{white}} & \multicolumn{1}{c}{\cellcolor{white}} & \multicolumn{1}{c}{\cellcolor{white}} & \multicolumn{1}{c}{\cellcolor{white}\homValue{76.42}{0.47}} & \cellcolor{homThird}\homValue{\textit{77.17}}{0.92} & \homValue{75.21}{0.63} & \homValue{68.20}{1.01} & \cellcolor{homBest}\homValue{\textbf{77.54}}{0.58} & \homValue{74.62}{0.05} & \homValue{74.87}{0.18} & \homValue{74.30}{0.38} & \homValue{76.75}{0.24} & \homValue{77.08}{0.02} & \cellcolor{homSecond}\homValue{\underline{77.22}}{0.13} & \homValue{54.99}{1.48} & \homValue{64.73}{0.00} & \homValue{59.33}{0.00} & \homValue{70.79}{0.05} & \homValue{76.45}{0.63} \\
\midrule[0.25pt]
\rowcolor{white}[3pt][3pt] \multicolumn{1}{l}{\cellcolor{white}} & 0.1 & \multicolumn{1}{c}{\cellcolor{white}} & \multicolumn{1}{c}{\cellcolor{white}} & \multicolumn{1}{c}{\cellcolor{white}} & \multicolumn{1}{c}{\cellcolor{white}\homValue{81.05}{0.14}} & \homValue{68.86}{0.01} & \cellcolor{homBest}\homValue{\textbf{75.48}}{0.05} & \homValue{64.09}{0.05} & \cellcolor{homSecond}\homValue{\underline{72.07}}{0.04} & \homValue{65.08}{0.00} & \homValue{64.82}{0.01} & \homValue{64.81}{0.00} & \homValue{70.74}{0.01} & \homValue{70.95}{0.03} & \homValue{68.89}{0.02} & \homValue{60.46}{0.04} & \homValue{65.50}{0.00} & \homValue{62.99}{0.00} & \homValue{64.14}{0.01} & \cellcolor{homThird}\homValue{\textit{71.92}}{0.11} \\
\rowcolor{homBand}[3pt][3pt] \multicolumn{1}{l}{\cellcolor{white}\textbf{Pokec}} & 0.3 & \multicolumn{1}{c}{\cellcolor{white}\homValue{62.31}{0.02}} & \multicolumn{1}{c}{\cellcolor{white}\homValue{69.36}{0.01}} & \multicolumn{1}{c}{\cellcolor{white}\homValue{86.24}{0.04}} & \homValue{84.45}{0.19} & \homValue{76.33}{0.05} & \cellcolor{homSecond}\homValue{\underline{80.07}}{0.09} & \homValue{67.29}{0.10} & \homValue{79.43}{0.10} & \homValue{68.56}{0.02} & \homValue{67.42}{0.02} & \homValue{67.35}{0.01} & \homValue{76.56}{0.06} & \cellcolor{homThird}\homValue{\textit{79.65}}{0.07} & \homValue{76.38}{0.07} & \homValue{69.39}{0.03} & \homValue{68.23}{0.00} & \homValue{64.42}{0.00} & \homValue{66.94}{0.07} & \cellcolor{homBest}\homValue{\textbf{81.19}}{0.08} \\
\rowcolor{white}[3pt][3pt] \multicolumn{1}{l}{\cellcolor{white}} & 0.5 & \multicolumn{1}{c}{\cellcolor{white}} & \multicolumn{1}{c}{\cellcolor{white}} & \multicolumn{1}{c}{\cellcolor{white}} & \multicolumn{1}{c}{\cellcolor{white}\homValue{85.49}{0.14}} & \homValue{80.61}{0.02} & \homValue{82.26}{0.08} & \homValue{70.64}{0.08} & \cellcolor{homThird}\homValue{\textit{82.31}}{0.03} & \homValue{74.64}{0.03} & \homValue{69.71}{0.05} & \homValue{69.70}{0.00} & \homValue{79.97}{0.01} & \cellcolor{homSecond}\homValue{\underline{82.46}}{0.01} & \homValue{80.70}{0.09} & \homValue{73.69}{0.04} & \homValue{69.87}{0.00} & \homValue{66.62}{0.00} & \homValue{69.47}{0.03} & \cellcolor{homBest}\homValue{\textbf{84.19}}{0.06} \\
\bottomrule
\end{tabular*}
\par\vspace{4pt}
{\fontsize{7pt}{8pt}\selectfont\raggedright  Each cell reports the mean above $\pm$ standard deviation. \textbf{Bold} (teal), \underline{underline} (gold), and \textit{italic} (lavender) identify the highest, second-highest, and third-highest distinct means among ranked methods within each row, respectively; ties share a rank. Reference methods are excluded from ranking. MLP, Spanning Tree, and Full GCN are shared across retention ratios. EDiS-Full is reported as a full-graph inference reference at each training budget, while EDiS-Lite uses sampled-graph evaluation.\par}
\end{minipage}
\end{adjustbox}
\end{sidewaystable}
\endgroup

\clearpage
\section{Runtime and Memory Details}
\label{app:runtime-memory}

This appendix quantifies what \edis{}'s sparsity costs and saves in
practice. It first covers the inference-time trade-off between
\edis{}-Lite and \edis{}-Full (Figure~\ref{fig:evaluation-cost}), then
a full end-to-end runtime and memory breakdown across
retention budgets
(Tables~\ref{tab:profiling-rho01}--\ref{tab:profiling-rho05}), and
finally a comparison against external baselines under matched hardware
(Tables~\ref{tab:runtime-sparsification} and~\ref{tab:runtime-end-to-end}).

\subsection{Lite versus Full Evaluation Trade-off}
\label{app:lite-full-tradeoff}

Switching to full-graph evaluation mainly pays off on comparatively
sparse graphs, where \edis{}-Lite drops enough structure at tight
budgets to cost real accuracy, and the full graph restores it. On
already dense graphs, a sparse subgraph retains enough structure to
stay competitive on its own, so evaluating on the full graph adds
little or no accuracy while still paying a real evaluation-time cost,
sometimes several times higher. The Lite/Full choice is therefore
something to weigh per graph, based on how sparse or dense it already
is, rather than a switch that pays off everywhere.

\begin{figure}[H]
\centering
\includegraphics[width=\linewidth]{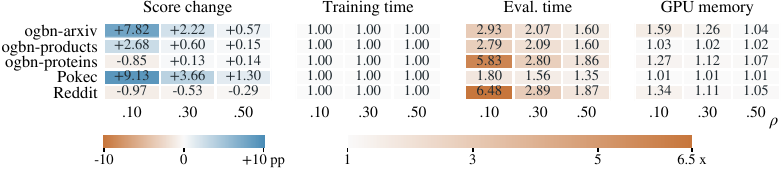}
\caption{\textbf{Full-graph evaluation tradeoffs.} Cells show Full minus Lite score (percentage points) or Full/Lite cost ratios at retention $\rho$ (MaxCF, cosine). The cost panels share a color scale.}
\label{fig:evaluation-cost}
\end{figure}
\vspace{-10pt}

Each variant is divided into construction
(support generation, graph preparation, and composition), training,
and evaluation, together with peak GPU and system memory.
Construction and training are nearly identical between variants,
since both share the same decomposition and training graphs. The gap
lives entirely in evaluation, where Lite reuses the already-composed
sparse graph while Full reprocesses the complete graph, shrinking as
$\rho$ grows until the two variants match at $\rho=1$.
Tables~\ref{tab:profiling-rho01}--\ref{tab:profiling-rho05} give the
full numbers behind this trade-off, in seconds and GiB, measured on a
single H200 80GB GPU. This internal comparison and the cross-method
comparison against baselines (Section~\ref{app:end-to-end-baselines})
were run on different GPU generations, so absolute times should be
compared only within each table, not across the two.

\begingroup
\setlength{\floatsep}{4pt plus 1pt minus 1pt}
\setlength{\textfloatsep}{4pt plus 1pt minus 1pt}
\setlength{\intextsep}{4pt plus 1pt minus 1pt}
\setlength{\abovecaptionskip}{2pt}
\setlength{\belowcaptionskip}{2pt}

\begin{table}[H]
\centering
\tiny
\setlength{\tabcolsep}{2pt}
\renewcommand{\arraystretch}{0.9}
\rowcolors{2}{white}{black!3}
\begin{tabular*}{\linewidth}{@{\extracolsep{\fill}}l rrrrrr rrrrrr@{}}
\toprule
& \multicolumn{6}{c}{\textbf{\edis-Lite}} & \multicolumn{6}{c}{\textbf{\edis-Full}} \\
\cmidrule(lr){2-7}\cmidrule(lr){8-13}
\textbf{Dataset} & Constr. & Train & Eval & Total & GPU & RAM & Constr. & Train & Eval & Total & GPU & RAM \\
\midrule
Reddit & 1400.2 & 126.1 & 49.9 & 1576.2 & 7.05 & 12.21 & 1365.5 & 126.0 & 323.6 & 1815.1 & 9.43 & 14.88 \\
ogbn-products & 291.4 & 633.1 & 218.2 & 1142.7 & 51.51 & 13.53 & 312.0 & 633.3 & 608.3 & 1553.6 & 53.17 & 13.83 \\
ogbn-arxiv & 16.2 & 599.2 & 230.6 & 846.0 & 6.79 & 1.79 & 16.5 & 598.9 & 675.9 & 1291.3 & 10.80 & 1.79 \\
ogbn-proteins & 355.1 & 26.5 & 10.9 & 392.5 & 4.98 & 8.36 & 356.1 & 26.5 & 63.6 & 446.2 & 6.32 & 8.36 \\
Pokec & 332.2 & 3451.2 & 1427.6 & 5211.0 & 49.90 & 5.02 & 334.5 & 3451.8 & 2574.1 & 6360.4 & 50.50 & 5.05 \\
\bottomrule
\end{tabular*}
\caption{Runtime and memory at $\rho=0.1$.}
\label{tab:profiling-rho01}
\end{table}

\begin{table}[H]
\centering
\tiny
\setlength{\tabcolsep}{2pt}
\renewcommand{\arraystretch}{0.9}
\rowcolors{2}{white}{black!3}
\begin{tabular*}{\linewidth}{@{\extracolsep{\fill}}l rrrrrr rrrrrr@{}}
\toprule
& \multicolumn{6}{c}{\textbf{\edis-Lite}} & \multicolumn{6}{c}{\textbf{\edis-Full}} \\
\cmidrule(lr){2-7}\cmidrule(lr){8-13}
\textbf{Dataset} & Constr. & Train & Eval & Total & GPU & RAM & Constr. & Train & Eval & Total & GPU & RAM \\
\midrule
Reddit & 1547.0 & 257.5 & 112.0 & 1916.5 & 11.40 & 14.90 & 1447.4 & 257.4 & 323.4 & 2028.2 & 12.60 & 14.88 \\
ogbn-products & 525.8 & 791.4 & 290.7 & 1607.9 & 55.47 & 13.91 & 432.1 & 791.1 & 607.7 & 1830.9 & 56.76 & 13.88 \\
ogbn-arxiv & 26.9 & 793.3 & 326.4 & 1146.6 & 8.62 & 1.95 & 24.6 & 793.3 & 676.1 & 1494.0 & 10.84 & 1.93 \\
ogbn-proteins & 384.6 & 51.1 & 22.7 & 458.4 & 6.68 & 10.44 & 363.2 & 51.0 & 63.5 & 477.7 & 7.50 & 10.76 \\
Pokec & 745.7 & 3926.3 & 1648.1 & 6320.1 & 51.79 & 5.47 & 918.4 & 3925.6 & 2573.2 & 7417.2 & 52.26 & 5.35 \\
\bottomrule
\end{tabular*}
\caption{Runtime and memory at $\rho=0.3$.}
\label{tab:profiling-rho03}
\end{table}

\begin{table}[H]
\centering
\tiny
\setlength{\tabcolsep}{2pt}
\renewcommand{\arraystretch}{0.9}
\rowcolors{2}{white}{black!3}
\begin{tabular*}{\linewidth}{@{\extracolsep{\fill}}l rrrrrr rrrrrr@{}}
\toprule
& \multicolumn{6}{c}{\textbf{\edis-Lite}} & \multicolumn{6}{c}{\textbf{\edis-Full}} \\
\cmidrule(lr){2-7}\cmidrule(lr){8-13}
\textbf{Dataset} & Constr. & Train & Eval & Total & GPU & RAM & Constr. & Train & Eval & Total & GPU & RAM \\
\midrule
Reddit & 1717.8 & 386.0 & 172.8 & 2276.6 & 15.95 & 12.40 & 1621.2 & 386.1 & 323.3 & 2330.6 & 16.79 & 14.93 \\
ogbn-products & 1422.7 & 984.3 & 380.7 & 2787.7 & 59.43 & 14.07 & 1402.5 & 984.2 & 607.8 & 2994.5 & 60.35 & 14.06 \\
ogbn-arxiv & 37.1 & 987.1 & 423.4 & 1447.6 & 10.44 & 2.05 & 37.3 & 987.0 & 676.0 & 1700.3 & 10.82 & 2.05 \\
ogbn-proteins & 449.4 & 75.2 & 34.3 & 558.9 & 8.46 & 8.79 & 444.5 & 75.2 & 63.7 & 583.4 & 9.05 & 8.79 \\
Pokec & 1105.3 & 4467.3 & 1899.5 & 7472.1 & 53.69 & 6.13 & 1206.8 & 4467.1 & 2573.8 & 8247.7 & 54.02 & 6.21 \\
\bottomrule
\end{tabular*}
\caption{Runtime and memory at $\rho=0.5$.}
\label{tab:profiling-rho05}
\end{table}
\endgroup

\subsection{End-to-End Runtime Against Baselines}
\label{app:end-to-end-baselines}

All methods in Tables~\ref{tab:runtime-sparsification} and~\ref{tab:runtime-end-to-end}, \edis{}-Lite and every baseline, are measured on the same single H100 80GB GPU, so the comparison is not confounded by hardware.

\paragraph{Sparsification time.}
Table~\ref{tab:runtime-sparsification} compares standalone sparsification time
at the same retention ratio, $\rho=0.3$, across five large-scale graphs.
\edis{}-Lite's reported times use each dataset's own tuned round cap
$K_{\max}$ (Appendix~\ref{app:hyperparameters}, $10$--$80$ across these
five graphs), since a deeper cap means more extraction rounds and
shifts cost from composition toward precomputation
(Appendix~\ref{app:kmax-sensitivity}), so these numbers reflect each
dataset's own cap rather than a single fixed setting.
\edis{}-Lite's one-time extraction costs more than simple degree- or
neighborhood-based heuristics (SCAN, L-Spar, G-Spar, L-Sim, Local degree),
is comparable to or below Spectral, and is far below Forest Fire. This
cost is paid once and reused across epochs and retention budgets
(Appendix~\ref{app:coverage-reuse}), so it should be weighed against the
training savings it buys, not against sparsification cost in isolation.

\begin{table}[!htbp]
\centering
\caption{Sparsification time (s) at $\rho=0.3$ on five large-scale graphs.}
\label{tab:runtime-sparsification}

\scriptsize
\setlength{\tabcolsep}{4pt}
\renewcommand{\arraystretch}{1.05}
\rowcolors{2}{white}{black!3}

\begin{tabular*}{\linewidth}{@{}l@{\extracolsep{\fill}}rrrrr@{}}
\toprule
\textbf{Method}
& \textbf{Reddit}
& \textbf{\shortstack{ogbn-\\products}}
& \textbf{\shortstack{ogbn-\\arxiv}}
& \textbf{\shortstack{ogbn-\\proteins}}
& \textbf{Pokec} \\
\midrule

Random
& \textbf{1.23}
& \textbf{1.37}
& \textbf{0.015}
& \textbf{0.819}
& \textbf{0.582} \\

Local degree
& 41.20
& 30.39
& 0.543
& 24.01
& 10.02 \\

Rank degree
& 249.83
& 69.20
& 4.05
& 169.86
& 27.50 \\

Forest Fire
& 2501.61
& ---
& 10.33
& 1414.05
& 692.14 \\

SCAN
& 78.66
& 47.57
& 0.649
& 40.52
& 13.32 \\

L-Spar
& 80.46
& 49.19
& 0.594
& 40.83
& 11.48 \\

G-Spar
& 78.30
& 47.63
& 0.654
& 40.20
& 13.43 \\

L-Sim
& 82.70
& 54.08
& 0.664
& 42.61
& 13.04 \\

Spectral$^{\ddagger}$
& 387.50
& 379.95
& 5.80
& 286.24
& 108.09 \\

Spanning tree$^{\dagger}$
& \underline{36.21}
& \underline{17.25}
& \underline{0.225}
& \underline{2.10}
& \underline{3.32} \\

\midrule

\textbf{\edis{}-Lite}
& 313.73
& 179.28
& 3.64
& 98.06
& 86.01 \\

\bottomrule
\end{tabular*}

\par\vspace{2pt}
\begin{minipage}{\linewidth}\scriptsize
Best and second-best reported times are bold and underlined, respectively.
$\dagger$: spanning-tree reference with unmatched retention;
---: unavailable.

$\ddagger$: approximate effective resistance; using the Spielman--Srivastava JL method
\citep{Spielman:2008effres}, avoiding explicit Laplacian pseudoinversion.
We use $d=\lceil4\ln n\rceil$ Gaussian projections with approximate-Cholesky
solves~\citep{Gao:2023laplacian} ($10^{-2}$ tolerance) on the smaller graph,
64 Rademacher projections with PCG on \textsc{ogbn-products},
\textsc{ogbn-proteins}, and \textsc{Pokec}, and
TGT+~\citep{Zhang:2023spanning} with 128 eigenpairs on \textsc{Reddit}.
\end{minipage}
\end{table}

\paragraph{End-to-end runtime.}
Table~\ref{tab:runtime-end-to-end} compares end-to-end runtime at
$\rho=0.3$ over 300 epochs, splitting $E_{300}$ into the one-time cost
$S$ and training cost $T_{300}$. Per epoch, \edis{}-Lite trains faster
than Full-GCN on all three datasets, \textsc{ogbn-arxiv} ($0.12$ vs.\
$0.32$\,s), \textsc{ogbn-products} ($1.35$ vs.\ $3.85$\,s), and
\textsc{Reddit} ($1.73$ vs.\ $2.38$\,s).
These timings are a separate measurement from the accuracies in
Table~\ref{tab:main-large-acc}: timing uses a fixed 300-epoch budget,
whereas accuracy follows each dataset's training protocol
(Table~\ref{tab:presets}: $2000$ epochs for \textsc{ogbn-arxiv}, $300$ for
\textsc{ogbn-products} and \textsc{Reddit}). For \textsc{ogbn-arxiv}, the two
therefore refer to different training durations, and we make no joint
speed--accuracy claim. In Table~\ref{tab:main-large-acc}, \edis{}-Lite
matches or exceeds Full-GCN's accuracy on \textsc{ogbn-products}
($80.78\%$ vs.\ $80.67\%$) and \textsc{Reddit} ($95.61\%$ vs.\ $95.44\%$)
but not on \textsc{ogbn-arxiv} ($71.41\%$ vs.\ $72.83\%$).
On runtime alone, \edis{}-Lite beats Full-GCN's end-to-end runtime within 300 epochs on
\textsc{ogbn-arxiv} ($39.5$\,s vs.\ $96.7$\,s) and \textsc{ogbn-products}
($583.8$\,s vs.\ $1154.0$\,s). On \textsc{Reddit}, the largest one-time
cost ($313.7$\,s) is not yet amortized at 300 epochs, so the total
($832.4$\,s) still exceeds Full-GCN's ($713.4$\,s). It crosses below
after roughly 480 epochs, sooner with reuse across budgets or runs. DSpar, a lighter one-shot spectral
sparsifier, beats \edis{}-Lite on \textsc{ogbn-products} and
\textsc{Reddit} but not \textsc{ogbn-arxiv} ($70.5$\,s vs.\ $39.5$\,s).
Its value there is accuracy and multi-budget reuse, not raw speed. MoG, AdaGLT, and
Unified-LTH remain far slower across all three datasets. The marked
AdaGLT and Unified-LTH configurations exceed the 3600\,s run limit and
are projected.
\enlargethispage{2\baselineskip}

\begin{table}[!htbp]
\centering
\caption{End-to-end runtime (s) at $\rho=0.3$ over 300 epochs.
$S$: sparsification; $T_{300}$: training; $E_{300}=S+T_{300}$.}
\label{tab:runtime-end-to-end}

\scriptsize
\setlength{\tabcolsep}{2pt}
\renewcommand{\arraystretch}{0.92}
\rowcolors{2}{white}{black!3}

\begin{tabular*}{\linewidth}{@{}l@{\extracolsep{\fill}}rrrrrrrrr@{}}
\toprule
\textbf{Method}
& \multicolumn{3}{c}{\textbf{Reddit}}
& \multicolumn{3}{c}{\textbf{ogbn-products}}
& \multicolumn{3}{c}{\textbf{ogbn-arxiv}} \\
\cmidrule(lr){2-4}
\cmidrule(lr){5-7}
\cmidrule(lr){8-10}
& $S$ & $T_{300}$ & $E_{300}$
& $S$ & $T_{300}$ & $E_{300}$
& $S$ & $T_{300}$ & $E_{300}$ \\
\midrule
Full-GCN
& 0.0 & 713.4 & \underline{713.4}
& 0.0 & 1154.0 & 1154.0
& 0.0 & 96.7 & 96.7 \\
\textbf{\edis{}-Lite}
& 313.7 & 518.7 & 832.4
& 179.3 & 404.5 & \underline{583.8}
& 3.6 & 35.9 & \textbf{39.5} \\
DSpar
& 4.0 & 224.4 & \textbf{228.4}
& 4.1 & 491.8 & \textbf{496.0}
& 0.1 & 70.4 & \underline{70.5} \\
\midrule
MoG
& incl. & 1432.3 & 1432.3
& incl. & 2536.4 & 2536.4
& incl. & 154.1 & 154.1 \\

AdaGLT
& incl. & 8819.6 & 8819.6$^{\dagger}$
& incl. & 15213.3 & 15213.3$^{\dagger}$
& incl. & 916.3 & 916.3 \\

Unified-LTH
& 5.4 & 7144.1 & 7149.5$^{\dagger}$
& 6.7 & 13161.1 & 13167.8$^{\dagger}$
& 0.1 & 395.9 & 396.0 \\

\bottomrule
\end{tabular*}

\par\vspace{2pt}
\begin{minipage}{\linewidth}\scriptsize
Best and second-best $E_{300}$ per dataset are bold and underlined, respectively.
$^{\dagger}$ marks AdaGLT/Unified-LTH totals exceeding the 3600\,s run limit and are projected from the epochs they ran;
``incl.'' means graph-selection cost is included in training.
\end{minipage}
\end{table}

\clearpage
\section{Ablations and Additional Profiles}
\label{app:additional-profiles}

\setlength{\intextsep}{4pt}
\setlength{\textfloatsep}{6pt plus 1pt minus 1pt}
\setlength{\floatsep}{6pt plus 1pt minus 1pt}

Sections~\ref{app:selector-profiles}--\ref{app:sage-robustness} sweep
the selector, score, mechanism, inference-graph, and backbone axes on
the 14 homophilic/heterophilic datasets. The five large-scale graphs
are profiled separately, for extraction depth
(\ref{app:kmax-sensitivity}) and decomposition reuse
(\ref{app:coverage-reuse}).

\subsection{Extended selector ablation}
\label{app:selector-profiles}

The \edis{} framework does not depend on which selector is used. Any of the seven arms can substitute for another with no other changes to the pipeline, and which selector performs best depends on the dataset rather than being fixed. So MaxCF is a strong default, not the choice that always wins.

Table~\ref{tab:selector-main} extends the main-text selector comparison to the 14 datasets. We fix the GCN backbone, cosine scores, and 30\% retention. Hybrid combines MaxCF, then $k$-NN, then degree-normalized extraction, using the $K=10$ round limit from Appendix~\ref{app:implementation}. Figure~\ref{fig:selector-trends-full} extends the comparison across the full measured retention range, testing whether MaxCF's advantage as the default selector holds throughout. Selectors diverge most at low retention, where the choice of edges to keep is most consequential, and converge as $\rho\to1$, where every selector necessarily retains the full edge set.

\begin{table}[H]
\centering
\footnotesize
\setlength{\tabcolsep}{2pt}
\renewcommand{\arraystretch}{1.02}
\begin{tabular*}{\linewidth}{@{}l@{\extracolsep{\fill}}ccccccc@{}}
\toprule
Dataset & \shortstack{MaxCF\\(default)} & MinCF & Degree & $k$-NN & Spanner & Random & Hybrid \\
\midrule
Cora & 77.16{\scriptsize$\pm$0.98} & 77.92{\scriptsize$\pm$0.71} & \underline{77.94}{\scriptsize$\pm$0.72} & 76.04{\scriptsize$\pm$0.45} & \textbf{78.02}{\scriptsize$\pm$0.80} & 76.20{\scriptsize$\pm$0.57} & 77.24{\scriptsize$\pm$1.00} \\
\rowcolor[gray]{0.95}
CiteSeer & 66.08{\scriptsize$\pm$1.23} & 67.26{\scriptsize$\pm$1.05} & \underline{67.38}{\scriptsize$\pm$0.88} & 65.34{\scriptsize$\pm$1.00} & \textbf{67.94}{\scriptsize$\pm$0.95} & 67.06{\scriptsize$\pm$0.91} & 66.14{\scriptsize$\pm$1.03} \\
PubMed & \underline{77.28}{\scriptsize$\pm$0.57} & 76.52{\scriptsize$\pm$1.99} & 76.76{\scriptsize$\pm$0.84} & 74.54{\scriptsize$\pm$0.96} & 76.26{\scriptsize$\pm$0.29} & 75.16{\scriptsize$\pm$1.15} & \textbf{77.76}{\scriptsize$\pm$0.41} \\
\rowcolor[gray]{0.95}
Amazon-Computer & \textbf{93.35}{\scriptsize$\pm$0.21} & 92.98{\scriptsize$\pm$0.11} & \underline{93.06}{\scriptsize$\pm$0.18} & 91.54{\scriptsize$\pm$0.33} & 91.76{\scriptsize$\pm$0.53} & 91.57{\scriptsize$\pm$0.22} & 92.07{\scriptsize$\pm$0.22} \\
Amazon-Photo & \textbf{95.60}{\scriptsize$\pm$0.15} & \underline{95.56}{\scriptsize$\pm$0.14} & 95.51{\scriptsize$\pm$0.27} & 94.18{\scriptsize$\pm$0.33} & 95.39{\scriptsize$\pm$0.39} & 95.14{\scriptsize$\pm$0.21} & 95.46{\scriptsize$\pm$0.16} \\
\rowcolor[gray]{0.95}
Coauthor-CS & \textbf{95.72}{\scriptsize$\pm$0.06} & 95.06{\scriptsize$\pm$0.09} & 95.23{\scriptsize$\pm$0.14} & 95.39{\scriptsize$\pm$0.28} & 95.14{\scriptsize$\pm$0.22} & 95.42{\scriptsize$\pm$0.19} & \underline{95.44}{\scriptsize$\pm$0.09} \\
Coauthor-Physics & \underline{97.06}{\scriptsize$\pm$0.09} & 96.84{\scriptsize$\pm$0.06} & 97.03{\scriptsize$\pm$0.09} & 97.00{\scriptsize$\pm$0.08} & 96.70{\scriptsize$\pm$0.08} & 96.66{\scriptsize$\pm$0.08} & \textbf{97.11}{\scriptsize$\pm$0.10} \\
\rowcolor[gray]{0.95}
WikiCS & \textbf{80.25}{\scriptsize$\pm$0.79} & 77.85{\scriptsize$\pm$0.31} & \underline{79.61}{\scriptsize$\pm$0.45} & 78.72{\scriptsize$\pm$0.36} & 77.80{\scriptsize$\pm$0.35} & 77.34{\scriptsize$\pm$0.37} & 78.44{\scriptsize$\pm$0.50} \\
Squirrel & 42.63{\scriptsize$\pm$2.04} & 42.33{\scriptsize$\pm$1.27} & \underline{44.05}{\scriptsize$\pm$1.84} & 40.59{\scriptsize$\pm$2.42} & \textbf{44.39}{\scriptsize$\pm$1.35} & 41.22{\scriptsize$\pm$1.96} & 43.04{\scriptsize$\pm$1.89} \\
\rowcolor[gray]{0.95}
Chameleon & \textbf{43.75}{\scriptsize$\pm$4.14} & 40.05{\scriptsize$\pm$3.33} & 40.73{\scriptsize$\pm$3.14} & 39.25{\scriptsize$\pm$3.35} & \underline{42.79}{\scriptsize$\pm$2.34} & 38.92{\scriptsize$\pm$4.22} & 41.86{\scriptsize$\pm$2.29} \\
Roman-Empire & \underline{68.95}{\scriptsize$\pm$0.41} & \textbf{69.06}{\scriptsize$\pm$0.58} & 68.52{\scriptsize$\pm$0.46} & 68.84{\scriptsize$\pm$0.22} & 68.91{\scriptsize$\pm$0.29} & 68.37{\scriptsize$\pm$0.73} & 68.87{\scriptsize$\pm$0.59} \\
\rowcolor[gray]{0.95}
Amazon-Ratings & 52.31{\scriptsize$\pm$0.85} & \underline{52.72}{\scriptsize$\pm$0.44} & 52.34{\scriptsize$\pm$0.52} & 51.87{\scriptsize$\pm$0.42} & 51.99{\scriptsize$\pm$0.30} & 51.94{\scriptsize$\pm$0.51} & \textbf{52.76}{\scriptsize$\pm$0.31} \\
Minesweeper$^\dagger$ & 73.82{\scriptsize$\pm$0.97} & \textbf{77.84}{\scriptsize$\pm$0.82} & \underline{75.69}{\scriptsize$\pm$0.34} & 74.25{\scriptsize$\pm$1.25} & 75.01{\scriptsize$\pm$1.05} & 75.30{\scriptsize$\pm$1.51} & 73.74{\scriptsize$\pm$1.68} \\
\rowcolor[gray]{0.95}
Questions$^\dagger$ & \underline{74.12}{\scriptsize$\pm$0.88} & \textbf{74.73}{\scriptsize$\pm$1.04} & 73.46{\scriptsize$\pm$1.41} & 72.91{\scriptsize$\pm$0.96} & 72.79{\scriptsize$\pm$0.69} & 71.39{\scriptsize$\pm$0.66} & 73.69{\scriptsize$\pm$0.72} \\
\bottomrule
\end{tabular*}
\caption{Selector ablation for \edis{}-Lite (GCN, cosine scores, $\rho=0.30$). Values are mean $\pm$ standard deviation (\%). $^\dagger$Reported as ROC-AUC rather than accuracy. Best and second-best reported means are bold and underlined, respectively.}
\label{tab:selector-main}
\end{table}

\begin{figure}[!htbp]
\centering
\includegraphics[width=\linewidth]{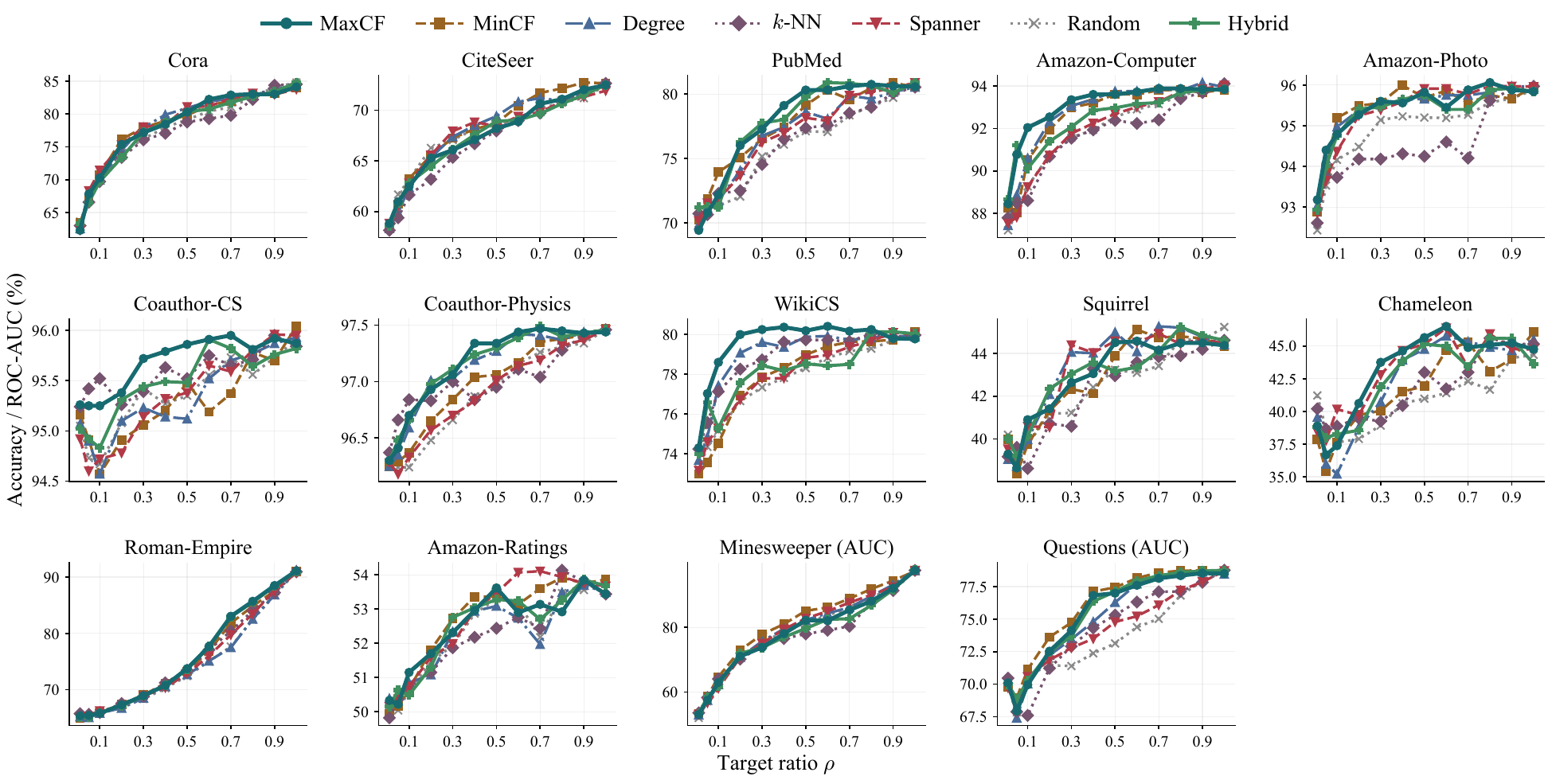}
\caption{\textbf{Selector ablation across the full retention range.} GCN backbone, cosine scores, all 14 datasets. MaxCF (default) and Hybrid are drawn with heavier lines.}
\label{fig:selector-trends-full}
\end{figure}

\clearpage
\subsection{Extended edge-score ablation}
\label{app:score-profiles}

The \edis{} framework does not depend on which edge score is used. Any of the five scores can substitute for another with no other changes to the pipeline, and which score performs best depends on the dataset rather than being fixed. So cosine is a strong default, not the choice that always wins.

Table~\ref{tab:score-main} extends the main-text edge-score comparison (Figure~\ref{fig:selector-gap}) to the 14 datasets. We fix MaxCF, the GCN backbone, and 30\% retention while varying only the edge score. Four of the
five (cosine, Euclidean, dot-product, and uniform) are fixed
formulas, while MLP instead learns its preference from a small
trained network over the two endpoint features. Figure~\ref{fig:score-trends-full} extends the comparison across the full measured retention range, testing whether cosine's advantage as the default score holds throughout. Scores diverge most at low retention, where edge weighting is most consequential, and converge as $\rho\to1$, where every score sees the same edge set.

\begin{table}[H]
\centering
\footnotesize
\setlength{\tabcolsep}{2pt}
\renewcommand{\arraystretch}{1.02}
\begin{tabular*}{\linewidth}{@{}l@{\extracolsep{\fill}}ccccc@{}}
\toprule
Dataset & \shortstack{Cosine\\(default)} & Euclidean & Dot & Uniform & MLP \\
\midrule
Cora & 77.16{\scriptsize$\pm$0.98} & 75.46{\scriptsize$\pm$2.16} & \underline{77.78}{\scriptsize$\pm$0.69} & \textbf{77.94}{\scriptsize$\pm$1.09} & 77.74{\scriptsize$\pm$1.20} \\
\rowcolor[gray]{0.95}
CiteSeer & 66.08{\scriptsize$\pm$1.23} & \textbf{66.66}{\scriptsize$\pm$1.04} & 66.18{\scriptsize$\pm$0.84} & 66.38{\scriptsize$\pm$0.54} & \underline{66.54}{\scriptsize$\pm$1.10} \\
PubMed & 77.28{\scriptsize$\pm$0.57} & \textbf{77.80}{\scriptsize$\pm$0.28} & \underline{77.30}{\scriptsize$\pm$0.81} & 76.82{\scriptsize$\pm$1.19} & -- \\
\rowcolor[gray]{0.95}
Amazon-Computer & \underline{93.35}{\scriptsize$\pm$0.21} & 92.58{\scriptsize$\pm$0.03} & \textbf{94.07}{\scriptsize$\pm$0.13} & 92.91{\scriptsize$\pm$0.08} & 93.06{\scriptsize$\pm$0.11} \\
Amazon-Photo & 95.60{\scriptsize$\pm$0.15} & 95.69{\scriptsize$\pm$0.46} & \textbf{96.01}{\scriptsize$\pm$0.18} & \underline{95.80}{\scriptsize$\pm$0.06} & 95.56{\scriptsize$\pm$0.44} \\
\rowcolor[gray]{0.95}
Coauthor-CS & \textbf{95.72}{\scriptsize$\pm$0.06} & \underline{95.71}{\scriptsize$\pm$0.05} & 95.51{\scriptsize$\pm$0.13} & 95.04{\scriptsize$\pm$0.18} & 95.28{\scriptsize$\pm$0.15} \\
Coauthor-Physics & 97.06{\scriptsize$\pm$0.09} & \underline{97.09}{\scriptsize$\pm$0.02} & \textbf{97.14}{\scriptsize$\pm$0.12} & \textbf{97.14}{\scriptsize$\pm$0.05} & 97.06{\scriptsize$\pm$0.06} \\
\rowcolor[gray]{0.95}
WikiCS & \textbf{80.25}{\scriptsize$\pm$0.79} & \underline{80.24}{\scriptsize$\pm$0.67} & 79.81{\scriptsize$\pm$0.29} & 79.01{\scriptsize$\pm$0.41} & 79.21{\scriptsize$\pm$0.39} \\
Squirrel & 42.63{\scriptsize$\pm$2.04} & \textbf{43.53}{\scriptsize$\pm$2.17} & 42.45{\scriptsize$\pm$2.98} & \underline{43.21}{\scriptsize$\pm$2.59} & 42.28{\scriptsize$\pm$1.18} \\
\rowcolor[gray]{0.95}
Chameleon & \textbf{43.75}{\scriptsize$\pm$4.14} & 40.65{\scriptsize$\pm$3.24} & \underline{43.38}{\scriptsize$\pm$3.52} & 41.31{\scriptsize$\pm$2.65} & 40.21{\scriptsize$\pm$2.85} \\
Roman-Empire & 68.95{\scriptsize$\pm$0.41} & \underline{69.18}{\scriptsize$\pm$0.20} & 68.80{\scriptsize$\pm$0.15} & \textbf{70.53}{\scriptsize$\pm$0.49} & 68.90{\scriptsize$\pm$0.37} \\
\rowcolor[gray]{0.95}
Amazon-Ratings & 52.31{\scriptsize$\pm$0.85} & 52.45{\scriptsize$\pm$0.88} & \textbf{53.38}{\scriptsize$\pm$0.50} & \underline{52.94}{\scriptsize$\pm$0.59} & 52.74{\scriptsize$\pm$0.21} \\
Minesweeper$^\dagger$ & 73.82{\scriptsize$\pm$0.97} & 74.26{\scriptsize$\pm$0.26} & 74.00{\scriptsize$\pm$0.80} & \underline{76.67}{\scriptsize$\pm$0.61} & \textbf{83.70}{\scriptsize$\pm$1.57} \\
\rowcolor[gray]{0.95}
Questions$^\dagger$ & \underline{74.12}{\scriptsize$\pm$0.88} & 72.72{\scriptsize$\pm$1.83} & \underline{74.12}{\scriptsize$\pm$0.17} & \textbf{74.15}{\scriptsize$\pm$0.08} & 73.38{\scriptsize$\pm$0.87} \\
\bottomrule
\end{tabular*}
\caption{Score-function ablation for \edis{}-Lite (GCN, MaxCF selector, $\rho=0.30$). Values are mean $\pm$ standard deviation (\%). $^\dagger$Reported as ROC-AUC rather than accuracy. Best and second-best reported means are bold and underlined, respectively. A dash marks PubMed--MLP, which is not applicable in this implementation.}
\label{tab:score-main}
\end{table}

\begin{keyquestion}
\textbf{Diagnostic.} Minesweeper is the outlier in Table~\ref{tab:score-main}. MLP reaches $83.70\pm1.57$ ROC-AUC, while the cosine default only reaches $73.82\pm0.97$. That is a 9.88-point gap, about six times the next-largest gap in the panel (Roman-Empire, 1.58 points). Figure~\ref{fig:score-trends-full} below shows this gap holds across most of the retention range, not just at $\rho=0.30$. So on Minesweeper, the edge score matters more than the selector.
\end{keyquestion}

\vspace{8pt}
\begin{figure}[!htbp]
\centering
\includegraphics[width=\linewidth]{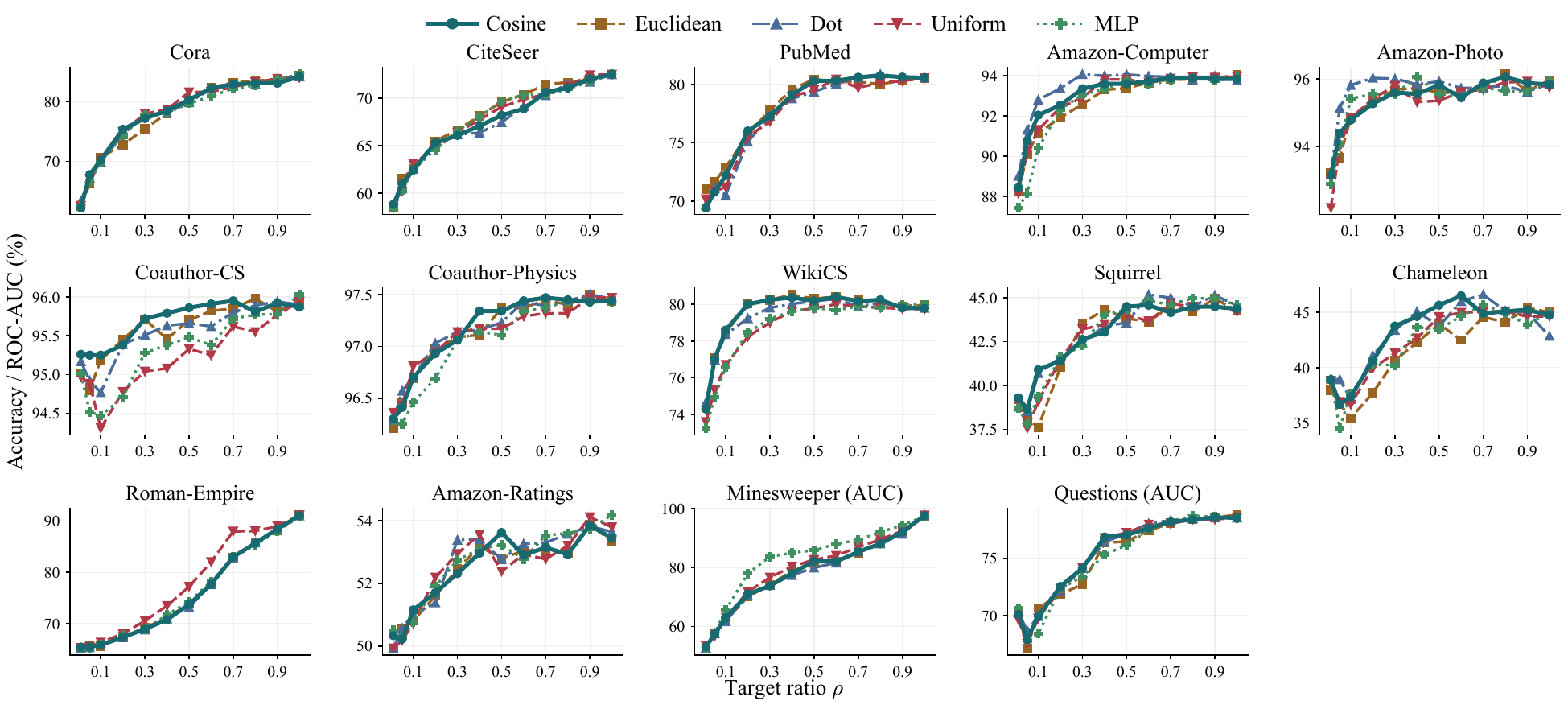}
\caption{\textbf{Score-function ablation across the full retention range.} GCN backbone, MaxCF selector, all 14 datasets. Cosine scores (default) are drawn with a heavier line.}
\label{fig:score-trends-full}
\end{figure}

\clearpage
\subsection{Extended mechanism ablation}
\label{app:mechanism-profile}

At 30\% retention, EDiS-Lite has a higher reported mean than Frozen on
13 of 14 datasets, Direct on 11, and Shuffled groups on 12. It leads
all three controls on nine datasets (Table~\ref{tab:mechanism-main}).
The largest control advantage in this table is 0.72 points for Shuffled
groups on Squirrel. These descriptive comparisons support the value of
decomposition and resampling across much of the panel, while showing
that neither is uniformly beneficial. The retention sweep further
examines where the Frozen gap is largest.

Table~\ref{tab:mechanism-main} extends the main-text mechanism comparison (Figure~\ref{fig:mechanism-profile}) to the 14 datasets. We fix the GCN backbone, MaxCF selector, and cosine scores at 30\% retention. Frozen reuses one sampled graph. Direct picks edges by score without forming subgraphs. Shuffled groups keeps subgraph sizes but scrambles which edges belong
together, testing the contribution of group membership at matched
subgraph sizes. Figure~\ref{fig:mechanism-heatmap} extends the comparison across the full measured retention range, with each cell showing the gap between EDiS-Lite and one control, at one dataset and one ratio. This shows where each control is weakest, not just its average, and a heatmap fits three controls times twelve ratios in one view better than a line per dataset would.

\begin{table}[H]
\centering
\footnotesize
\setlength{\tabcolsep}{2pt}
\renewcommand{\arraystretch}{1.02}
\begin{tabular*}{\linewidth}{@{}l@{\extracolsep{\fill}}cccc@{}}
\toprule
\textbf{Dataset} & \textbf{\edis{}-Lite} & \textbf{Frozen} & \textbf{Direct} & \textbf{Shuffled groups} \\
\midrule
Cora & \cellcolor[gray]{0.95}\underline{77.16}{\scriptsize$\pm$0.98} & 74.46{\scriptsize$\pm$1.00} & \cellcolor[rgb]{0.88,0.93,0.98}\textbf{77.44}{\scriptsize$\pm$1.21} & 76.46{\scriptsize$\pm$1.98} \\
\rowcolor[gray]{0.97}
CiteSeer & \cellcolor[gray]{0.95}\underline{66.08}{\scriptsize$\pm$1.23} & 65.24{\scriptsize$\pm$0.67} & \cellcolor[rgb]{0.88,0.93,0.98}\textbf{66.64}{\scriptsize$\pm$0.79} & 65.96{\scriptsize$\pm$1.05} \\
PubMed & \cellcolor[rgb]{0.88,0.93,0.98}\textbf{77.28}{\scriptsize$\pm$0.57} & \cellcolor[gray]{0.95}\underline{76.40}{\scriptsize$\pm$0.77} & 74.52{\scriptsize$\pm$0.72} & 73.88{\scriptsize$\pm$1.63} \\
\rowcolor[gray]{0.97}
Amazon-Computer & \cellcolor[rgb]{0.88,0.93,0.98}\textbf{93.35}{\scriptsize$\pm$0.21} & \cellcolor[gray]{0.95}\underline{92.96}{\scriptsize$\pm$0.32} & 91.44{\scriptsize$\pm$0.19} & 91.14{\scriptsize$\pm$0.26} \\
Amazon-Photo & \cellcolor[rgb]{0.88,0.93,0.98}\textbf{95.60}{\scriptsize$\pm$0.15} & \cellcolor[gray]{0.95}\underline{95.06}{\scriptsize$\pm$0.59} & 95.01{\scriptsize$\pm$0.39} & 94.94{\scriptsize$\pm$0.38} \\
\rowcolor[gray]{0.97}
Coauthor-CS & \cellcolor[rgb]{0.88,0.93,0.98}\textbf{95.72}{\scriptsize$\pm$0.06} & \cellcolor[gray]{0.95}\underline{95.59}{\scriptsize$\pm$0.05} & 94.97{\scriptsize$\pm$0.22} & 95.05{\scriptsize$\pm$0.33} \\
Coauthor-Physics & \cellcolor[gray]{0.95}\underline{97.06}{\scriptsize$\pm$0.09} & \cellcolor[rgb]{0.88,0.93,0.98}\textbf{97.10}{\scriptsize$\pm$0.11} & 96.74{\scriptsize$\pm$0.09} & 96.84{\scriptsize$\pm$0.08} \\
\rowcolor[gray]{0.97}
WikiCS & \cellcolor[rgb]{0.88,0.93,0.98}\textbf{80.25}{\scriptsize$\pm$0.79} & \cellcolor[gray]{0.95}\underline{79.98}{\scriptsize$\pm$0.35} & 77.47{\scriptsize$\pm$0.49} & 77.54{\scriptsize$\pm$0.56} \\
Squirrel & 42.63{\scriptsize$\pm$2.04} & 41.40{\scriptsize$\pm$2.08} & \cellcolor[gray]{0.95}\underline{43.13}{\scriptsize$\pm$1.12} & \cellcolor[rgb]{0.88,0.93,0.98}\textbf{43.35}{\scriptsize$\pm$2.37} \\
\rowcolor[gray]{0.97}
Chameleon & \cellcolor[rgb]{0.88,0.93,0.98}\textbf{43.75}{\scriptsize$\pm$4.14} & \cellcolor[gray]{0.95}\underline{43.73}{\scriptsize$\pm$3.19} & 42.14{\scriptsize$\pm$1.98} & 39.79{\scriptsize$\pm$3.50} \\
Roman-Empire & \cellcolor[rgb]{0.88,0.93,0.98}\textbf{68.58}{\scriptsize$\pm$0.12} & 67.34{\scriptsize$\pm$0.49} & 68.38{\scriptsize$\pm$0.59} & \cellcolor[gray]{0.95}\underline{68.54}{\scriptsize$\pm$0.38} \\
\rowcolor[gray]{0.97}
Amazon-Ratings & \cellcolor[rgb]{0.88,0.93,0.98}\textbf{52.31}{\scriptsize$\pm$0.85} & 50.03{\scriptsize$\pm$0.81} & \cellcolor[gray]{0.95}\underline{52.15}{\scriptsize$\pm$0.57} & 51.72{\scriptsize$\pm$0.39} \\
Minesweeper$^\dagger$ & \cellcolor[gray]{0.95}\underline{73.82}{\scriptsize$\pm$0.97} & 69.13{\scriptsize$\pm$0.97} & 73.15{\scriptsize$\pm$0.69} & \cellcolor[rgb]{0.88,0.93,0.98}\textbf{74.20}{\scriptsize$\pm$0.44} \\
\rowcolor[gray]{0.97}
Questions$^\dagger$ & \cellcolor[rgb]{0.88,0.93,0.98}\textbf{74.12}{\scriptsize$\pm$0.88} & \cellcolor[gray]{0.95}\underline{73.59}{\scriptsize$\pm$0.66} & 72.21{\scriptsize$\pm$0.95} & 72.38{\scriptsize$\pm$0.73} \\
\bottomrule
\end{tabular*}
\caption{Mechanism ablation for \edis{}-Lite at $\rho=0.30$ with GCN, MaxCF, and cosine scores. Values are mean $\pm$ standard deviation (\%). $^\dagger$Reported as ROC-AUC rather than accuracy. Best and second-best means are bold (blue) and underlined (gray).}
\label{tab:mechanism-main}
\end{table}

\begin{figure}[!htbp]
\centering
\includegraphics[width=\linewidth]{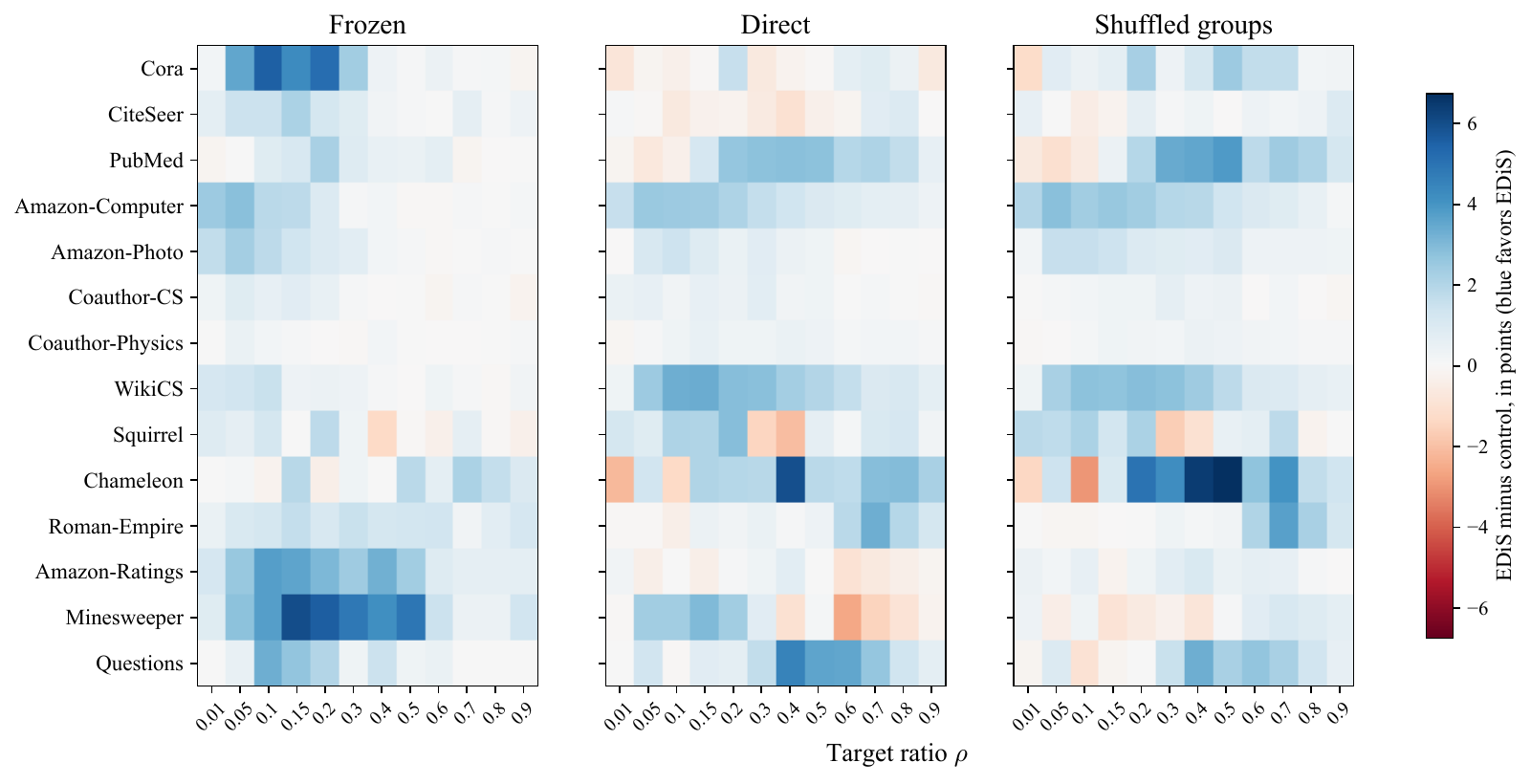}
\caption{\textbf{Mechanism ablation across the full retention range.} GCN backbone, MaxCF selector, cosine scores, all 14 datasets. Each cell is EDiS-Lite minus the named control, in accuracy or ROC-AUC points. Blue means EDiS-Lite wins.}
\label{fig:mechanism-heatmap}
\end{figure}

\clearpage
\subsection{Extended inference-graph ablation}
\label{app:inference-profile}

Evaluating on the full graph generally recovers some accuracy over EDiS-Lite, but how much depends heavily on the dataset. The gain is largest at low retention and shrinks toward zero as retention grows, since Sparse and Full become the same graph there. EDiS still evaluates on the sparse graph by default, which gives up some accuracy but keeps the method usable on very large graphs. Switching only the evaluation graph does not always close the gap to the independently-tuned Full GCN reference. EDiS-Full closely tracks that reference on several datasets (e.g., Coauthor-Physics, Amazon-Photo, WikiCS, Coauthor-CS), even slightly exceeding it on three of them, but a double-digit gap remains on Roman-Empire and Minesweeper, where the sparse training graph itself costs real accuracy.

The main text evaluates EDiS-Lite on the same sparse graph it trains on, since full-graph inference is often infeasible at the scales this method targets. Table~\ref{tab:inference-main} quantifies what that choice costs, comparing evaluation protocols for the 14 datasets. We fix GCN, MaxCF, and cosine scores at 30\% training retention. EDiS-Lite trains and evaluates on the sparse graph. EDiS-Full trains on the same sparse graph but evaluates on the full graph, isolating what the evaluation graph contributes on its own. Full GCN trains and evaluates on the full graph throughout, as a ratio-independent reference. Figure~\ref{fig:inference-graph-compact} extends this comparison across the full measured retention range, with each dataset given its own $y$-scale since the recovered amount varies by more than $40\times$ across datasets.

\begin{table}[H]
\centering
\footnotesize
\setlength{\tabcolsep}{2pt}
\renewcommand{\arraystretch}{1.02}
\begin{tabular*}{\linewidth}{@{}l@{\extracolsep{\fill}}cccc@{}}
\toprule
\textbf{Dataset} & \textbf{\edis{}-Lite} & \textbf{\edis{}-Full} & \textbf{Full GCN (ref.)} & \textbf{Gap (Full$-$Lite)} \\
\midrule
Cora & 77.16{\scriptsize$\pm$0.98} & 84.08{\scriptsize$\pm$0.83} & 84.26{\scriptsize$\pm$0.66} & $+6.92$ \\
\rowcolor[gray]{0.95}
CiteSeer & 66.08{\scriptsize$\pm$1.23} & 70.88{\scriptsize$\pm$0.78} & 72.50{\scriptsize$\pm$0.36} & $+4.80$ \\
PubMed & 77.28{\scriptsize$\pm$0.57} & 80.20{\scriptsize$\pm$0.46} & 80.76{\scriptsize$\pm$0.55} & $+2.92$ \\
\rowcolor[gray]{0.95}
Amazon-Computer & 93.35{\scriptsize$\pm$0.21} & 93.59{\scriptsize$\pm$0.05} & 93.70{\scriptsize$\pm$0.31} & $+0.24$ \\
Amazon-Photo & 95.60{\scriptsize$\pm$0.15} & 96.14{\scriptsize$\pm$0.19} & 95.95{\scriptsize$\pm$0.14} & $+0.54$ \\
\rowcolor[gray]{0.95}
Coauthor-CS & 95.72{\scriptsize$\pm$0.06} & 95.98{\scriptsize$\pm$0.15} & 95.61{\scriptsize$\pm$0.36} & $+0.26$ \\
Coauthor-Physics & 97.06{\scriptsize$\pm$0.09} & 97.45{\scriptsize$\pm$0.03} & 97.51{\scriptsize$\pm$0.07} & $+0.39$ \\
\rowcolor[gray]{0.95}
WikiCS & 80.25{\scriptsize$\pm$0.79} & 80.39{\scriptsize$\pm$0.29} & 80.57{\scriptsize$\pm$0.06} & $+0.14$ \\
Squirrel & 42.63{\scriptsize$\pm$2.04} & 43.26{\scriptsize$\pm$2.43} & 44.79{\scriptsize$\pm$1.82} & $+0.63$ \\
\rowcolor[gray]{0.95}
Chameleon & 43.75{\scriptsize$\pm$4.14} & 43.13{\scriptsize$\pm$2.53} & 44.67{\scriptsize$\pm$3.84} & $-0.62$ \\
Roman-Empire & 68.58{\scriptsize$\pm$0.12} & 75.00{\scriptsize$\pm$0.51} & 90.74{\scriptsize$\pm$0.03} & $+6.42$ \\
\rowcolor[gray]{0.95}
Amazon-Ratings & 52.31{\scriptsize$\pm$0.85} & 53.16{\scriptsize$\pm$0.68} & 53.52{\scriptsize$\pm$0.42} & $+0.85$ \\
Minesweeper$^\dagger$ & 73.82{\scriptsize$\pm$0.97} & 82.62{\scriptsize$\pm$1.09} & 97.49{\scriptsize$\pm$0.18} & $+8.80$ \\
\rowcolor[gray]{0.95}
Questions$^\dagger$ & 74.12{\scriptsize$\pm$0.88} & 77.70{\scriptsize$\pm$0.53} & 78.49{\scriptsize$\pm$0.18} & $+3.58$ \\
\bottomrule
\end{tabular*}
\caption{Inference-graph ablation for \edis{} at $\rho=0.30$ (GCN, MaxCF, cosine). \textbf{\edis{}-Lite} trains and evaluates on the sparse graph. \textbf{\edis{}-Full} trains on the sparse graph but evaluates on the full graph. \textbf{Full GCN} trains and evaluates on the full graph. Gap is \edis{}-Full minus \edis{}-Lite. $^\dagger$ROC-AUC rather than accuracy.}
\label{tab:inference-main}
\end{table}

\begin{figure}[!htbp]
\centering
\includegraphics[width=0.95\linewidth]{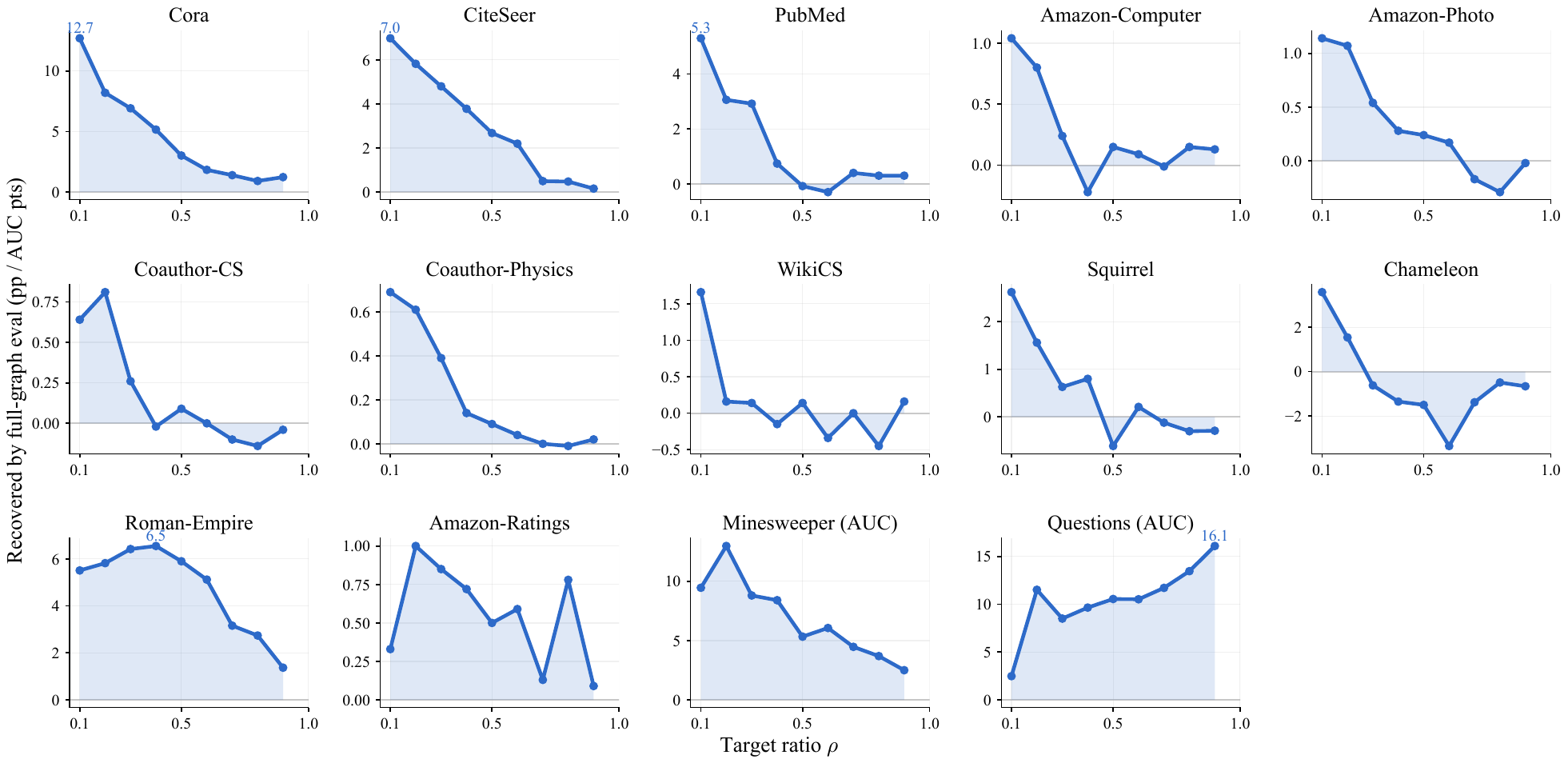}
\caption{\textbf{How much full-graph evaluation recovers over EDiS-Lite, across the full retention range.} GCN backbone, MaxCF selector, cosine scores. One dataset per panel, each with its own $y$-scale since the recovered amount varies by more than $40\times$ across datasets.}
\label{fig:inference-graph-compact}
\end{figure}

\clearpage
\subsection{Extended backbone comparison}
\label{app:backbone-profile}

EDiS is a framework that works with any backbone, not just GCN or GraphSAGE, and neither backbone dominates across this panel. How well EDiS performs on a given dataset depends more on how expressive that backbone is for the dataset, a property of the backbone's own architecture, than on the sparsification procedure itself. EDiS simply plugs into whichever backbone a task calls for and inherits its expressivity, rather than being tied to one. The largest gap between backbones appears on Squirrel, a few points at most, while most of the panel differs by a point or less, so the two backbones are close substitutes almost everywhere despite that occasional larger split. The stronger backbone does not track homophily in any simple way. GCN wins on some of the most heterophilic datasets in the panel (Squirrel, Chameleon) and some of the most homophilic (Cora, Coauthor-Physics) alike, so backbone choice reflects the architecture-dataset match itself rather than a graph statistic we could read off in advance.

Table~\ref{tab:backbone-main} extends the main-text backbone comparison to the 14 datasets, at $\rho\in\{0.30,0.50,0.70\}$. We fix MaxCF and cosine scores and vary only the backbone. Both training and evaluation use the sparse graph. Figure~\ref{fig:backbone-trends-full} extends the comparison across the full measured retention range, one line per backbone per dataset.

\begin{table}[H]
\centering
\footnotesize
\setlength{\tabcolsep}{2pt}
\renewcommand{\arraystretch}{1.02}
\begin{tabular*}{\linewidth}{@{}l@{\extracolsep{\fill}}cccccc@{}}
\toprule
& \multicolumn{2}{c}{$\rho=0.30$} & \multicolumn{2}{c}{$\rho=0.50$} & \multicolumn{2}{c}{$\rho=0.70$} \\
\cmidrule(lr){2-3}\cmidrule(lr){4-5}\cmidrule(lr){6-7}
Dataset & \textbf{GCN} & \textbf{GraphSAGE} & \textbf{GCN} & \textbf{GraphSAGE} & \textbf{GCN} & \textbf{GraphSAGE} \\
\midrule
Cora & \textbf{77.16{\scriptsize$\pm$0.98}} & 75.34{\scriptsize$\pm$0.91} & \textbf{80.28{\scriptsize$\pm$0.75}} & 78.90{\scriptsize$\pm$1.03} & \textbf{82.88{\scriptsize$\pm$0.50}} & 81.32{\scriptsize$\pm$0.34} \\
\rowcolor[gray]{0.95}
CiteSeer & \textbf{66.08{\scriptsize$\pm$1.23}} & 64.54{\scriptsize$\pm$1.51} & \textbf{68.20{\scriptsize$\pm$0.61}} & 66.72{\scriptsize$\pm$0.87} & \textbf{70.64{\scriptsize$\pm$0.84}} & 69.30{\scriptsize$\pm$0.78} \\
PubMed & \textbf{77.28{\scriptsize$\pm$0.57}} & 75.36{\scriptsize$\pm$1.52} & \textbf{80.32{\scriptsize$\pm$0.72}} & 77.52{\scriptsize$\pm$1.03} & \textbf{80.62{\scriptsize$\pm$0.54}} & 77.50{\scriptsize$\pm$1.56} \\
\rowcolor[gray]{0.95}
Amazon-Computer & \textbf{93.35{\scriptsize$\pm$0.21}} & 92.84{\scriptsize$\pm$0.30} & \textbf{93.60{\scriptsize$\pm$0.14}} & 92.98{\scriptsize$\pm$0.18} & \textbf{93.88{\scriptsize$\pm$0.11}} & 92.73{\scriptsize$\pm$0.09} \\
Amazon-Photo & 95.60{\scriptsize$\pm$0.15} & \textbf{96.71{\scriptsize$\pm$0.08}} & 95.82{\scriptsize$\pm$0.19} & \textbf{96.38{\scriptsize$\pm$0.08}} & 95.88{\scriptsize$\pm$0.05} & \textbf{96.56{\scriptsize$\pm$0.20}} \\
\rowcolor[gray]{0.95}
Coauthor-CS & 95.72{\scriptsize$\pm$0.06} & \textbf{96.06{\scriptsize$\pm$0.12}} & 95.86{\scriptsize$\pm$0.15} & \textbf{96.11{\scriptsize$\pm$0.09}} & 95.95{\scriptsize$\pm$0.11} & \textbf{96.30{\scriptsize$\pm$0.08}} \\
Coauthor-Physics & \textbf{97.06{\scriptsize$\pm$0.09}} & 96.98{\scriptsize$\pm$0.05} & \textbf{97.34{\scriptsize$\pm$0.05}} & 97.18{\scriptsize$\pm$0.03} & \textbf{97.47{\scriptsize$\pm$0.07}} & 97.08{\scriptsize$\pm$0.17} \\
\rowcolor[gray]{0.95}
WikiCS & \textbf{80.25{\scriptsize$\pm$0.79}} & 79.97{\scriptsize$\pm$0.32} & 80.19{\scriptsize$\pm$0.33} & \textbf{80.76{\scriptsize$\pm$0.48}} & 80.17{\scriptsize$\pm$0.43} & \textbf{80.69{\scriptsize$\pm$0.30}} \\
Squirrel & \textbf{42.63{\scriptsize$\pm$2.04}} & 39.30{\scriptsize$\pm$2.29} & \textbf{44.50{\scriptsize$\pm$1.69}} & 40.37{\scriptsize$\pm$2.48} & \textbf{44.15{\scriptsize$\pm$1.65}} & 40.10{\scriptsize$\pm$2.19} \\
\rowcolor[gray]{0.95}
Chameleon & \textbf{43.75{\scriptsize$\pm$4.14}} & 42.21{\scriptsize$\pm$4.98} & \textbf{45.64{\scriptsize$\pm$2.72}} & 42.77{\scriptsize$\pm$4.50} & \textbf{44.88{\scriptsize$\pm$4.25}} & 43.53{\scriptsize$\pm$5.64} \\
Roman-Empire & \textbf{68.95{\scriptsize$\pm$0.41}} & 68.94{\scriptsize$\pm$0.28} & 73.75{\scriptsize$\pm$0.48} & \textbf{73.93{\scriptsize$\pm$0.40}} & 83.03{\scriptsize$\pm$0.32} & \textbf{83.35{\scriptsize$\pm$0.05}} \\
\rowcolor[gray]{0.95}
Amazon-Ratings & 52.31{\scriptsize$\pm$0.85} & \textbf{53.19{\scriptsize$\pm$0.36}} & 53.62{\scriptsize$\pm$0.58} & \textbf{54.37{\scriptsize$\pm$0.31}} & 53.14{\scriptsize$\pm$0.27} & \textbf{54.58{\scriptsize$\pm$0.44}} \\
Minesweeper$^\dagger$ & 73.82{\scriptsize$\pm$0.97} & \textbf{74.06{\scriptsize$\pm$0.92}} & \textbf{82.21{\scriptsize$\pm$0.60}} & 81.93{\scriptsize$\pm$0.26} & 85.37{\scriptsize$\pm$0.59} & \textbf{85.51{\scriptsize$\pm$0.54}} \\
\rowcolor[gray]{0.95}
Questions$^\dagger$ & \textbf{74.12{\scriptsize$\pm$0.88}} & 72.72{\scriptsize$\pm$0.68} & \textbf{76.99{\scriptsize$\pm$0.54}} & 74.87{\scriptsize$\pm$0.86} & \textbf{78.13{\scriptsize$\pm$0.28}} & 75.28{\scriptsize$\pm$0.81} \\
\bottomrule
\end{tabular*}
\caption{Backbone-generalization ablation for \edis{}-Lite with MaxCF and cosine scores, at $\rho\in\{0.30,0.50,0.70\}$. Values are mean $\pm$ standard deviation (\%). $^\dagger$Reported as ROC-AUC rather than accuracy. The higher-scoring backbone at each ratio is bold. No single backbone dominates across the panel.}
\label{tab:backbone-main}
\end{table}

\begin{figure}[!htbp]
\centering
\includegraphics[width=\linewidth]{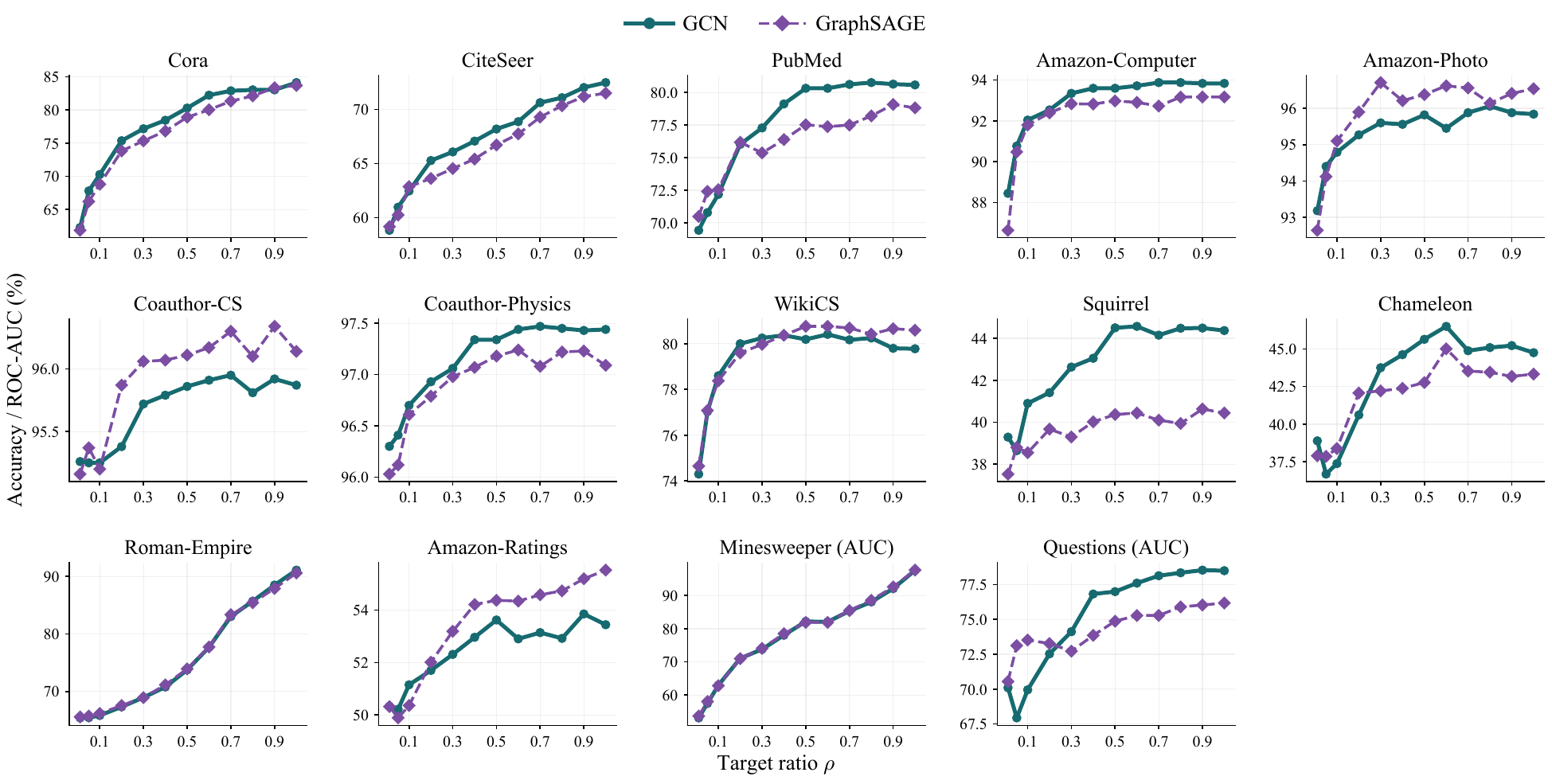}
\caption{\textbf{Backbone comparison across the full retention range.} MaxCF selector, cosine scores, all 14 datasets. GCN is drawn heavier since it is the backbone used everywhere else in this paper.}
\label{fig:backbone-trends-full}
\end{figure}

\clearpage
\subsection{Selector and score robustness under GraphSAGE}
\label{app:sage-robustness}

The same broad pattern holds under GraphSAGE as under GCN. No single selector or edge score wins on every dataset, and the defaults remain reasonable choices rather than the best possible pick everywhere, though the specific winners are not identical across backbones. Degree and $k$-NN never win outright under GCN anywhere in this appendix, yet under GraphSAGE Degree wins on three datasets (Coauthor-Physics, WikiCS, Minesweeper) and $k$-NN wins one (Roman-Empire), so which selectors are merely acceptable versus genuinely competitive is itself backbone-dependent. Minesweeper's large MLP advantage, flagged for GCN earlier in this appendix, also appears under GraphSAGE, so it is a property of the dataset rather than the backbone.

RQ5 and RQ7 each vary one axis of the pipeline while holding the other fixed, and this section checks whether the two interact. If GraphSAGE preferred a different selector or score than GCN does, RQ5's conclusion would only hold for the one backbone it was tested under. Tables~\ref{tab:selector-main} and~\ref{tab:score-main} show that no single selector or edge score dominates under GCN. Tables~\ref{tab:selector-sage-main} and~\ref{tab:score-sage-main} repeat that check under GraphSAGE, at $\rho=0.30$. The selector table fixes cosine scores. The score table fixes MaxCF. The PubMed--MLP score is marked N/A because its training split cannot supply the balanced labeled-edge sample required to train the scorer.

\begin{table}[H]
\centering
\footnotesize
\setlength{\tabcolsep}{2pt}
\renewcommand{\arraystretch}{1.02}
\begin{tabular*}{\linewidth}{@{}l@{\extracolsep{\fill}}ccccccc@{}}
\toprule
Dataset & \shortstack{MaxCF\\(default)} & MinCF & Degree & $k$-NN & Spanner & Random & Hybrid \\
\midrule
Cora & \underline{75.34}{\scriptsize$\pm$0.91} & \textbf{75.42}{\scriptsize$\pm$1.22} & 75.14{\scriptsize$\pm$0.87} & 73.96{\scriptsize$\pm$0.98} & 75.26{\scriptsize$\pm$1.10} & 74.38{\scriptsize$\pm$1.14} & 74.64{\scriptsize$\pm$1.29} \\
\rowcolor[gray]{0.95}
CiteSeer & 64.54{\scriptsize$\pm$1.51} & \underline{65.62}{\scriptsize$\pm$0.35} & 65.18{\scriptsize$\pm$0.92} & 64.54{\scriptsize$\pm$1.33} & \textbf{66.36}{\scriptsize$\pm$0.70} & 65.32{\scriptsize$\pm$0.71} & 64.54{\scriptsize$\pm$0.69} \\
PubMed & 75.36{\scriptsize$\pm$1.52} & \textbf{75.62}{\scriptsize$\pm$0.39} & \underline{75.54}{\scriptsize$\pm$2.02} & 73.86{\scriptsize$\pm$1.26} & 75.26{\scriptsize$\pm$1.20} & 73.60{\scriptsize$\pm$1.01} & 75.32{\scriptsize$\pm$1.76} \\
\rowcolor[gray]{0.95}
Amazon-Computer & \underline{92.84}{\scriptsize$\pm$0.30} & \textbf{93.01}{\scriptsize$\pm$0.23} & 92.62{\scriptsize$\pm$0.10} & 91.37{\scriptsize$\pm$0.48} & 91.62{\scriptsize$\pm$0.15} & 91.37{\scriptsize$\pm$0.57} & 92.18{\scriptsize$\pm$0.29} \\
Amazon-Photo & \textbf{96.71}{\scriptsize$\pm$0.08} & 96.12{\scriptsize$\pm$0.16} & \underline{96.41}{\scriptsize$\pm$0.33} & 94.71{\scriptsize$\pm$0.14} & 95.48{\scriptsize$\pm$0.22} & 95.27{\scriptsize$\pm$0.21} & 95.88{\scriptsize$\pm$0.29} \\
\rowcolor[gray]{0.95}
Coauthor-CS & \textbf{96.06}{\scriptsize$\pm$0.12} & 95.55{\scriptsize$\pm$0.15} & 95.55{\scriptsize$\pm$0.05} & 95.72{\scriptsize$\pm$0.34} & 95.65{\scriptsize$\pm$0.25} & 95.40{\scriptsize$\pm$0.09} & \underline{95.91}{\scriptsize$\pm$0.25} \\
Coauthor-Physics & \underline{96.98}{\scriptsize$\pm$0.05} & 96.75{\scriptsize$\pm$0.13} & \textbf{97.01}{\scriptsize$\pm$0.02} & 96.70{\scriptsize$\pm$0.08} & 96.64{\scriptsize$\pm$0.09} & 96.61{\scriptsize$\pm$0.07} & 96.97{\scriptsize$\pm$0.17} \\
\rowcolor[gray]{0.95}
WikiCS & \underline{79.97}{\scriptsize$\pm$0.32} & 79.38{\scriptsize$\pm$0.53} & \textbf{80.05}{\scriptsize$\pm$0.40} & 79.07{\scriptsize$\pm$0.16} & 78.36{\scriptsize$\pm$0.47} & 78.11{\scriptsize$\pm$0.29} & 78.73{\scriptsize$\pm$0.56} \\
Squirrel & 39.30{\scriptsize$\pm$2.29} & \underline{40.38}{\scriptsize$\pm$2.01} & 38.97{\scriptsize$\pm$2.02} & 38.96{\scriptsize$\pm$2.12} & \textbf{40.42}{\scriptsize$\pm$1.37} & 39.85{\scriptsize$\pm$1.00} & 39.14{\scriptsize$\pm$1.51} \\
\rowcolor[gray]{0.95}
Chameleon & \underline{42.21}{\scriptsize$\pm$4.98} & \textbf{42.41}{\scriptsize$\pm$4.74} & 40.49{\scriptsize$\pm$6.17} & 41.00{\scriptsize$\pm$4.42} & 40.00{\scriptsize$\pm$6.78} & 40.28{\scriptsize$\pm$4.23} & 41.56{\scriptsize$\pm$7.15} \\
Roman-Empire & 68.94{\scriptsize$\pm$0.28} & \underline{69.25}{\scriptsize$\pm$0.64} & 69.02{\scriptsize$\pm$0.33} & \textbf{69.31}{\scriptsize$\pm$0.19} & 68.88{\scriptsize$\pm$0.85} & 68.88{\scriptsize$\pm$0.28} & 68.98{\scriptsize$\pm$0.55} \\
\rowcolor[gray]{0.95}
Amazon-Ratings & \textbf{53.19}{\scriptsize$\pm$0.36} & \underline{53.09}{\scriptsize$\pm$0.26} & 51.96{\scriptsize$\pm$0.33} & 52.46{\scriptsize$\pm$0.35} & 52.16{\scriptsize$\pm$0.53} & 52.96{\scriptsize$\pm$0.35} & 52.76{\scriptsize$\pm$0.33} \
 \\
Minesweeper$^\dagger$ & 74.06{\scriptsize$\pm$0.92} & 75.00{\scriptsize$\pm$0.60} & \textbf{76.25}{\scriptsize$\pm$0.70} & 73.32{\scriptsize$\pm$0.18} & 75.03{\scriptsize$\pm$0.94} & \underline{75.84}{\scriptsize$\pm$1.01} & 73.53{\scriptsize$\pm$1.81} \\
\rowcolor[gray]{0.95}
Questions$^\dagger$ & 72.72{\scriptsize$\pm$0.68} & \textbf{74.65}{\scriptsize$\pm$1.23} & \underline{73.93}{\scriptsize$\pm$0.49} & 72.39{\scriptsize$\pm$0.96} & 72.91{\scriptsize$\pm$0.96} & 73.59{\scriptsize$\pm$0.86} & 72.84{\scriptsize$\pm$0.72} \\
\bottomrule
\end{tabular*}
\caption{Selector ablation for \edis{}-Lite under GraphSAGE (cosine scores, $\rho=0.30$). Values are mean $\pm$ standard deviation (\%). $^\dagger$Reported as ROC-AUC rather than accuracy. Best and second-best reported means are bold and underlined, respectively.}
\label{tab:selector-sage-main}
\end{table}

\begin{table}[H]
\centering
\footnotesize
\setlength{\tabcolsep}{2pt}
\renewcommand{\arraystretch}{1.02}
\begin{tabular*}{\linewidth}{@{}l@{\extracolsep{\fill}}ccccc@{}}
\toprule
Dataset & \shortstack{Cosine\\(default)} & Euclidean & Dot & Uniform & MLP \\
\midrule
Cora & 75.34{\scriptsize$\pm$0.91} & 74.82{\scriptsize$\pm$1.42} & 74.82{\scriptsize$\pm$0.37} & \underline{75.48}{\scriptsize$\pm$1.65} & \textbf{75.68}{\scriptsize$\pm$1.04} \\
\rowcolor[gray]{0.95}
CiteSeer & 64.54{\scriptsize$\pm$1.51} & \textbf{65.82}{\scriptsize$\pm$0.99} & 64.32{\scriptsize$\pm$0.48} & 63.76{\scriptsize$\pm$2.00} & \underline{65.56}{\scriptsize$\pm$1.07} \\
PubMed & 75.36{\scriptsize$\pm$1.52} & \textbf{75.80}{\scriptsize$\pm$2.34} & 75.16{\scriptsize$\pm$1.74} & \underline{75.56}{\scriptsize$\pm$1.77} & -- \\
\rowcolor[gray]{0.95}
Amazon-Computer & 92.84{\scriptsize$\pm$0.30} & 92.29{\scriptsize$\pm$0.15} & \textbf{93.38}{\scriptsize$\pm$0.14} & \underline{93.37}{\scriptsize$\pm$0.27} & 92.94{\scriptsize$\pm$0.16} \\
Amazon-Photo & \textbf{96.71}{\scriptsize$\pm$0.08} & 96.32{\scriptsize$\pm$0.12} & \underline{96.58}{\scriptsize$\pm$0.08} & 96.43{\scriptsize$\pm$0.08} & \underline{96.58}{\scriptsize$\pm$0.03} \\
\rowcolor[gray]{0.95}
Coauthor-CS & \textbf{96.06}{\scriptsize$\pm$0.12} & \underline{96.01}{\scriptsize$\pm$0.13} & 95.75{\scriptsize$\pm$0.10} & 95.70{\scriptsize$\pm$0.11} & 95.56{\scriptsize$\pm$0.09} \\
Coauthor-Physics & \underline{96.98}{\scriptsize$\pm$0.05} & \underline{96.98}{\scriptsize$\pm$0.08} & \textbf{96.99}{\scriptsize$\pm$0.08} & 96.96{\scriptsize$\pm$0.08} & 96.75{\scriptsize$\pm$0.09} \\
\rowcolor[gray]{0.95}
WikiCS & 79.97{\scriptsize$\pm$0.32} & 80.02{\scriptsize$\pm$0.20} & \underline{80.17}{\scriptsize$\pm$0.34} & \textbf{80.39}{\scriptsize$\pm$0.48} & 79.63{\scriptsize$\pm$0.33} \\
Squirrel & 39.30{\scriptsize$\pm$2.29} & \underline{40.65}{\scriptsize$\pm$1.51} & 40.01{\scriptsize$\pm$1.30} & \textbf{41.07}{\scriptsize$\pm$1.55} & 39.22{\scriptsize$\pm$2.21} \\
\rowcolor[gray]{0.95}
Chameleon & 42.21{\scriptsize$\pm$4.98} & 40.21{\scriptsize$\pm$4.71} & \textbf{43.56}{\scriptsize$\pm$5.47} & \underline{42.41}{\scriptsize$\pm$4.42} & 41.52{\scriptsize$\pm$4.34} \\
Roman-Empire & 68.94{\scriptsize$\pm$0.28} & 69.26{\scriptsize$\pm$0.16} & 69.30{\scriptsize$\pm$0.23} & \textbf{70.88}{\scriptsize$\pm$0.27} & \underline{69.76}{\scriptsize$\pm$0.52} \\
\rowcolor[gray]{0.95}
Amazon-Ratings & 53.19{\scriptsize$\pm$0.36} & \underline{53.49}{\scriptsize$\pm$0.50} & \textbf{53.70}{\scriptsize$\pm$0.68} & 53.48{\scriptsize$\pm$0.18} & 52.95{\scriptsize$\pm$0.39} \\
Minesweeper$^\dagger$ & 74.06{\scriptsize$\pm$0.92} & 74.20{\scriptsize$\pm$0.34} & 72.90{\scriptsize$\pm$0.57} & \underline{76.84}{\scriptsize$\pm$0.80} & \textbf{82.82}{\scriptsize$\pm$1.82} \\
\rowcolor[gray]{0.95}
Questions$^\dagger$ & 72.72{\scriptsize$\pm$0.68} & 71.99{\scriptsize$\pm$1.19} & \underline{73.03}{\scriptsize$\pm$0.89} & 72.93{\scriptsize$\pm$0.61} & \textbf{73.46}{\scriptsize$\pm$0.96} \\
\bottomrule
\end{tabular*}
\caption{Score-function ablation for \edis{}-Lite under GraphSAGE (MaxCF selector, $\rho=0.30$). Values are mean $\pm$ standard deviation (\%). $^\dagger$Reported as ROC-AUC rather than accuracy. Best and second-best reported means are bold and underlined, respectively. A dash marks PubMed--MLP, which is not applicable in this implementation.}
\label{tab:score-sage-main}
\end{table}

\clearpage
\subsection{Extraction depth on large graphs}
\label{app:kmax-sensitivity}

This subsection asks how much the round cap $K_{\max}$ matters once
training already uses a fixed edge-budget. $K_{\max}$ limits how
many subgraphs extraction builds, not how many edges training uses, so
a shallow cap could starve composition or change nothing at all.

Table~\ref{tab:kmax-main} extends the main-text extraction-depth comparison (RQ6, Figure~\ref{fig:kmax-sensitivity}) with the actual extracted depth and the construction-time breakdown behind it. We examine $K_{\max}\in\{5,10,20,40,80,160,320\}$ at 30\% retention on the five large graphs.

\begin{table}[H]
\centering
\footnotesize
\setlength{\tabcolsep}{3pt}
\renewcommand{\arraystretch}{1.05}
\begin{tabular*}{\linewidth}{@{}l@{\extracolsep{\fill}}ccccc@{}}
\toprule
Dataset & \shortstack{Saturation\\depth} & \shortstack{Best $K_{\max}$\\(Lite)} & \shortstack{Best $K_{\max}$\\(Full)} & \shortstack{Precompute\\range (s)} & \shortstack{Composition\\range (s)} \\
\midrule
Pokec & 67 & 5 & 80 & 47--86 & 410--438 \\
\rowcolor[gray]{0.95}
ogbn-products & 114 & 10 & 10 & 122--378 & 192--1155 \\
ogbn-proteins$^\dagger$ & 320 & 10 & 10 & 75--1081 & 40--477 \\
\rowcolor[gray]{0.95}
ogbn-arxiv & 43 & 320 & 80 & 2--3 & 12--12 \\
Reddit & 320 & 20 & 5 & 135--1906 & 196--1232 \\
\bottomrule
\end{tabular*}
\caption{Extraction-depth summary for the five large graphs at $\rho=0.30$. Saturation depth is the largest actual group count reached in the sweep. Best $K_{\max}$ is the cap with the highest reported mean under each evaluation protocol. Precompute and composition ranges span the Lite sweep from $K_{\max}=5$ to the largest measured cap. $^\dagger$Reported as ROC-AUC rather than accuracy.}
\label{tab:kmax-main}
\end{table}

Extraction saturates once the residual graph runs out of edges to
add, so increasing the cap does not necessarily deepen the
decomposition. Arxiv, Products, and Pokec saturate well below the
cap, at 43, 114, and 67 subgraphs, respectively, while Proteins and
Reddit exhaust the full cap of 320. Once a dataset saturates, its
stored decomposition no longer changes, so any accuracy variation
observed above that point, as on Arxiv, cannot be attributed to
additional extraction rounds.

Even where extraction does not saturate, deeper extraction shifts
cost from composition to precomputation rather than delivering an
end-to-end speedup. On Reddit, for instance, Lite precomputation
time rises from cap 5 to cap 320 while composition time falls.

To separate a shifting best-observed cap from practically meaningful
sensitivity, Table~\ref{tab:kmax-common-best} compares a common
$K_{\max}=20$ against each dataset's highest observed test mean
across the sweep, the latter reported only as a retrospective
reference, not a tuning target. On ogbn-arxiv, the best cap shifts
to $K=320$ for Lite and $K=80$ for Full, yet neither exceeds
$K=20$ by more than a third of a percentage point. On Reddit,
$K=20$ is already the best cap for Lite, leaving nothing on the
table. The largest gain in the entire sweep, on ogbn-proteins,
remains under one percentage point.
The identity of the best cap shifts considerably more than its
performance advantage does.

\begin{table}[H]
\centering
\footnotesize
\setlength{\tabcolsep}{4pt}
\renewcommand{\arraystretch}{1.05}
\begin{tabular*}{\linewidth}{@{}l@{\extracolsep{\fill}}c cc cc@{}}
\toprule
 & & \multicolumn{2}{c}{\textbf{\edis-Lite}} & \multicolumn{2}{c}{\textbf{\edis-Full}} \\
\cmidrule(lr){3-4}\cmidrule(lr){5-6}
Dataset & $K_{\mathrm{common}}$ & Best $K$ & $\Delta$ (pt) & Best $K$ & $\Delta$ (pt) \\
\midrule
ogbn-arxiv & 20 & 320 & $+0.27$ & 80 & $+0.32$ \\
\rowcolor[gray]{0.95}
ogbn-products & 20 & 10 & $+0.42$ & 10 & $+0.40$ \\
ogbn-proteins$^\dagger$ & 20 & 10 & $+0.59$ & 10 & $+0.83$ \\
\rowcolor[gray]{0.95}
Pokec & 20 & 5 & $+0.23$ & 80 & $+0.75$ \\
Reddit & 20 & 20 & $0.00$ & 5 & $+0.08$ \\
\bottomrule
\end{tabular*}
\caption{Common-$K$ vs.\ best observed $K$ at $\rho=0.3$ on the five
large graphs ($K_{\mathrm{common}}=20$). $\Delta$ is the best cap's
gain over the default, in points, and is a retrospective reference,
not a tuning rule. $^\dagger$ROC-AUC rather than accuracy.}
\label{tab:kmax-common-best}
\end{table}

\begin{keyquestion}
\textbf{Observation.} The extraction-depth argmax varies across datasets and inference protocols, but this overstates practical sensitivity. Across all five large graphs, the common $K_{\max}=20$ configuration remains within one percentage point of the best observed mean for both Lite and Full. $K_{\max}$ mainly trades precomputation time against composition time, not accuracy, so it needs no precise per-dataset tuning.
\end{keyquestion}

\clearpage
\subsection{Coverage, overlap, and decomposition reuse}
\label{app:coverage-reuse}

\captionsetup{font=footnotesize}
\setlength{\intextsep}{3pt}

Using the Frozen and Direct controls from the mechanism ablation at
$\rho=0.3$, we measure how training views cover the edge set over
time and how one extraction serves multiple retention budgets.

\underline{\textbf{Coverage and temporal overlap.}}\quad
Figure~\ref{fig:coverage-overlap} measures cumulative edge coverage,
the fraction of edges used at least once, and the Jaccard overlap
between consecutive epoch graphs. Direct and Frozen mark the two
extremes, one with full turnover and almost no overlap, the other
with one fixed graph and no turnover at all. \edis{} sits between
them in these measurements,
exploring nearly the whole edge set like Direct while keeping more
structure in common between epochs. In these runs it covers $99$--$100\%$ of
edges, close to Direct, but its Jaccard overlap ($0.22$--$0.34$) is
roughly double Direct's ($0.17$--$0.18$), still well short of
Frozen's trivial $1.0$.

\begin{figure}[H]
\centering
\includegraphics[width=\linewidth]{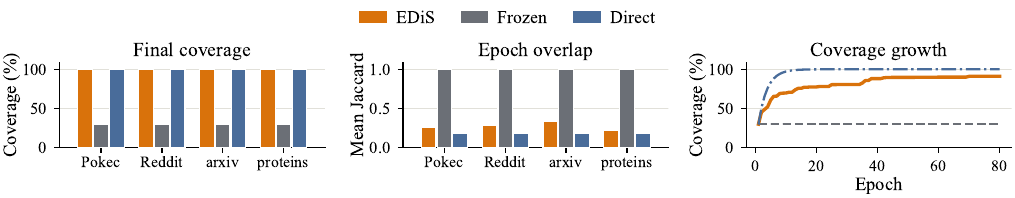}
\caption{\textbf{Coverage and temporal overlap at $\rho=0.3$.}
Cumulative edge coverage (left), mean consecutive-epoch Jaccard
(middle), and coverage over epochs for Pokec (right). Direct
maximizes turnover, Frozen fixes one graph, and \edis{} lies between
them, near-complete coverage with greater temporal persistence than
Direct.}
\label{fig:coverage-overlap}
\end{figure}

\edis{} does not maximize raw turnover. Direct explores at least as
broadly, but exploration alone is not sufficient. \edis{} exceeds
Direct on 11 of the 14 datasets (Appendix~\ref{app:mechanism-profile},
$\rho=0.3$), and Shuffled groups is worse despite matching \edis{}'s
subgraph sizes, so how edges are organized into subgraphs matters, not
just how many are seen. This analysis measures turnover from subgraph
resampling and trimming together and does not isolate the residual
pool's contribution. That separation is left to future work.

\underline{\textbf{Cost of dynamic composition.}}\quad Direct pays a repeated
cost because it resamples from the full edge set every epoch,
whereas \edis{} draws from a decomposition computed once. Training,
evaluation, and peak GPU memory stay essentially unchanged across
all three methods, isolating the difference to graph generation
alone. There, \edis{} costs roughly half of Direct's per-epoch
composition time ($1196$ ms against $2513$ ms) while still varying
the graph every epoch. Frozen is cheaper still, so
\edis{} costs more than a fixed graph but less than per-epoch
resampling.

\underline{\textbf{Decomposition reuse across edge-budgets.}}\quad \edis{} computes
its decomposition independently of the edge-budget, so once cached it
serves multiple $\rho$ without rerunning extraction
(Figure~\ref{fig:amortization}). Rebuilding separately for each of
$q$ edge-budgets grows construction cost linearly with $q$. \edis{} pays
it once, cutting cumulative cost by $5\times$ over five independent
rebuilds.

\begin{figure}[H]
\centering
\includegraphics[width=\linewidth]{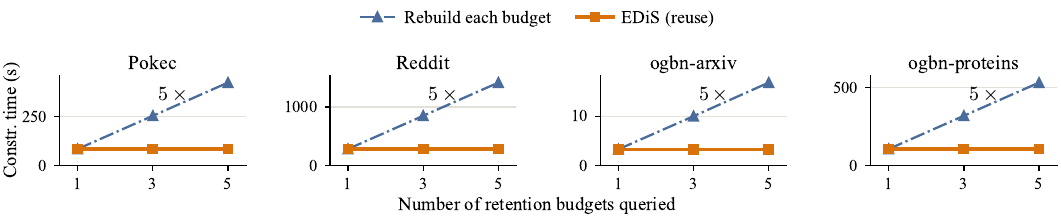}
\caption{\textbf{Amortization of decomposition across edge-budgets.}
\edis{} reuses one cached decomposition. The rebuild-each-edge-budget
curve is the cumulative cost of repeating extraction per edge-budget,
growing linearly, whereas \edis{} pays it once.}
\label{fig:amortization}
\end{figure}

\begin{keyquestion}
\textbf{Observation.} \edis{} is neither a fixed sparsifier nor a
maximizer of random edge turnover. In the profiled runs it covers nearly all edges over
training, keeps more overlap between consecutive epochs than Direct,
and costs less per epoch to compose. With the Frozen, Direct, and
Shuffled groups controls, performance depends not only on how many
distinct edges are seen but on how changing sparse views are
organized and reused.
\end{keyquestion}

\newpage

\section{Additional Experiments}
\label{app:gern_spangnn}

\subsection{Comparison with SpanGNN and GERN}
\label{app:spangnn_gern_sub}

Table~\ref{tab:relwork-comparison} (Appendix~\ref{app:related_work})
compares \edis{} against AGS-GNN \citep{Das:2024agsgnn}, SpanGNN
\citep{Gu:2024spangnn}, and GERN \citep{Bonchi:2025gern} qualitatively,
by budget control, reusability, multi-structure exposure, and
edge-disjointness. Here we compare SpanGNN and GERN quantitatively and
discuss our attempted AGS-GNN evaluation.

\underline{\textbf{AGS-GNN}} \citep{Das:2024agsgnn} samples edges by a
combination of attribute similarity and diversity, selecting diverse
yet locally similar neighborhoods through greedy submodular function
maximization. Its sampling objective differs from \edis{}'s
score-based composition, and  its published
version provides no analysis of edge-disjoint or reusable structural
properties.
We attempted to include AGS-GNN under the same matched-budget
protocol as our other 17 baselines, but its submodular diversity
step was extremely slow and failed to complete on our evaluation
hardware. AGS-GNN's own paper notes that this precomputation is
embarrassingly parallel and relies on scaling up the number of
worker threads to run in practice, a requirement outside the
modest budget we allocate uniformly across all 19 datasets and
three retention ratios. We omit AGS-GNN from
Table~\ref{tab:gern_spangnn}.

\underline{\textbf{SpanGNN}} \citep{Gu:2024spangnn} instead grows a single
spanning subgraph incrementally across epochs to reduce memory during
training, rather than composing a fresh sparse graph from reusable
structure each epoch. We adapt SpanGNN to our training framework and
compare it against \edis{} under matched training conditions.

\underline{\textbf{GERN}} \citep{Bonchi:2025gern} instead trains a GNN over a
sequence of random spanning trees, each converted into a path graph
and used as that epoch's sparse training structure, so it exposes the
model to many cheap, independently resampled sparse views rather than
a persistent, budget-controlled decomposition. We adapt GERN to our
training framework, running it with the same tuned GCN
backbone, optimizer, training protocol, and dataset splits as
\edis{}, isolating the sparsification and composition mechanism from
incidental training differences.

Both methods are \underline{\textbf{evaluated under full-graph inference}}, so this
table reads against \textbf{EDiS-Full}, not EDiS-Lite, on the five
datasets with SpanGNN and GERN results. Both reuse a sequence of
sparse structures across epochs, as \edis{} does, but neither
separates one-time decomposition from per-epoch composition
(Section~\ref{sec:RelatedWork}).

\vspace{0.5\baselineskip}
\begingroup
\definecolor{gernBand}{RGB}{235,235,235}
\newcommand{\gv}[2]{\shortstack{#1\\[-1pt]{\fontsize{5.5pt}{6pt}\selectfont\color{black!80}\textpm#2}}}
\begin{table}[h]
\centering
\setlength{\tabcolsep}{3pt}
\renewcommand{\arraystretch}{1.35}
\begin{adjustbox}{max width=\textwidth}
\small
\begin{tabular}{l ccc ccc ccc ccc ccc}
\toprule
& \multicolumn{3}{c}{Cora} & \multicolumn{3}{c}{CiteSeer} & \multicolumn{3}{c}{PubMed} & \multicolumn{3}{c}{WikiCS} & \multicolumn{3}{c}{Amazon-Ratings} \\
\cmidrule(lr){2-4}\cmidrule(lr){5-7}\cmidrule(lr){8-10}\cmidrule(lr){11-13}\cmidrule(lr){14-16}
Method & 0.3 & 0.5 & 0.7 & 0.3 & 0.5 & 0.7 & 0.3 & 0.5 & 0.7 & 0.3 & 0.5 & 0.7 & 0.3 & 0.5 & 0.7 \\
\midrule
SpanGNN
 & \cellcolor{gernBand}\gv{74.60}{0.76} & \cellcolor{gernBand}\gv{77.74}{1.19} & \cellcolor{gernBand}\gv{77.98}{1.66}
 & \gv{60.60}{2.20} & \gv{64.52}{1.55} & \gv{62.82}{0.64}
 & \cellcolor{gernBand}\gv{73.14}{0.55} & \cellcolor{gernBand}\gv{75.66}{0.89} & \cellcolor{gernBand}\gv{76.58}{0.87}
 & \gv{78.70}{0.22} & \gv{78.51}{0.06} & \gv{78.56}{0.31}
 & \cellcolor{gernBand}\gv{51.05}{0.48} & \cellcolor{gernBand}\gv{51.35}{0.28} & \cellcolor{gernBand}\gv{51.45}{0.50} \\
\midrule
GERN
 & \cellcolor{gernBand}\gv{80.46}{0.72} & \cellcolor{gernBand}\gv{80.68}{0.57} & \cellcolor{gernBand}\gv{81.10}{0.97}
 & \gv{68.22}{1.24} & \gv{69.34}{1.06} & \gv{70.06}{0.86}
 & \cellcolor{gernBand}\gv{75.88}{1.21} & \cellcolor{gernBand}\gv{76.00}{1.33} & \cellcolor{gernBand}\gv{75.58}{1.29}
 & \gv{80.37}{0.23} & \gv{\textbf{80.41}}{0.14} & \gv{\textbf{80.42}}{0.19}
 & \cellcolor{gernBand}\gv{46.33}{0.19} & \cellcolor{gernBand}\gv{49.35}{0.31} & \cellcolor{gernBand}\gv{51.04}{0.60} \\
\midrule
\textbf{EDiS-Full}
 & \cellcolor{gernBand}\gv{\textbf{84.08}}{0.83} & \cellcolor{gernBand}\gv{\textbf{83.30}}{0.61} & \cellcolor{gernBand}\gv{\textbf{84.28}}{0.55}
 & \gv{\textbf{70.88}}{0.78} & \gv{\textbf{70.88}}{0.47} & \gv{\textbf{71.14}}{0.29}
 & \cellcolor{gernBand}\gv{\textbf{80.20}}{0.46} & \cellcolor{gernBand}\gv{\textbf{80.24}}{0.61} & \cellcolor{gernBand}\gv{\textbf{81.02}}{0.38}
 & \gv{\textbf{80.39}}{0.29} & \gv{80.33}{0.32} & \gv{80.17}{0.40}
 & \cellcolor{gernBand}\gv{\textbf{53.16}}{0.68} & \cellcolor{gernBand}\gv{\textbf{54.12}}{0.77} & \cellcolor{gernBand}\gv{\textbf{53.27}}{0.84} \\
\bottomrule
\end{tabular}
\end{adjustbox}
\vspace{9pt}
\caption{\textbf{Accuracy (\%) under full-graph inference, EDiS-Full
vs.\ SpanGNN and GERN.} Mean above, standard deviation below, at
retained-edge fraction $\rho$ (columns), by method (rows).}
\label{tab:gern_spangnn}
\end{table}
\endgroup

\subsection{Comparison with MoG under GraphSAGE}
\label{app:mog_graphsage}

\underline{\textbf{MoG}} \citep{Zhang:2025mog} mixes several sparsification
experts, each favoring a different structural criterion, into one
adaptive sparse graph. \edis{} instead decomposes the graph once and
composes each epoch's graph from that cached decomposition. MoG is
one of the 17 baselines in Section~\ref{sec:experiments} under our
own GCN protocol. Here we use MoG's own published GraphSAGE results
as an additional reference and run \edis{}-Lite under the same
backbone, datasets, and evaluation.

\enlargethispage{2.75\baselineskip}
\begingroup
\definecolor{mogBand}{RGB}{235,235,235}
\begin{table}[h]
\centering
\setlength{\tabcolsep}{6pt}
\renewcommand{\arraystretch}{1.0}
\small
\begin{tabular}{l cc cc cc}
\toprule
& \multicolumn{2}{c}{ogbn-arxiv (Acc.)} & \multicolumn{2}{c}{ogbn-proteins (ROC-AUC)} & \multicolumn{2}{c}{ogbn-products (Acc.)} \\
\cmidrule(lr){2-3}\cmidrule(lr){4-5}\cmidrule(lr){6-7}
Method & $\rho{=}0.30$ & $\rho{=}0.50$ & $\rho{=}0.30$ & $\rho{=}0.50$ & $\rho{=}0.30$ & $\rho{=}0.50$ \\
\midrule
MoG \citep{Zhang:2025mog}
 & \cellcolor{mogBand}$67.31$ & \cellcolor{mogBand}$69.06$
 & $76.12$ & $76.46$
 & \cellcolor{mogBand}$74.91$ & \cellcolor{mogBand}$76.98$ \\
\midrule
\textbf{EDiS-Lite}
 & \cellcolor{mogBand}$\mathbf{70.35}$ & \cellcolor{mogBand}$\mathbf{72.32}$
 & $\mathbf{79.02}$ & $\mathbf{78.81}$
 & \cellcolor{mogBand}$\mathbf{79.68}$ & \cellcolor{mogBand}$\mathbf{80.14}$ \\
\bottomrule
\end{tabular}
\vspace{5pt}
\caption{\textbf{MoG's published GraphSAGE results vs.\ EDiS-Lite.}
Sparsity ($70\%$, $50\%$ removed) is converted to
$\rho=1-\text{sparsity}$. Bold marks the higher value per column.}
\label{tab:mog_graphsage}
\end{table}
\endgroup

\newpage
\section{Limitations}
\label{app:limitations}

Our experiments focus on node classification, so the benefits of
reusable edge-disjoint decompositions for graph- and link-level
prediction remain to be established. Our default edge score,
residual selector, and extraction-depth cap all come from an
empirical sweep rather than a principled per-dataset selection rule,
and no single choice dominates across datasets, so automatic tuning
of any of the three remains open. The composed graph is not
guaranteed to remain connected at every fixed edge budget, and our
default trimming rule favors total cached weight over expected
retained score, a choice we did not test against that alternative
empirically. \edis{} also assumes the full graph is available before
extraction begins, so inductive or streaming settings where edges
arrive afterward are not addressed. Finally, our empirical study
focuses on GCN and GraphSAGE backbones, and broader architectural
validation remains future work.

\end{document}